\documentclass[final,12pt]{colt2021} % Include author names
\usepackage{times}
\usepackage{natbib}
\usepackage{dsfont}
\usepackage{enumitem}
\newcommand{\C}{G}

\usepackage[english]{babel}
\usepackage{amsmath, amssymb}
\usepackage{graphicx}
\usepackage{latexsym}
\usepackage{graphicx}
\usepackage{cases}
\usepackage[mathscr]{euscript}
\usepackage{xcolor}
\usepackage{colortbl}
\usepackage{thmtools}
\usepackage{thm-restate}
\usepackage{mathtools}
\usepackage{caption}

\usepackage{pifont}% http://ctan.org/pkg/pifont
\newcommand{\cmark}{\ding{51}}%
\newcommand{\xmark}{\ding{55}}%

\let\vec\undefined
\usepackage{MnSymbol}
\DeclareMathAlphabet\mathbb{U}{msb}{m}{n}
\usepackage{xpatch}

\definecolor{Gray}{gray}{0.85}
\newcolumntype{g}{>{\columncolor{Gray}}c}

\hypersetup{
  colorlinks   = true, %Colours links instead of ugly boxes
  urlcolor     = blue, %Colour for external hyperlinks
  linkcolor    = blue, %Colour of internal links
  citecolor   = blue %Colour of citations
}

\def\Rset{\mathbb{R}}
\def\Nset{\mathbb{N}}

\DeclareMathOperator*{\E}{\mathbb{E}}

\DeclareMathOperator*{\argmin}{\rm argmin}
\DeclareMathOperator{\sgn}{sgn}

\newcommand{\nrm}[1]{{\left\vert\kern-0.25ex\left\vert\kern-0.25ex\left\vert #1 
    \right\vert\kern-0.25ex\right\vert\kern-0.25ex\right\vert}}

\DeclarePairedDelimiter{\bracket}{[}{]}
\DeclarePairedDelimiter{\curl}{\{}{\}}
\DeclarePairedDelimiter{\paren}{(}{)}

\newcommand{\cA}{\mathcal{A}}
\newcommand{\cC}{\mathcal{C}}

\newcommand{\cF}{\mathcal{F}}

\newcommand{\cL}{\mathcal{L}}

\newcommand{\cR}{\mathcal{R}}

\newcommand{\cU}{\mathcal{U}}

\newcommand{\sA}{{\mathscr A}}

\newcommand{\sH}{{\mathscr H}}

\newcommand{\sP}{{\mathscr P}}

\newcommand{\sX}{{\mathscr X}}
\newcommand{\sY}{{\mathscr Y}}

\newcommand{\bs}{{\mathbf s}}

\newcommand{\bu}{{\mathbf u}}

\newcommand{\bw}{{\mathbf w}}
\newcommand{\bx}{{\mathbf x}}
\newcommand{\bz}{{\mathbf z}}

\newcommand{\e}{\epsilon}

\newcommand{\ignore}[1]{}

\newcommand{\Cpae}{\mathcal{C}_{\phi}(t,\eta)}

\title[Calibration and Consistency of Adversarial Surrogate Losses]
{Calibration and Consistency of Adversarial Surrogate Losses}

\coltauthor{%
 \Name{Pranjal Awasthi} \Email{pranjal.awasthi@rutgers.edu}\\
 \addr Google Research and Rutgers University, New York
 \AND
 \Name{Natalie Frank} \Email{nf1066@nyu.edu}\\
 \addr Courant Institute of Mathematical Sciences, New York%
 \AND
 \Name{Anqi Mao} \Email{aqmao@cims.nyu.edu}\\
 \addr Courant Institute of Mathematical Sciences, New York%
 \AND
 \Name{Mehryar Mohri} \Email{mohri@google.com}\\
 \addr Google Research and Courant Institute of Mathematical Sciences, New York%
 \AND
 \Name{Yutao Zhong} \Email{yutao@cims.nyu.edu}\\
 \addr Courant Institute of Mathematical Sciences, New York%
}
\usepackage[toc,page,header]{appendix}

\begin{document}

\maketitle

\begin{abstract}%
  Adversarial robustness is an increasingly critical property of
  classifiers in applications. The design of robust algorithms relies
  on surrogate losses since the optimization of the adversarial loss
  with most hypothesis sets is NP-hard. But which surrogate losses
  should be used and when do they benefit from theoretical guarantees?
  We present an extensive study of this question, including a detailed
  analysis of the \emph{$\sH$-calibration} and
  \emph{$\sH$-consistency} of adversarial surrogate losses. We show
  that, under some general assumptions, convex loss functions, or the
  supremum-based convex losses often used in applications, are not
  \emph{$\sH$-calibrated} for important hypothesis sets such as
  generalized linear models or one-layer neural networks. We then give
  a characterization of \emph{$\sH$-calibration} and prove that some
  surrogate losses are indeed \emph{$\sH$-calibrated} for the
  adversarial loss, with these hypothesis sets. Next, we show that
  \emph{$\sH$-calibration} is not sufficient to guarantee consistency
  and prove that, in the absence of any distributional assumption, no
  continuous surrogate loss is consistent in the adversarial
  setting. This, in particular, proves that a claim presented in a COLT 2020
  publication is inaccurate.\footnote{Calibration results there are correct 
  modulo subtle definition differences, but the consistency claim does not hold.} 
  Next, we identify natural conditions under which some
  surrogate losses that we describe in detail are
  \emph{$\sH$-consistent} for hypothesis sets such as generalized
  linear models and one-layer neural networks.
  We also report a series of empirical results with simulated data,
  which show that many \emph{$\sH$-calibrated} surrogate losses are
  indeed not \emph{$\sH$-consistent}, and validate our theoretical
  assumptions.
\end{abstract}

\begin{keywords}%
  Adversarial Robustness, Surrogate Losses, Calibration, Consistency.
\end{keywords}

\section{Introduction}
\label{sec:introduction}

  Complex multi-layer artificial neural networks trained on large
  datasets have been shown to form accurate learning models which have
  achieved a remarkable performance in several applications in recent
  years, in particular in speech and visual recognition tasks
  \citep{SutskeverVinyalsLe2014,KrizhevskySutskeverHinton2012}.
  However, these rich models are susceptible to imperceptible
  perturbations \citep{szegedy2013intriguing}. A complex neural
  network may, for example, misclassify a traffic sign, as a result of
  a minor variation, which may be the presence of a small
  advertisement sticker on the sign. Such misclassifications can have
  dramatic consequences in practice, for example, for self-driving
  cars.
  These concerns have motivated the study of \emph{adversarial
  robustness}, that is the design of classifiers that are robust to
  small $\ell_p$ norm input perturbations
  \citep{goodfellow2014explaining, madry2017towards,
    tsipras2018robustness, carlini2017towards}. The standard $0/1$
  loss is then replaced with a more stringent \emph{adversarial loss},
  which requires a predictor to correctly classify an input point
  $\bx$ and also to maintain the same classification for all points at
  a small $\ell_p$ distance of $\bx$. But, can we devise efficient
  learning algorithms with theoretical guarantees for the adversarial
  loss?

  Designing such robust algorithms requires resorting to appropriate
  surrogate losses as optimizing the adversarial loss is NP-hard for
  most hypothesis sets. A key property for surrogate adversarial
  losses is their consistency, that is, that exact or near optimal
  minimizers of the surrogate loss are also exact or near optimal
  minimizers of the original adversarial loss.  The notion of
  consistency has been extensively studied in the case of the standard
  $0/1$ loss or the multi-class setting
  \citep{Zhang2003,bartlett2006convexity,tewari2007consistency,
    steinwart2007compare}. However, those results or proof techniques
  cannot be used to establish or characterize consistency in
  adversarial settings. This is because the adversarial loss of a
  predictor $f$ at point $\bx$ is inherently not just a function of
  $f(\bx)$ but also of its values around a neighborhood of $\bx$. As
  we shall see, the study of consistency is significantly more complex
  in the adversarial setting, with subtleties that have in fact led to
  some inaccurate claims made in prior work that we discuss later.
  
  Consistency requires a property of the surrogate and the original
  losses to hold true for the family of all measurable functions. As
  argued by \citet{long2013consistency}, the notion of
  \emph{$\sH$-consistency} which requires a similar property for the
  surrogate and original losses, but with the near or optimal
  minimizers considered on the restricted hypothesis set $\sH$, is a
  more relevant and desirable property for learning.
  \citet{long2013consistency} gave examples of surrogate losses that
  are not \emph{$\sH$-consistent} when $\sH$ is the class of all
  measurable functions but satisfy a condition namely,
  \emph{realizable $\sH$-consistency} when $\sH$ is the class of
  linear functions. More recently, \citet{zhang2020bayes} studied the
  notion of \emph{improper realizable $\sH$-consistency} of linear
  classes where the surrogate $\phi$ can be optimized over a larger
  class such as that of piecewise linear functions. Note that these
  works concern the standard $0/1$ classification loss.

  This motivates our main objective: an extensive study of the
  \emph{$\sH$-consistency} of adversarial surrogate losses, which is
  critical to the design of robust algorithms with guarantees in this
  setting. A more convenient notion in the study of
  \emph{$\sH$-consistency} is that of \emph{$\sH$-calibration}, which
  is a related notion that involves conditioning on the input
  point. \emph{$\sH$-calibration} often is a sufficient condition for
  \emph{$\sH$-consistency} in the standard classification settings
  \citep{steinwart2007compare}.
  However, the adversarial loss presents new challenges and requires
  carefully distinguishing among these notions to avoid drawing false
  conclusions. As an example, the recent COLT 2020 paper of
  \citet{pmlr-v125-bao20a}, which presents a study of
  \emph{$\sH$-calibration} for the adversarial loss in the special
  case where $\sH$ is the class of linear functions, concludes that
  the \emph{$\sH$-calibrated} surrogates they propose are
  \emph{$\sH$-consistent}. This is falsified as a by-product of our
  results, which further suggests that the adversarial setting is more
  complex and requires a more delicate analysis. At the same time, our
  work is inspired by the work of \citet{pmlr-v125-bao20a} where the
  author propose a natural robust loss function and studied
  calibration and consistency of surrogates for optimizing
  it. However, the proposed loss function corresponds to the
  adversarial $0/1$ loss only when the class $\sH$ of functions
  comprises of linear classifiers. We on the other hand directly study
  the adversarial $0/1$ loss and for hypothesis sets beyond linear
  classifiers.

%   We present a more systematic study of the $\sH$-calibration and
%   $\sH$-consistency of adversarial surrogate losses.
%
  In Section~\ref{sec:calibration}, we give a detailed analysis of the
  \emph{$\sH$-calibration} properties of several natural surrogate
  losses.  We present a series of new negative results showing that,
  under some general assumptions, convex loss functions and
  \emph{supremum-based convex losses}, that is losses defined as the
  supremum over a ball of a convex function, which are those commonly
  used in applications, are not \emph{$\sH$-calibrated} for several
  important hypothesis sets such as that of ReLU functions,
  generalized linear models and one-layer neural networks. Next, we
  give a characterization of calibration and prove that a family of
  \emph{quasi-concave even} surrogates are
  \emph{$\sH$-calibrated}. These significantly extend existing results
  of \citet{pmlr-v125-bao20a} given for the special case where $\sH$
  is the class of linear functions.

  In Section~\ref{sec:consistency}, we study the
  \emph{$\sH$-consistency} of surrogate loss functions. We prove that,
  in the absence of distributional assumptions, many surrogate losses
  shown to be \emph{$\sH$-calibrated} in Section~\ref{sec:calibration}
  are in fact not \emph{$\sH$-consistent}.
  Next, in contrast, we show that when the minimum of the surrogate
  loss is achieved within $\sH$, under some general conditions, the
  $\rho$-margin ramp loss (see, for example,
  \citet{MohriRostamizadehTalwalkar2018}) is \emph{$\sH$-consistent}
  for $\sH$ being the linear hypothesis set, any non-decreasing and
  continuous $g$-based hypothesis set, or the ReLU-based function
  class.
  We then give similar \emph{$\sH$-consistency} guarantees for
  supremum-based surrogate losses based on a non-increasing auxiliary
  function, including the calibrated supremum-based $\rho$-margin ramp
  loss when $\sH$ is the family of one-layer neural networks.

  In Section~\ref{sec:experiments}, we report a series of empirical
  results on simulated data, which show that many
  \emph{$\sH$-calibrated} surrogate losses are indeed not
  \emph{$\sH$-consistent}, and justify our realizability assumptions.
  Overall, our analysis suggests that surrogate losses typically used
  in practice do not benefit from any guarantee and that minimizing
  such losses may not in fact lead to a more favorable adversarial
  loss.  They also provide alternative surrogate losses with
  theoretical guarantees that can be useful to the design of
  algorithms in this setting.

We give a more detailed discussion of related work in
Appendix~\ref{app:related}. We start with an introduction of some
notation and key definitions (Section~\ref{sec:preliminaries}).

\section{Preliminaries}
\label{sec:preliminaries}

We will denote vectors as lowercase bold letters (e.g.\ $\bx$). The
$d$-dimensional $l_2$-ball with radius $r$ is denoted by
$B_2^d(r)\colon=\curl*{\bz\in\mathbb{R}^d\mid\|\bz\|_2\leq r }$. We
denote by $\sX$ the set of all possible examples. $\sX$ is also
sometimes referred to as the input space. The set of all possible
labels is denoted by $\sY$. We will limit ourselves to the case of
binary classification where $\sY=\{-1,1\}$. Let $\sH$ be a family of
functions from $\Rset^d$ to $\Rset$. Given a fixed but unknown
distribution $\sP$ over $\sX\times\sY$, the binary classification
learning problem is then formulated as follows. The learner is asked
to select a classifier $f^*\in \sH$ that has the minimal \emph{generalization error} with respect to the distribution $\sP$. The \emph{generalization error} of a
classifier $f \in \sH$ is defined by $\cR_{\ell_0}(f)=\E_{(\bx, y)
  \sim \sP}[\ell_0(f,\bx,y)]$, where $\ell_0(f,\bx,y)=\mathds{1}_{y
  f(\bx) \leq 0}$ is the standard $0/1$ loss. More generally, the
\emph{$\ell$-risk} of a classifier $f$ for a surrogate loss
$\ell(f,\bx,y)$ is defined by
\begin{align}
\label{eq:surrogate-risk}
    \cR_{\ell}(f) = \E_{(\bx, y) \sim \sP}[\ell(f, \bx, y)].
\end{align}
Moreover, the \emph{minimal ($\ell$,$\sH$)-risk}, which is also called
the \emph{Bayes ($\ell$,$\sH$)-risk}, is defined by
$\cR_{\ell,\sH}^*=\inf_{f\in\sH}\cR_{\ell}(f)$.  Our goal is to
understand whether the minimization of the $\ell$-risk can lead to
that of the generalization error. This motivates the definition of
\emph{$\sH$-consistency} (or simply \emph{consistency}) stated below.
% We would hope that the surrogate $\ell$ who has a good property can
% do so, which is called \textit{$\sH$-consistency} (or simply
% \textit{consistency}) with respect to $\ell_0$ and defined as
% follows.
\begin{definition}[$\sH$-Consistency]
  Given a hypothesis set $\sH$, we say that a loss function $\ell_1$
  is \emph{$\sH$-consistent} with respect to loss function $\ell_2$,
  if the following holds:
\begin{align}
\label{eq:H-consistency}
    \cR_{\ell_1}(f_n)-\cR_{\ell_1,\sH}^* \xrightarrow{n \rightarrow +\infty} 0 \implies \cR_{\ell_2}(f_n)-\cR_{\ell_2,\sH}^* \xrightarrow{n \rightarrow +\infty} 0,
\end{align}
for all probability distributions and sequences of $\{f_n\}_{n\in \Nset}\subset \sH$.
\end{definition}
In the rest of the paper, the loss $\ell_2$ in the definition above
will correspond to the $0/1$ loss or the adversarial $0/1$ loss
depending on the context, $\ell_1$ to a surrogate loss for
$\ell_2$. For a distribution $\sP$ over $\sX \times \sY$ with random
variables $X$ and $Y$, let $\eta_{\sP} \colon\sX \rightarrow [0,1]$ be a
measurable function such that, for any $\bx \in \sX$, $\eta_{\sP}(\bx) =
\sP(Y = 1 \mid X=\bx)$. By the property of conditional expectation, we
can rewrite \eqref{eq:surrogate-risk} as $\cR_{\ell}(f) =
\E_{X}[\cC_{\ell}(f, \bx, \eta_{\sP}(\bx))]$, where $\cC_{\ell}(f, \bx,
\eta)$ is the \emph{generic conditional $\ell$-risk} (or
\emph{inner $\ell$-risk}) defined as followed:
\begin{equation}
\label{eq:conditional-risk}
\forall \bx \in \sX, \forall \eta \in [0, 1], \quad 
\cC_{\ell}(f,\bx,\eta)\colon 
= \eta \ell(f, \bx, +1) + (1 - \eta)\ell(f, \bx, -1).
\end{equation}
Moreover, the \emph{minimal inner $\ell$-risk} on $\sH$ is denoted by
$\cC_{\ell,\sH}^{*}(\bx,\eta)\colon=\inf_{f\in\sH}\cC_{\ell}(f,\bx,\eta).$
We also define, overloading the notation, the \emph{pseudo-minimal
inner $\ell$-risk}
$\cC_{\ell,\sH}^*(\eta)\colon=\inf_{f\in\sH,\bx\in\sX}\cC_{\ell}(f,\bx,\eta)$. For convenience, we denote $\Delta\cC_{\ell,\sH}(f,\bx,\eta)\colon=\cC_{\ell}(f,\bx,\eta)-\cC_{\ell,\sH}^*(\eta)$. The notion of calibration for the inner risk is often a powerful tool for
the analysis of $\sH$-consistency \citep{steinwart2007compare}.
In this paper, we consider a uniform version of the notion of calibration.
% \begin{definition}[$\sH$-Calibration]\emph{[Definition~2.7 in \citep{steinwart2007compare}]}
% \label{def:H-calibration_true}
% Given a hypothesis set $\sH$, we say that a loss function $\ell_1$ is
% \emph{$\sH$-calibrated} with respect to a loss function $\ell_2$
% if, for any $\epsilon>0$, $\eta\in
%      [0,1]$, and $\bx \in \sX$, there exists $\delta>0$ such that for all $f\in\sH$ we have
% \begin{align}
% \label{eq:H-calibration_true}
%     \cC_{\ell_1}(f,\bx,\eta)<\cC_{\ell_1,\sH}^{*}(\bx,\eta)+\delta \implies \cC_{\ell_2}(f,\bx,\eta)<\cC_{\ell_2,\sH}^{*}(\bx,\eta)+\epsilon.
% \end{align}
% \end{definition}

% There is also a uniform version of the notion of calibration, that is, \emph{Uniform-$\sH$-Calibration}, which would imply \emph{$\sH$-Calibration}.
\begin{definition}[Uniform $\sH$-Calibration]\emph{[Definition~2.15 in \citep{steinwart2007compare}]}
\label{def:H-calibration_real}
Given a hypothesis set $\sH$, we say that a loss function $\ell_1$ is
uniformly \emph{$\sH$-calibrated} with respect to a loss function $\ell_2$
if, for any
$\epsilon>0$, there exists $\delta>0$ such that for all $\eta\in
     [0,1]$, $f\in\sH$, $\bx \in \sX$, we have
\begin{align}
\label{eq:H-calibration_real}
    \cC_{\ell_1}(f,\bx,\eta)<\cC_{\ell_1,\sH}^{*}(\bx,\eta)+\delta \implies \cC_{\ell_2}(f,\bx,\eta)<\cC_{\ell_2,\sH}^{*}(\bx,\eta)+\epsilon.
\end{align}
\end{definition}
% Note $\delta$ in uniform calibration (Definition~\ref{def:H-calibration_real}) is independent of $\eta$ and $\bx$, while $\delta$ in non-uniform calibration (Definition~\ref{def:H-calibration_true}) is dependent on $\eta$ and $\bx$. In general, the two definitions are both called $\sH$-calibration or $\sH$-calibrated, unless otherwise stated.
\citet[Remark 2.14]{steinwart2007compare} points out that the excess risk of a surrogate loss $\ell_1$ can be upper bounded in terms of the excess risk of target loss $\ell_2$ with a function that is independent of the specific distribution $\sP$ if $\ell_1$ is uniformly calibrated with respect to $\ell_2$ under certain conditions. For convenience of proofs, we also introduce the
\emph{Uniform Pseudo-$\sH$-Calibration} from \citet{pmlr-v125-bao20a}.
\begin{definition}[Uniform Pseudo-$\sH$-Calibration]\emph{[Definition~2 in \citep{pmlr-v125-bao20a}]}
\label{def:H-calibration}
Given a hypothesis set $\sH$, we say that a loss function $\ell_1$ is
\emph{uniformly pseudo-$\sH$-calibrated} with respect to a loss function $\ell_2$ if, for any
$\epsilon>0$, there exists $\delta>0$ such that for all $\eta\in
     [0,1]$ and $f\in\sH,\bx\in\sX$, we have
\begin{align}
\label{eq:H-calibration}
    \cC_{\ell_1}(f,\bx,\eta)<\cC_{\ell_1,\sH}^*(\eta)+\delta \implies \cC_{\ell_2}(f,\bx,\eta)<\cC_{\ell_2,\sH}^*(\eta)+\epsilon.
\end{align}
\end{definition}
Although the only difference between \eqref{eq:H-calibration_real} and
\eqref{eq:H-calibration} is the definition of minimal inner risk:
$\cC_{\ell_2,\sH}^*(\eta)$ and $\cC_{\ell_2,\sH}^{*}(\bx,\eta)$, in
general, uniform pseudo-$\sH$-calibration does not imply
$\sH$-consistency. However, as shown in Section~\ref{sec:calibration},
for the appropriate hypothesis sets $\sH$ and losses considered in
this paper, the two definitions coincide, that is, for any $\bx \in
\sX$, $\cC_{\ell,\sH}^{*}(\bx,\eta) = \cC_{\ell,\sH}^*(\eta)$ when
$\ell = \ell_1$ and $\ell_2$, and thus we can make use of
Definition~\ref{def:H-calibration} in the proofs. For simplicity, we are referring to Definition~\ref{def:H-calibration_real} (or Definition~\ref{def:H-calibration}), when we later write $\sH$-Calibration and $\sH$-calibrated (or Pseudo-$\sH$-Calibration and pseudo-$\sH$-calibrated).
% Based on this result, we mainly consider uniform-$\sH$-calibration in Section \ref{sec:calibration}, which would imply non-uniform $\sH$-calibration. However, for consistency results under the realizability assumption, it suffices to consider non-uniform $\sH$-calibration in
%  Theorem \ref{Thm:calibrate_consistent_nonsup} and Theorem \ref{Thm:calibrate_consistent_sup} in Section \ref{sec:consistency}.

\citet{steinwart2007compare} points out that if $\ell_1$ is $\sH$-calibrated with respect to $\ell_2$, then $\sH$-consistency, that is condition \eqref{eq:H-consistency}, holds for any probability distribution verifying the 
additional condition of
\emph{$\sP$-minimizability}
\citep[Definition 2.4]{steinwart2007compare}. This result holds,
in fact, under more general assumptions, as we will show later.
% \begin{proposition}[Theorem 2.8 in \citet{steinwart2007compare}]
% $\sH$-Calibration would imply $\sH$-Consistency under appropriate assumption of the underlying distribution.
% \end{proposition}
%
Next, we introduce the notions of \emph{uniform calibration function}
\citep{steinwart2007compare}, and \emph{uniform pseudo-calibration function}.
\begin{definition}[Uniform Calibration function]
\label{def:def-calibration-function}
Given a hypothesis set $\sH$, we define the \emph{uniform calibration function} $\delta$ and \emph{uniform pseudo-calibration function} $\hat \delta$ for a pair of losses $(\ell_1,\ell_2)$ as follows: for any $\e > 0$,
\begin{equation}
\label{eq:def-calibration-function}
\begin{aligned}
    &\delta(\epsilon) = \inf_{\eta\in[0,1]}\inf_{f\in\sH,\bx\in\sX} \curl[\Big]{\cC_{\ell_1}(f,\bx,\eta) - \cC^{*}_{\ell_1,\sH}(\bx,\eta) \mid
    \cC_{\ell_2}(f,\bx,\eta) - \cC^{*}_{\ell_2,\sH}(\bx,\eta)\geq\epsilon}\\
    &\hat{\delta}(\epsilon) = \inf_{\eta\in[0,1]}\inf_{f\in\sH,\bx\in\sX} \curl[\Big]{\cC_{\ell_1}(f,\bx,\eta) - \cC^{*}_{\ell_1,\sH}(\eta) \mid \cC_{\ell_2}(f,\bx,\eta) - \cC^{*}_{\ell_2,\sH}(\eta)\geq\epsilon }.
\end{aligned}
\end{equation}

\end{definition}
The uniform calibration function gives the maximal value $\delta$ satisfying
condition \eqref{eq:H-calibration_real} for a given $\e$, and,
similarly, the uniform pseudo-calibration function gives the maximal $\delta$
satisfying condition \eqref{eq:H-calibration}.

\ignore{ For
  convenience, we refer ``calibration function'' and
  ``$\delta(\epsilon)$'' alternatively as uniform-calibration function or
  pseudo-uniform-calibration function based on the context. Sometimes, we also
  refer ``$\sH$-calibrated'' and ``calibrated'' alternatively as
  pseudo-uniform-$\sH$-calibration or true $\sH$-calibration based on the
  context.}
  
The following proposition is an important result from
\citet{steinwart2007compare}. The sub-result for
uniform pseudo-$\sH$-calibration can be derived in the exact same way as for
uniform $\sH$-calibration.
\begin{proposition}[Lemma~2.16 in \citep{steinwart2007compare}]
\label{prop:calibration_function_positive}
Given a hypothesis set $\sH$, loss $\ell_1$ is uniformly $\sH$-calibrated
(or uniformly pseudo-$\sH$-calibrated) with respect to $\ell_2$ if and only if
its uniform calibration function $\delta$ satisfies $\delta(\epsilon)>0$
(resp. its uniform pseudo-calibration function $\hat{\delta}$ satisfies
$\hat{\delta}(\epsilon)>0$) for all $\epsilon>0$.
\end{proposition}
For simplicity, we are referring to Definition~\ref{def:def-calibration-function}, when we later write calibration function or pseudo-calibration function.

\ignore{
Given a loss function $\phi(t)\colon t\in \Rset\rightarrow \Rset_{+}$,
we define the \emph{margin-based loss} as
$\hat{\phi}(f,\bx,y)\colon=\phi(yf(\bx))$ for all $f\in \sH,\bx\in
\sX,y\in \sY$. For simplicity, we always identify $\phi$ with
$\hat{\phi}$ and call them both margin-based losses throughout the
paper.
}

\paragraph{Robust Classification.} In adversarially robust
classification, the loss at $(\bx,y)$ is measured in terms of the
worst loss incurred over an adversarial perturbation of $\bx$ within a
ball of a certain radius in a norm. In this work we will consider
perturbations in the $l_2$ norm $\|\cdot\|$. We will denote by
$\gamma$ the maximum magnitude of the allowed perturbations. Given
$\gamma>0$, a data point $(\bx,y)$, a function $f\in \sH$, and a
margin-based loss $\phi\colon\mathbb{R}\rightarrow\mathbb{R}_{+}$, we
define the adversarial loss of $f$ at $(\bx,y)$ as
\begin{align}
\label{eq:sup-based-surrogate}
\tilde{\phi}(f,\bx,y)=\sup\limits_{\bx'\colon \|\bx-\bx'\|\leq \gamma}\phi(y f(\bx')).
\end{align}
The above naturally motivates {\em supremum-based} surrogate losses
that are commonly used to optimize the adversarial $0/1$
loss \citep{goodfellow2014explaining, madry2017towards,
  shafahi2019adversarial, wong2020fast}. We say that a surrogate loss
$\tilde{\phi}(f, \bx, y)$ is \emph{supremum-based} if it is of the
form defined in \eqref{eq:sup-based-surrogate}. We say that the
supremum-based surrogate is convex if the function $\phi$ in
\eqref{eq:sup-based-surrogate} is convex. When $\phi$ is
non-increasing, the following equality
holds \citep{YinRamchandranBartlett2019}:
\begin{align}
\label{eq:sup=inf}
\sup\limits_{\bx'\colon \|\bx-\bx'\|\leq \gamma}\phi(y f(\bx')) = \phi\paren*{\inf\limits_{\bx'\colon \|\bx-\bx'\|\leq \gamma}yf(\bx')}.
\end{align}
Next we define the adversarial $0/1$ loss as
\begin{equation}
\label{eq:supinf01}
    \ell_{\gamma}(f,\bx,y)=\sup\limits_{\bx'\colon \|\bx-\bx'\|\leq \gamma}\mathds{1}_{y f(\bx') \leq 0}=\mathds{1}_{\inf\limits_{\bx'\colon \|\bx-\bx'\|\leq \gamma}yf(\bx')\leq 0}.
\end{equation}
Similarly, we define the \emph{adversarial generalization error} and the Bayes ($\ell_{\gamma}$,$\sH$)-risk as
\[
\cR_{\ell_{\gamma}}(f)=\E_{(\bx, y) \sim \sP}[\ell_{\gamma}(f,\bx,y)] \quad \text{and} \quad \cR_{\ell_{\gamma},\sH}^*=\inf_{f\in\sH}\cR_{\ell_{\gamma}}(f).
\]
In this paper, we aim to characterize surrogate losses satisfying
$\sH$-consistency and $\sH$-calibration with $\ell_2 = \ell_{\gamma}$
and for the following natural hypothesis sets $\sH$:
\begin{itemize}[itemsep=-1mm]

\item linear models: $\sH_{\mathrm{lin}}=\curl*{\bx\rightarrow \bw \cdot \bx \mid \|\bw\|=1}$, as in \citep{pmlr-v125-bao20a};

\item generalized linear models: $\sH_{g} = \curl*{\bx\rightarrow
  g(\bw \cdot \bx)+b\mid\|\bw\|= 1, |b|\leq \C}$ where $g$ is a
  non-decreasing function; and

\item one-layer ReLU neural networks:
$\sH_{\mathrm{NN}} = \curl*{\bx\rightarrow \sum_{j = 1}^n u_j(\bw_j \cdot \bx)_{+} \mid \|\bu \|_{1}\leq \Lambda, \|\bw_j\|\leq W}$, where $(\cdot)_+ = \max(\cdot,0)$.

\end{itemize}
In the special case of $g = (\cdot)_+$, 
we denote the corresponding ReLU-based hypothesis set as $\sH_{\mathrm{relu}} = \curl*{\bx\rightarrow (\bw \cdot \bx)_{+} + b \mid \|\bw\|=1, |b|\leq \C}$.
We also denote the set of all measurable functions by $\sH_{\mathrm{all}}$.

\section{\texorpdfstring{$\sH$}{H}-Calibration}
\label{sec:calibration}

Calibration is a condition often used to prove consistency and is typically 
a first step in analyzing surrogate losses. Thus, in this section, we
first present a detailed study of the calibration properties of
several loss functions.  We first prove the equivalence of $\sH$-calibration and pseudo-$\sH$-calibration under some broad
assumptions.
Next, we give a series of negative results showing that, under general
assumptions, convex losses and supremum-based convex losses, which are
typically used in practice for adversarial robustness, \emph{are not
calibrated}.
We then complement
these results with positive ones by identifying a family of
quasi-concave even functions that are indeed calibrated under certain
general conditions.
Without loss of generality, in this section, we assume the input space
to be $\sX = B_{2}^d(1)$ and $\gamma \in (0, 1)$. Specifically, we
assume the input space to be $\sX =\curl*{\bx\in\mathbb{R}^d\mid\gamma
  <\|\bx\|_2\leq 1 }$ when considering one-layer ReLU neural
networks $\sH_{\mathrm{NN}}$.

\subsection{Equivalence of calibration definitions}

We first show that the definitions of $\sH$-calibration and pseudo-$\sH$-calibration
coincide for the hypothesis sets and losses considered in the paper.

\begin{restatable}{theorem}{calibrationDefinitionEquivalent}
  \textbf{\emph{[Equivalence of calibration definitions]}}
\label{Thm:calibration_definition_equivalent}
Without loss of generality, let $\sX = B_{2}^d(1)$ and
$\gamma \in (0, 1)$. Then,
\begin{enumerate}
    \item If $\sH$ satisfies: for any $\bx\in \sX$, there exists $f\in
      \sH$ such that $\inf_{\bx'\colon \|\bx-\bx'\|\leq \gamma}f(\bx')>0$, and $f\in \sH$ such that $\sup_{\bx'\colon
        \|\bx - \bx'\|\leq \gamma}f(\bx')< 0$, then, for any $\bx \in
      \sX$,
      $\cC_{\ell_{\gamma},\sH}^{*}(\bx,\eta)=\cC_{\ell_{\gamma},\sH}^*(\eta)$.

    \item Let $\phi$ be a margin-based loss. If $\sH$ satisfies: for
      any $\bx\in \sX$, $\{f(\bx)\colon f\in \sH\}=\mathbb{R}$, then,
      for any $\bx \in \sX$ ,
      $\cC_{\phi,\sH}^{*}(\bx,\eta)=\cC_{\phi,\sH}^*(\eta)$.

    \item Let $\phi_{\rho}(t) =
      \min \curl*{1,\max\curl*{0, 1 - \frac{t}{\rho}}}$ for a fixed
      $\rho > 0$ be the \emph{$\rho$-margin loss} and $\tilde \phi_{\rho}(f, \bx, y) =
      \sup_{\bx'\colon \| \bx - \bx' \|\leq \gamma} \phi_{\rho}(y
      f(\bx'))$ be the corresponding supremum-based loss. If $\sH$
      satisfies: for any $\bx \in \sX$, there exists $f\in \sH$ such
      that $\inf_{\bx' \colon \|\bx - \bx'\| \leq \gamma}f(\bx') >
      \rho$, and $f\in \sH$ such that $\sup_{\bx'\colon \|\bx -
        \bx'\|\leq \gamma}f(\bx')< -\rho$, then, for any $\bx \in
      \sX$, $\cC_{\tilde \phi_{\rho}, \sH}^*(\bx,\eta) = \cC_{\tilde
        \phi_{\rho}, \sH}^*(\eta)$.
\end{enumerate}
\end{restatable}   
The proof is deferred to
Appendix~\ref{app:calibration_definition_equivalent}.  Note that, by
Definitions~\ref{def:H-calibration_real} and \ref{def:H-calibration},
when
$\cC_{\ell_{\gamma},\sH}^{*}(\bx,\eta)=\cC_{\ell_{\gamma},\sH}^*(\eta),
\forall \bx\in \sX$, if a loss function $\ell$ is $\sH$-calibrated
with respect to $\ell_{\gamma}$, then it is also $\sH$-calibrated with respect to $\ell_{\gamma}$ since
$\cC_{\ell,\sH}^*(\eta)\leq\cC_{\ell,\sH}^{*}(\bx,\eta), \forall
\bx\in \sX$. As a result, we obtain the following.

\begin{corollary}
\label{corollary:calibration_negative}
Assume that for any $\bx\in \sX$, there exists $f\in \sH$ such that $\inf_{\bx'\colon \|\bx-\bx'\|\leq \gamma}f(\bx')>0$, and
 $f\in \sH$ such that $\sup_{\bx'\colon \|\bx - \bx'\|\leq\gamma}f(\bx')< 0$, then if $\ell$ is not
 pseudo-$\sH$-calibrated with respect to $\ell_{\gamma}$, then $\ell$
 is also not $\sH$-calibrated with respect to $\ell_{\gamma}$.
\end{corollary}
This result is most helpful for obtaining our negative results of $\sH$-calibration in Section
\ref{sec:calibration_negative}. Specifically, in order to prove that a
loss function $\ell$ is not $\sH$-calibrated with respect to
$\ell_{\gamma}$, we only need to prove that $\ell$ is not
pseudo-$\sH$-calibrated with respect to $\ell_{\gamma}$, which
helps simplify our proofs.
Similarly, by Theorem \ref{Thm:calibration_definition_equivalent} and
Definitions~\ref{def:H-calibration_real} and
\ref{def:H-calibration}, we can derive the following corollary, which
is most helpful for obtaining our positive results of
$\sH$-calibration in Section \ref{sec:calibration_positive}.
\begin{corollary}
\label{corollary:calibration_positive}
Let $\phi$ be a margin-based loss,
$\phi_{\rho}(t)=\min\curl*{1,\max\curl*{0,1-\frac{t}{\rho}}},~\rho>0$
be the \emph{$\rho$-margin loss} and
$\tilde{\phi}_{\rho}(f,\bx,y)=\sup_{\bx'\colon \|\bx-\bx'\|\leq
  \gamma}\phi_{\rho}(y f(\bx'))$ be the corresponding supremum-based
loss. Then,
\begin{enumerate}
    \item If $\sH$ satisfies: for any $\bx\in \sX$, $\{f(\bx)\colon
      f\in \sH\}=\mathbb{R}$, then $\phi$ is pseudo-$\sH$-calibrated
      with respect to $\ell_{\gamma}$ if and only if $\phi$ is $\sH$-calibrated with respect to $\ell_{\gamma}$.

    \item If $\sH$ satisfies: for any $\bx\in \sX$, there exists $f\in
      \sH$ such that $\inf_{\bx'\colon \|\bx-\bx'\|\leq
        \gamma}f(\bx')>\rho$, and $f\in \sH$ such that
      $\sup_{\bx'\colon \|\bx - \bx'\|\leq \gamma}f(\bx')< -\rho$,
      then $\tilde{\phi}_{\rho}$ is pseudo-$\sH$-calibrated with
      respect to $\ell_{\gamma}$ if and only if $\tilde{\phi}_{\rho}$
      is $\sH$-calibrated with respect to $\ell_{\gamma}$.
\end{enumerate}
\end{corollary}
Therefore, for the hypothesis sets $\sH$, under broad assumptions, we
can provide alternative losses which are $\sH$-calibrated with respect
to $\ell_{\gamma}$ by considering losses that are pseudo-$\sH$-calibrated with respect to $\ell_{\gamma}$.

In this paper, we will adopt the natural condition 1.\ of Theorem
\ref{Thm:calibration_definition_equivalent} for the hypothesis set,
which is easily satisfied for any non-trivial class: for any $\bx\in
\sX$, there exists $f\in \sH$ such that
$\inf_{\bx'\colon \|\bx-\bx'\|\leq \gamma}f(\bx')=\inf_{\|\bs\|\leq1}f(\bx+\gamma \bs)>0$ and there exists $f\in \sH$
such that $\sup_{\bx'\colon \|\bx - \bx'\|\leq\gamma}f(\bx')=\sup_{\|\bs\|\leq1}f(\bx+\gamma \bs)<0$. As an example,
consider the class $\sH_{\mathrm{NN}}$ of one layer ReLU networks as
described in Section~\ref{sec:preliminaries}.
% These conditions hold for a wide range of hypothesis sets, including the one-layer ReLU Neural Network class $\sH_{\mathrm{NN}}$. 
For any $\bx \in \sX$ with $\|\bx\| = t > \gamma$, let $\bw_j = W \bx$ and
$u_j = \frac{\Lambda}{n}$, for $j=1,\dots,n$. Then, the following holds:
% Then for any $\bs\in \curl*{\bs\colon\|\bs\|\leq 1}$, 
% the following inequality holds:
\[
\forall \bs: \|\bs\| \leq 1, \quad \bw_j\cdot(\bx+\gamma \bs)
= W(\bx\cdot \bx+\gamma(\bx\cdot \bs))
\geq W(\|\bx\|^2-\gamma\|\bx\|\|\bs\|)
\geq Wt(t - \gamma)>0.
\]
Therefore, we have
\[
\inf\limits_{\|\bs\|\leq 1}\sum\limits_{j=1}^n u_j\left(\bw_j \cdot (\bx+\gamma \bs)\right)_{+}\geq \Lambda Wt(t-\gamma)>0.
\] 
Similarly, taking $u_j=-\frac{\Lambda}{n}$ instead, for $j = 1, \dots, n$, yields 
\[
\sup\limits_{\|\bs\|\leq 1} \sum\limits_{j=1}^n u_j\left(\bw_j \cdot (\bx+\gamma \bs)\right)_{+}
\leq - \Lambda Wt(t - \gamma) < 0.
\]

\subsection{Negative results}
\label{sec:calibration_negative}
In this section, we aim to study that common losses are not calibrated with respect to $\ell_{\gamma}$.
Note by Corollary \ref{corollary:calibration_negative}, in order to prove that a loss $\ell$ is not $\sH$-calibrated with respect to $\ell_{\gamma}$, we only need to prove that $\ell$ is not pseudo-$\sH$-calibrated with respect to $\ell_{\gamma}$, as showed in our proofs of this section.
\subsubsection{Convex losses}
\label{sec:calibration_negative_convex}

We first study convex losses which are often used
for standard binary classification problems. For a linear hypothesis set, $\sH = \sH_{\mathrm{lin}}$, \citet[Corollary 9]{pmlr-v125-bao20a} showed that convex losses are not pseudo-$\sH_{\mathrm{lin}}$-calibrated for the adversarial $0/1$ loss. 
\begin{theorem}[\citet{pmlr-v125-bao20a}]
\label{Thm:calibration_convex_linear}
If a margin-based loss $\phi\colon\Rset\rightarrow \Rset_{+}$ is convex, then it is not pseudo-$\sH_{\mathrm{lin}}$-calibrated with respect to $\ell_{\gamma}$.
\end{theorem}
Note that this result would not imply that $\phi$ is not $\sH$-calibrated with respect to $\ell_{\gamma}$, since $\sH_{\mathrm{lin}}$ does not satisfy condition 1.\ of Theorem \ref{Thm:calibration_definition_equivalent}. However, all of our results below hold under both $\sH$-calibration (Definition~\ref{def:H-calibration_real}) and pseudo-$\sH$-calibration (Definition~\ref{def:H-calibration}), since the hypothesis sets $\sH$ considered below all satisfy that condition. Actually, we give the proofs under the Definition~\ref{def:H-calibration} of pseudo-$\sH$-calibration, which, by Corollary~\ref{corollary:calibration_negative}, imply the negative results of $\sH$-calibration \eqref{eq:H-calibration_real}.
% \begin{theorem}[\cite{pmlr-v125-bao20a}]
% \label{Thm:calibration_convex_linear}
% Any convex margin-based loss $\phi\colon\Rset\rightarrow \Rset_{+}$ is not $\sH_{\mathrm{lin}}$-calibrated with respect to $\ell_{\gamma}$.
% \end{theorem}

Our first main contribution is to extend the above result to a more general case when $\sH$ is the class of generalized linear models $\sH_g$ under both calibration definitions.
% In the adversarial scenario, a series 
% of results show that convex losses are not $\sH$-calibrated with respect to the adversarial $0/1$ loss, where $\sH$ refers to the family of linear functions or to that of 
% non-decreasing $g$-based functions.
% %
% The following negative result for a linear hypothesis set $\sH_{\mathrm{lin}}$ is proven in \cite[Corollary~9]{pmlr-v125-bao20a}.
% \begin{theorem}
% If a margin-based loss $\phi\colon\Rset\rightarrow \Rset_{+}$ is convex, then $\phi$ is not $\sH_{\mathrm{lin}}$-calibrated with respect to $\ell_{\gamma}$.
% \end{theorem}
In particular, we show that convex losses are not $\sH_g$-calibrated with respect to $\ell_{\gamma}$ for 
a non-decreasing and continuous function $g$ that satisfies $g(1+\gamma)< \C$ and  $g(-1-\gamma)> -\C$ for some $\C > 0$. Verifying this condition is straightforward for $\C$ sufficiently large. It is obvious that $\sH_g$ with this condition on $g$ satisfy the condition 1. in Theorem \ref{Thm:calibration_definition_equivalent} on $\sH$.

\begin{restatable}{theorem}{calibrationConvexGeneral} 
\label{Thm:calibration_convex_general}
   Let $g$ be a non-decreasing and continuous function such that $g(1+\gamma)< \C$ and  $g(-1-\gamma)> -\C$. If a margin-based loss $\phi\colon\Rset\rightarrow \Rset_{+}$ is convex, then it is not $\sH_g$-calibrated with respect to $\ell_{\gamma}$.
\end{restatable}   
The proof of Theorem~\ref{Thm:calibration_convex_general} is included in Appendix~\ref{app:calibration_convex_general}. The key in proving the above theorem is to analyze the pseudo-calibration function $\hat{\delta}(\epsilon)$ as defined in \eqref{eq:def-calibration-function}. Naturally, this requires us to understand $\cC_{\phi,\sH_g}^*(\eta)$ that in turn depends on the worst case perturbation of a given data point according to $\phi$. To do so, we use the result of \citet{awasthi2020adversarial} that characterizes such perturbations for the case where $g$ is the ReLU function. We extend the characterization to non-decreasing continuous functions, and as a result obtain the form of the pseudo-calibration function in Lemma~\ref{lemma:bar_delta_general}. Requiring $\hat{\delta}(\epsilon) > 0$ for an appropriate value of $\eta$, then leads to a natural condition on the function $\bar{\phi}\big({\alpha_1, \alpha_2}\big) = \frac{1}{2} \phi(g(\alpha_1) + \alpha_2) + \frac{1}{2} \phi(-g(\alpha_1) - \alpha_2)$. Notice that this function is solely determined by the value of $g(\alpha_1) + \alpha_2$. For $\phi$ to be calibrated, we obtain the condition that $\bar{\phi}\big({\alpha_1, \alpha_2}\big)$ should not achieve a minimum inside the set $\sA = \curl*{g(\alpha_1)+ \alpha_2\colon -g(\alpha_1 + \gamma) \leq \alpha_2 \leq -g(\alpha_1 - \gamma)}$. However, notice that $0 \in \sA$ and $g(\alpha_1) + \alpha_2 = 0$ implies that $\bar{\phi}$ equals $\phi(0)$. Furthermore, due to convexity of $\phi$, $\phi(0) \leq \bar{\phi}\big(\alpha_1, \alpha_2 \big)$ thereby leading to a contradiction.
% The piece-wise calibration function $\delta(\epsilon)$ have one more piece which depends on $\gamma$ in the adversarial scenario. Due to the convexity of $\phi$, this piece of the calibration function $\delta$ would achieve $0$ in the case of $\eta=\frac12$, which implies the equivalent condition of calibration in Proposition~\ref{prop:calibration_function_positive} would not be satisfied.
As a special case, consider $(\cdot)_+$ which is non-decreasing and continuous. Then the condition $(-1-\gamma)_+=0>-\C$ is trivially satisfied, leading to the following corollary.
%for ReLU-based classes hypothesis sets.
\begin{corollary}
\label{corollary:convex_relu}
  Assume that $\C>1+\gamma$. If a margin-based loss $\phi\colon\Rset\rightarrow \Rset_{+}$ is convex, then $\phi$ is not $\sH_{\mathrm{relu}}$-calibrated with respect to $\ell_{\gamma}$.
\end{corollary}

While convex surrogates are natural for the $0/1$ loss, the current practice in designing practical algorithms for the adversarial loss involves using convex supremum-based surrogates \citep{madry2017towards, wong2020fast, shafahi2019adversarial}. We next investigate such losses.
%whether such surrogates could be calibrated.

\subsubsection{Supremum-based convex losses}
\label{sec:sup_based_convex}

We study losses of the type $\tilde{\phi}(f,\bx,y)=\sup_{\bx'\colon \|\bx-\bx'\|\leq \gamma}\phi(y f(\bx'))$, with $\phi$ convex, which are often used in practice
as surrogates for the adversarial $0/1$ loss. 
% As shown by \citet{awasthi2020adversarial}, the adversarial $0/1$ loss has the following equivalent form:
% \begin{equation}
% \label{eq:GN loss}
% 	l_{\gamma}(f,\bx,y)=\mathds{1}_{\inf\limits_{\|\bs\|\leq 1}\left(yf(\bx+\gamma \bs)\right)\leq 0}\,.
% \end{equation}
The following theorem presents a negative result for
supremum-based convex surrogate losses for the broad class of hypothesis sets $\sH$ investigated in this section. Its proof is deferred to Appendix~\ref{app:calibration_sup_convex}.
\begin{restatable}{theorem}{CalibrationSupConvex}
\label{Thm:calibration_sup_convex}
   Let $\sH$ be a hypothesis set containing 0. Assume that for any $\bx\in \sX$, there exists $f\in \sH$ such that $\inf_{\| \bx' - \bx \| \leq \gamma } f(\bx')>0$, and $f\in \sH$ such that $\sup_{\| \bx' - \bx \| \leq \gamma }f(\bx')<0$.
   If a margin-based loss $\phi$ is convex and non-increasing, then the surrogate loss defined by $	\tilde{\phi}(f,\bx,y)=\sup_{\bx'\colon \|\bx-\bx'\|\leq \gamma}\phi(y f(\bx'))$ is not $\sH$-calibrated with respect to $\ell_{\gamma}$.
\end{restatable}
The theorem above provides theoretical evidence that the current practice of making neural networks adversarially robust via minimizing convex supremum-based surrogates may have serious deficiencies. This lack of a principled choice of the surrogate loss may also explain why in practice the adversarial accuracies that are achievable are much lower than the corresponding natural accuracies of the model \citep{madry2017towards}. In general, optimizing non-calibrated or non-consistent surrogates could lead to undesirable solutions even under strong assumptions~(such as the Bayes risk being zero). See Section~\ref{sec:experiments}, where we empirically demonstrate this in a variety of settings.

In contrast with Theorem~\ref{Thm:calibration_convex_general}, the challenge in proving the above theorem is that, since we are working with a general class of functions, we no longer can hope for a complete characterization of the worst-case adversarial perturbations around a given point $\bx$. In fact, this is a challenging problem even in the case of one-layer networks \citep{awasthi2020adversarial}. This presents a difficulty in analyzing the pseudo-calibration function $\hat{\delta}(\epsilon)$. Our key insight~(Lemma~\ref{lemma:bar_delta_GN}) is that the pseudo-calibration function can be characterized by two quantities $\underline{M}(f,\bx,\gamma) =\inf_{\bx'\colon \|\bx - \bx'\|\leq\gamma} f(\bx')$, $\overline{M}(f,\bx,\gamma)=\sup_{\bx'\colon \|\bx - \bx'\|\leq\gamma}f(\bx')$. Once this is achieved, we follow a strategy similar to that of the proof of Theorem~\ref{Thm:calibration_convex_general}, where the condition $\hat{\delta}(\epsilon) > 0$ corresponds to an appropriate convex function not achieving a minimum in a set that contains $0$, thereby reaching a contradiction.  
% The proof is based on the equivalent form of $\tilde{\phi}$ for any non-increasing $\phi$ pointed out by \cite{YinRamchandranBartlett2019}:
% \begin{align}
% \label{eq:ChangeSup}
%   \tilde{\phi}(f,\bx,y)=\phi\left(\inf\limits_{\bx'\colon \|\bx-\bx'\|\leq\gamma}\left(yf(\bx')\right)\right). 
% \end{align}
% We can explicitly characterize the calibration function $\delta$ using this form and then the arguments about calibration would be similar to the non-supremum case. 
By Theorem~\ref{Thm:calibration_sup_convex} and the fact that $0\in\sH_{\mathrm{NN}}$, we can derive the following corollary for the class of one layer ReLU neural networks.
\begin{corollary}
\label{corollary:sup_convex_NN}
  If a margin-based loss $\phi$ is convex and non-increasing, then the surrogate loss defined by $\tilde{\phi}(f,\bx,y)=\sup_{\bx'\colon \|\bx-\bx'\|\leq \gamma}\phi(y f(\bx'))$ is not $\sH_{\mathrm{NN}}$-calibrated with respect to $\ell_{\gamma}$.
\end{corollary}

\subsection{Positive results}
\label{sec:calibration_positive}
In this section, we aim to provide alternative losses which could be calibrated with respect to $\ell_{\gamma}$. By Corollary \ref{corollary:calibration_positive}, we first give general pseudo-calibration results of our hypothesis $\sH_g$ and $\sH_{\mathrm{NN}}$, and then show that specific $\sH_g$ with $\C=+\infty$ and $\sH_{\mathrm{NN}}$ with $\Lambda=+\infty$ has corresponding true calibration results and then would also has consistency results under appropriate conditions in Section \ref{sec:consistency}.
\subsubsection{Characterization}
\label{sec:calibration_characterization}

In light of the negative results in Section~\ref{sec:calibration_negative}, to find calibrated surrogate losses for adversarially robust classification, we need
to consider non-convex ones. One possible candidate is the family of
\emph{quasi-concave even} losses introduced by \citet{pmlr-v125-bao20a},
which were shown to be pseudo-$\sH_{\mathrm{lin}}$-calibrated with respect to the adversarial $0/1$ loss under certain assumptions.
\begin{definition}[\citet{pmlr-v125-bao20a}]
  A margin-based loss function $\phi$ is said to be \emph{quasi-concave even}, 
  if $\phi(t)+\phi(-t)$ is quasi-concave. 
\end{definition}
\begin{theorem}[\citet{pmlr-v125-bao20a}]
\label{Thm:char_linear}
  Assume that a margin-based loss $\phi$ is bounded, non-increasing, and quasi-concave even. Let $B\overset{\text{def}}=\phi(1)+\phi(-1)$ and assume $\phi(-1)>\phi(1)$. Then $\phi$ is pseudo-$\sH_{\mathrm{lin}}$-calibrated with respect to $\ell_{\gamma}$ if and only if $\phi(\gamma)+\phi(-\gamma)>B$.
\end{theorem}
Note that this result doesn't hold under true calibration definition \eqref{eq:H-calibration_real}. We first extend the above to show that under certain conditions quasi-concave even surrogate losses are pseudo-$\sH_g$-calibrated for the class of generalized linear models with respect to the adversarial $0/1$ loss.
% More generally, for non-decreasing $g$-based hypothesis sets, the following theorem characterizes $\sH_g$-calibrated surrogates with respect to the adversarial $0/1$ loss. 
\begin{restatable}{theorem}{QuasiconcaveCalibrateGeneral}
\label{Thm:quasiconcave_calibrate_general}
  Let $g$ be a non-decreasing and continuous function such that $g(1+\gamma)< \C$ and  $g(-1-\gamma)> -\C$ for some $\C > 0$. Let a margin-based loss $\phi$ be bounded, continuous, non-increasing, and quasi-concave even. Assume that $\phi(g(-1)-\C)>\phi(\C-g(-1))$ and $g(-1)+g(1)\geq0$. Then $\phi$ is pseudo-$\sH_g$-calibrated with respect to $\ell_{\gamma}$ if and only if 
  \begin{align*}
  \phi(\C-g(-1))+\phi(g(-1)-\C) & = \phi(g(1)+\C)+\phi(-g(1)-\C)\\
  \text{and} \quad
  \min\curl*{\phi(\overline{A})+\phi(-\overline{A}),\phi(\underline{A})+\phi(-\underline{A}) } & > \phi(\C-g(-1))+\phi(g(-1)-\C),
 \end{align*}
  where $\overline{A}=\sup_{\alpha_1\in[-1,1]}g(\alpha_1)-g(\alpha_1-\gamma)$ and $\underline{A}=\inf_{\alpha_1\in[-1,1]}g(\alpha_1)-g(\alpha_1+\gamma)$.
\end{restatable}
The conditions in the Theorem above are if and only if and hence precisely characterize when quasi-concave even losses are pseudo-$\sH$-calibrated. To interpret the conditions better, consider ReLU functions. In this case, the assumptions in Theorem \ref{Thm:quasiconcave_calibrate_general} can be further simplified, since $\overline{A}=\sup_{\alpha_1\in[-1,1]}(\alpha_1)_+-(\alpha_1-\gamma)_+ = \gamma$ and $\underline{A}=\inf_{\alpha_1\in[-1,1]}(\alpha_1)_+-(\alpha_1+\gamma)_+=-\gamma$. As a result we get the following. \begin{corollary}
\label{corollary:char_relu}
  Assume that $\C>1+\gamma$. Let a margin-based loss $\phi$ be bounded, continuous, non-increasing, and quasi-concave even. Assume that $\phi(-\C)>\phi(\C)$. Then $\phi$ is pseudo-$\sH_{\mathrm{relu}}$-calibrated with respect to $\ell_{\gamma}$ if and only if
  \begin{align*}
  \phi(\C)+\phi(-\C)=\phi(1+\C)+\phi(-1-\C)
  \quad \text{and} \quad
  \phi(\gamma)+\phi(-\gamma)>\phi(\C)+\phi(-\C).
  \end{align*}
\end{corollary}
Theorem \ref{Thm:quasiconcave_calibrate_general} is proved in Appendix~\ref{app:quasiconcave_calibrate_general}. We again use the characterization of the pseudo-calibration function as derived in Lemma~\ref{lemma:bar_delta_general}. In Lemma~\ref{lemma:equivalent2_general} we further simplify the characterization to a set of three conditions that the surrogate loss must satisfy. Finally, we show that quasi-concave even losses satisfy them under the conditions of the Theorem. Along the way, building on the work of \citet{pmlr-v125-bao20a}, we establish several useful properties of quasi-concave even losses in Lemma~\ref{lemma:quasiconcave_even_general}.

% The ideas is to propose the equivalent condition of calibration based on inner risk of $\phi$ and $\sH_g$ by Proposition \ref{prop:calibration_function_positive} and calibration function $\delta$. Then we can establish that $\phi$ satisfies this equivalent condition exploiting the property of quasi-concave even. 
% As with Corollary~\ref{corollary:convex_relu} in Section \ref{sec:calibration_negative_convex}, for ReLU-based hypothesis set, 
    
\subsubsection{Calibration}
\label{sec:calibration_rho_margin}

To demonstrate the applicability of Theorem~\ref{Thm:quasiconcave_calibrate_general}, we consider a specific surrogate loss namely the \emph{$\rho$-margin loss} $\phi_{\rho}(t)=\min\curl*{1,\max\curl*{0,1-\frac{t}{\rho}}},~\rho>0$, 
which is a generalization of the ramp loss (see, for example, \citet{MohriRostamizadehTalwalkar2018}). Using Theorem~\ref{Thm:char_linear}, Theorem~\ref{Thm:quasiconcave_calibrate_general} and Corollary \ref{corollary:char_relu} in Section \ref{sec:calibration_characterization}, we can conclude that the $\rho$-margin loss is pseudo-$\sH$-calibrated under reasonable conditions for linear hypothesis sets and non-decreasing $g$-based hypothesis sets, since $\phi_{\rho}(t)$ is bounded, non-increasing and quasi-concave even. This is stated formally below.
\begin{theorem}
\label{Thm:rho_margin_calibtartion}
  Consider $\rho$-margin loss $\phi_{\rho}(t)=\min\curl*{1,\max\curl*{0,1-\frac{t}{\rho}}},~\rho>0$. Then,
  \begin{enumerate}
  
      \item $\phi_{\rho}$ is pseudo-$\sH_{\mathrm{lin}}$-calibrated with respect to $\ell_{\gamma}$ if and only if $\rho>\gamma$;
     
      \item Given a non-decreasing and continuous function $g$ such that $g(1+\gamma)< \C$ and  $g(-1-\gamma)> -\C$. Assume that $g(-1)+g(1)\geq0$. Then $\phi_{\rho}$ is pseudo-$\sH_g$-calibrated with respect to $\ell_{\gamma}$ if and only if 
      \begin{align*}
      \phi_{\rho}(\C-g(-1)) & = \phi_{\rho}(g(1)+\C)
  \quad \text{and} \quad
  \min\curl*{\phi_{\rho}(\overline{A}),\phi_{\rho}(-\underline{A}) }  > \phi_{\rho}(\C-g(-1)),
  \end{align*}
  where $\overline{A}=\sup_{\alpha_1\in[-1,1]}g(\alpha_1)-g(\alpha_1-\gamma)$ and $\underline{A}=\inf_{\alpha_1\in[-1,1]}g(\alpha_1)-g(\alpha_1+\gamma)$;
      \item Assume that $\C>1+\gamma$. Then $\phi_{\rho}$ is pseudo-$\sH_{\mathrm{relu}}$-calibrated with respect to $\ell_{\gamma}$ if and only if 
      $\C\geq\rho>\gamma$.
  \end{enumerate}
\end{theorem}
Specifically, $\sH_g$ with the extra assumption $\C=+\infty$ satisfy the condition 1. in Corollary \ref{corollary:calibration_positive}, as a result we get the following.
\begin{corollary}
\label{corollary:rho_margin_calibtartion}
 Consider $\rho$-margin loss $\phi_{\rho}(t)=\min\curl*{1,\max\curl*{0,1-\frac{t}{\rho}}},~\rho>0$. Then,
  \begin{enumerate}
      \item Given a non-decreasing and continuous function $g$. Assume that $\C=+\infty$ and $g(-1)+g(1)\geq0$. Then $\phi_{\rho}$ is $\sH_g$-calibrated with respect to $\ell_{\gamma}$ if and only if $\min\curl*{\phi_{\rho}(\overline{A}),\phi_{\rho}(-\underline{A}) }  > 0,$
  where $\overline{A}=\sup_{\alpha_1\in[-1,1]}g(\alpha_1)-g(\alpha_1-\gamma)$ and $\underline{A}=\inf_{\alpha_1\in[-1,1]}g(\alpha_1)-g(\alpha_1+\gamma)$;
  
      \item Assume that $\C=+\infty$. Then $\phi_{\rho}$ is $\sH_{\mathrm{relu}}$-calibrated with respect to $\ell_{\gamma}$ if and only if 
      $\rho>\gamma$.
  \end{enumerate}
\end{corollary}

Recall that in Theorem~\ref{corollary:sup_convex_NN}
we ruled out the possibility of finding $\sH_{\mathrm{NN}}$-calibrated supremum-based convex surrogate losses with respect to the adversarial $0/1$ loss, where $\sH_{\mathrm{NN}}$ is the class of one layer neural networks. However, we show that the supremum-based $\rho$-margin loss is indeed pseudo-$\sH$-calibrated and $\sH$-calibrated. We state the pseudo-calibration result below and present the proof in Appendix~\ref{app:positive_NN}.
\begin{restatable}{theorem}{PosGN}
\label{Thm:pos1_GN}
  Consider $\rho$-margin loss $\phi_{\rho}(t)=\min\curl*{1,\max\curl*{0,1-\frac{t}{\rho}}},\rho>0$. If $\Lambda W(1-\gamma)\geq \rho$, then the surrogate loss $\tilde{\phi}_{\rho}(f,\bx,y)=\sup_{\bx'\colon \|\bx-\bx'\|\leq \gamma}\phi_{\rho}(y f(\bx'))$ is pseudo-$\sH_{\mathrm{NN}}$-calibrated with respect to $\ell_{\gamma}$.
\end{restatable}
Specifically, $\sH_{\mathrm{NN}}$ with the extra assumption $\Lambda=+\infty$ satisfy the condition 2. in Corollary \ref{corollary:calibration_positive}, as a result we get the following.
\begin{corollary}
\label{corollary:pos1_GN}
 Consider $\rho$-margin loss $\phi_{\rho}(t)=\min\curl*{1,\max\curl*{0,1-\frac{t}{\rho}}},\rho>0$. If $\Lambda=+\infty$, then the surrogate loss $\tilde{\phi}_{\rho}(f,\bx,y)=\sup_{\bx'\colon \|\bx-\bx'\|\leq \gamma}\phi_{\rho}(y f(\bx'))$ is $\sH_{\mathrm{NN}}$-calibrated with respect to $\ell_{\gamma}$.
\end{corollary}

The results of this section suggest that the ramp loss and more generally quasi-concave even losses may be good surrogates for the adversarial $0/1$ loss. However, calibration, in general, is not equivalent to consistency, our eventual goal. In the next section we study conditions under which we can expect these surrogates losses to be $\sH$-consistent as well.

\section{\texorpdfstring{$\sH$}{H}-Consistency}
\label{sec:consistency}
In this section, we study the $\sH$-consistency of surrogate loss functions. The results of the previous section %rule out the possibility of 
suggest that convex losses or supremum-based convex losses would not be $\sH$-consistent. 
However, $\sH$-calibrated quasi-concave even losses, such as the ramp loss present an intriguing possibility. In fact, the recent work of \citet{pmlr-v125-bao20a} made a claim that since quasi-concave even losses are $\sH_{\mathrm{lin}}$-calibrated they are also $\sH_{\mathrm{lin}}$-consistent. We first present a result that falsifies this claim. In fact, our result stated below shows that without assumptions on the data distribution, no continuous margin based loss or a continuous supremum-based surrogate could be $\sH_{\mathrm{lin}}$-consistent.

% We first give a series of negative results showing that many of the calibrated
% losses considered in the previous section are in fact not consistent, in the 
% absence of any distributional assumption. 
%

\subsection{Negative results}
\label{sec:consistency_negative}
% In the absence of any realizability-type assumption, many of the $\sH$-calibrated losses of Section~\ref{sec:calibration_positive} are not $\sH$-consistent with respect to the adversarial $0/1$ loss.
\begin{restatable}{theorem}{ConsistentLinear}
\label{Thm:consistent_linear}
No continuous margin-based loss function $\phi$ is  $\sH_{\mathrm{lin}}$-consistent with respect to $\ell_{\gamma}$.
\end{restatable}
This theorem is proved in Appendix~\ref{app:consistent_linear}. In order to establish the theorem, we carefully design a distribution on the unit disk where the label of each example $\bx$ is first generated as $\text{sgn}(\bw^* \cdot \bx)$ and then flipped independently with a carefully chosen probability. It is crucial that this flipping probability is asymmetric thereby ensuring that for the resulting joint distribution, $\bw^*$ remains the optimal linear classifier according to $\ell_\gamma$, but any continuous surrogate is led astray to a classifier that is far from $\bw^*$. 
% which sets the labels of $\bx$ with non-symmetric flipping probabilities. As shown by the simple algebra computation, this distribution would lead the surrogate loss astray due to the non-symmetry.
In particular, Theorem~\ref{Thm:consistent_linear} contradicts the $\sH$-consistency 
claim of \citet{pmlr-v125-bao20a} for quasi-concave even
losses when $\sH$ is the family of linear functions. Furthermore, the theorem can be easily extended to rule out $\sH$-consistency of supremum-based surrogates as well.
\begin{restatable}{theorem}{ConsistentSupLinear}
\label{Thm:consistent_sup_linear}
For continuous and non-increasing margin-based loss $\phi$,
surrogates of the form \[\tilde{\phi}(f,\bx,y)=\sup\limits_{\bx'\colon \|\bx-\bx'\|\leq \gamma}\phi(y f(\bx'))\]
are not $\sH_{\mathrm{lin}}$-consistent with respect to $\ell_{\gamma}$.
\end{restatable}
\begin{proof}
As shown by \citet{awasthi2020adversarial}, for a continuous and non-increasing margin-based loss $\phi$, when $f\in \sH_{\mathrm{lin}}$, the supremum-based surrogate loss
can be expressed as follows:
\begin{equation*}
	\tilde{\phi}(f,\bx,y)=\sup\limits_{\bx'\colon \|\bx-\bx'\|\leq \gamma}\phi(y f(\bx'))=\phi\left(\inf\limits_{\|\bs\|\leq 1}\left(yf(\bx+\gamma \bs)\right)\right)=\phi(y(\bw \cdot \bx)-\gamma)=\psi(y(\bw \cdot\bx))\,,
\end{equation*}
where $\psi(t)=\phi(t-\gamma)$ is also a continuous margin-based loss. In view of Theorem~\ref{Thm:consistent_linear}, 
we conclude that the supremum-based surrogate loss $\tilde{\phi}$ is also not $\sH_{\mathrm{lin}}$-consistent with respect to $\ell_{\gamma}$.
\end{proof}

\subsection{Positive results}
\label{sec:consistency_positive}

In this section, we investigate the nature of the assumptions on the data distributions that may lead to $\sH$-consistency of surrogate losses. We take inspiration from the work of \citet{long2013consistency} and \citet{zhang2020bayes} who study 
$\sH$-consistency for the standard $0/1$ loss. These studies establish consistency under a realizability assumption on the data distribution stated below that requires the Bayes ($\ell_0$,$\sH$)-risk to be zero. 
\begin{definition}[$\sH$-realizability]
A distribution $\sP$ over $\sX\times\sY$ is $\sH$-realizable if it labels points according to a deterministic model in $\sH$, i.e., if $\exists f\in \sH$ such that $\mathbb{P}_{(\bx,y)\sim \sP}(\sgn(f(\bx))=y)=1$. 
\end{definition}
Similar to $\sH$-realizability, we will assume that, under the data distribution, the Bayes ($\ell_{\gamma}$,$\sH$)-risk is zero. 
% It is clear that a distribution $\sP$ is $\sH$-realizable if and only if $\cR^*_{\ell_{0},\sH} = 0$. Therefore, our assumption $\cR^*_{\ell_{\gamma},\sH} = 0$ 
% on distribution $\sP$ implies that $\sP$ is $\sH$-realizable.
We show that the $\sH$-calibrated losses studied in previous sections are $\sH$-consistent under natural conditions along with the realizability assumption.

\subsubsection{Non-supremum-based surrogates}

% Theorem: if surrogate loss is minimized in H (or minimized modulo eta)
% and is H-calibrated, then it is H-consistent (resp. H-consistent
% modulo eta).

\begin{restatable}{theorem}{CalibrateConsistentNonsup}
\label{Thm:calibrate_consistent_nonsup}
Let $\sP$ be a distribution over $\sX \times \sY$ and $\sH$ a hypothesis set for which $\cR^*_{\ell_{\gamma}, \sH}=0$. Let $\phi$ be a margin-based loss. If for $\eta \geq 0$, there exists $f^* \in \sH \subset \sH_{\mathrm{all}}$ such that $\cR_{\phi}(f^*)\leq \cR^*_{\phi, \sH_{\mathrm{all}}} + \eta< +\infty$ and $\phi$ is $\sH$-calibrated\footnote{The theorem still holds if uniform $\sH$-calibration is replaced by weaker non-uniform $\sH$-calibration \citep[Definition 2.7]{steinwart2007compare}, since the proof only makes use of the weaker non-uniform property.} with respect to $\ell_{\gamma}$, then for all $\epsilon > 0$ there exists $\delta > 0$ such that for all $f\in\sH$ we have
    \[
        \cR_{\phi}(f)+\eta <\cR_{\phi,\sH}^*+\delta \implies \cR_{\ell_{\gamma}}(f)<\cR_{\ell_{\gamma},\sH}^*+\epsilon.
    \]
\end{restatable}
The proof of Theorem~\ref{Thm:calibrate_consistent_nonsup} is presented in Appendix~\ref{app:calibrate_consistent_nonsup}. Using Corollary~\ref{corollary:rho_margin_calibtartion} in Section~\ref{sec:calibration_rho_margin} and Theorem~\ref{Thm:calibrate_consistent_nonsup} above, we immediately conclude that the calibrated $\rho$-margin loss in Section~\ref{sec:calibration_rho_margin} is consistent with respect to $\ell_{\gamma}$ for all distributions that satisfy our realizability assumptions.
\begin{theorem}
\label{Thm:rho_margin_consistent}
  Consider the $\rho$-margin loss $\phi_{\rho}(t)=\min\curl*{1,\max\curl*{0,1-\frac{t}{\rho}}},~\rho>0$. Then,
  \begin{enumerate}
      \item 
      Let $g$ be a non-decreasing and continuous function. Assume $\C=+\infty$ and $g(-1)+g(1)\geq0$. Let $\overline{A}=\sup_{\alpha_1\in[-1,1]}g(\alpha_1)-g(\alpha_1-\gamma)$ and $\underline{A}=\inf_{\alpha_1\in[-1,1]}g(\alpha_1)-g(\alpha_1+\gamma)$. If
      $\min\curl*{\phi_{\rho}(\overline{A}),\phi_{\rho}(-\underline{A}) }>0$,
      then $\phi_{\rho}$ is $\sH_g$-consistent with respect to $\ell_{\gamma}$ for all distribution $\sP$ over $\sX\times\sY$ that satisfies $\cR^*_{\ell_{\gamma},\sH_{g}}=0$ and there exists $f^*\in\sH_{g}$ such that $\cR_{\phi}(f^*)= \cR^*_{\phi,\sH_{\mathrm{all}}}<\infty$.
      
      \item If $\rho>\gamma$, then $\phi_{\rho}$ is $\sH_{\mathrm{relu}}$-consistent with respect to $\ell_{\gamma}$ for all distribution $\sP$ over $\sX\times\sY$ that satisfies $\cR^*_{\ell_{\gamma},\sH_{\mathrm{relu}}}=0$ and there exists $f^*\in\sH_{\mathrm{relu}}$ such that $\cR_{\phi}(f^*)= \cR^*_{\phi,\sH_{\mathrm{all}}}<\infty$.
      
  \end{enumerate}
\end{theorem}

\subsubsection{Supremum-based surrogates}
We can also extend the above to obtain $\sH$-consistency of supremum-based convex surrogates. However we need the stronger condition that $\cR_\phi$ is minimized exactly inside $\sH$.

\begin{restatable}{theorem}{CalibrateConsistentSup}
\label{Thm:calibrate_consistent_sup}
Given a distribution $\sP$ over $\sX\times\sY$ and a hypothesis set $\sH$ such that $\cR^*_{\ell_{\gamma},\sH}=0$. Let $\phi$ be a non-increasing margin-based loss. If there exists $f^*\in\sH\subset\sH_{\mathrm{all}}$ such that $\cR_{\phi}(f^*)=\cR^*_{\phi,\sH_{\mathrm{all}}}<\infty$ and $\tilde{\phi}(f,\bx,y)=\sup_{\bx'\colon \|\bx-\bx'\|\leq \gamma}\phi(y f(\bx'))$ is $\sH$-calibrated\footnote{The theorem still holds if uniform $\sH$-calibration is replaced by weaker non-uniform $\sH$-calibration \citep[Definition 2.7]{steinwart2007compare}, since the proof only makes use of the weaker non-uniform property.} with respect to $\ell_{\gamma}$, then for all $\epsilon>0$ there exists $\delta>0$ such that for all $f\in\sH$ we have
    \[
        \cR_{\tilde{\phi}}(f) <\cR_{\tilde{\phi},\sH}^*+\delta \implies \cR_{\ell_{\gamma}}(f)<\cR_{\ell_{\gamma},\sH}^*+\epsilon.
    \]
\end{restatable}
The proof of Theorem~\ref{Thm:calibrate_consistent_sup} is presented in Appendix~\ref{app:calibrate_consistent_nonsup}. Again, when combined with Corollary~\ref{corollary:pos1_GN} in Section~\ref{sec:calibration_rho_margin} we conclude that the $\sH_{\mathrm{NN}}$-calibrated supremum-based $\rho$-margin loss is also $\sH_{\mathrm{NN}}$-consistent with respect to $\ell_{\gamma}$ for all distributions that satisfy our realizability assumptions.
\begin{theorem}
\label{Thm:sup_rho_margin_consistent}
  Consider the $\rho$-margin loss $\phi_{\rho}(t)=\min \curl*{1, \max\curl*{0, 1 - \frac{t}{\rho}}},\rho>0$. If $\Lambda=+\infty$, then $\tilde{\phi}_{\rho}(f,\bx,y)=\sup_{\bx'\colon \|\bx - \bx'\|\leq \gamma}\phi_{\rho}(y f(\bx'))$ is $\sH_{\mathrm{NN}}$-consistent with respect to $\ell_{\gamma}$ for all distributions $\sP$ over $\sX\times\sY$ that satisfy: $\cR^*_{\ell_{\gamma},\sH_{\mathrm{NN}}}=0$ and there exists $f^*\in\sH_{\mathrm{NN}}$ such that $\cR_{\phi}(f^*)= \cR^*_{\phi,\sH_{\mathrm{all}}}<\infty$.
\end{theorem}

We prove the above theorems by building upon the framework of \citet{steinwart2007compare}. The goal is to show that under the assumptions of the theorems, the adversarial $0/1$ loss and the surrogate loss become $\sP$-minimizable which is enough to show consistency. In Theorem~\ref{Thm:calibrate_consistent_nonsup} we use the fact that the adversarial $0/1$ Bayes risk is zero to obtain $\sP$-minimizability of the adversarial $0/1$ loss. We also show that the framework of \citet{steinwart2007compare} can be extended to only require $\sP$-minimizability of the surrogate up to an error of $\eta$, and the consistency guarantees degrade smoothly with $\eta$. To prove Theorem~\ref{Thm:calibrate_consistent_sup}, we prove a general result (Lemma~\ref{lemma:sup-P mini}) that if the adversarial $0/1$ Bayes risk is zero then $\sP$-minimizability of a surrogate implies $\sP$-minimizability of its supremum based counterpart.

\section{Experiments}
\label{sec:experiments}

\begin{table}[t]
\centering
\begin{tabular}{@{\hspace{-2.5cm}}c@{\hspace{-.3cm}}c@{\hspace{-0cm}}}
 \begin{minipage}{0.7\columnwidth}
 \centering
 \subfigure[sample 1000][c]{
    \label{fig:sample1000}
    \includegraphics[width=0.4\columnwidth]{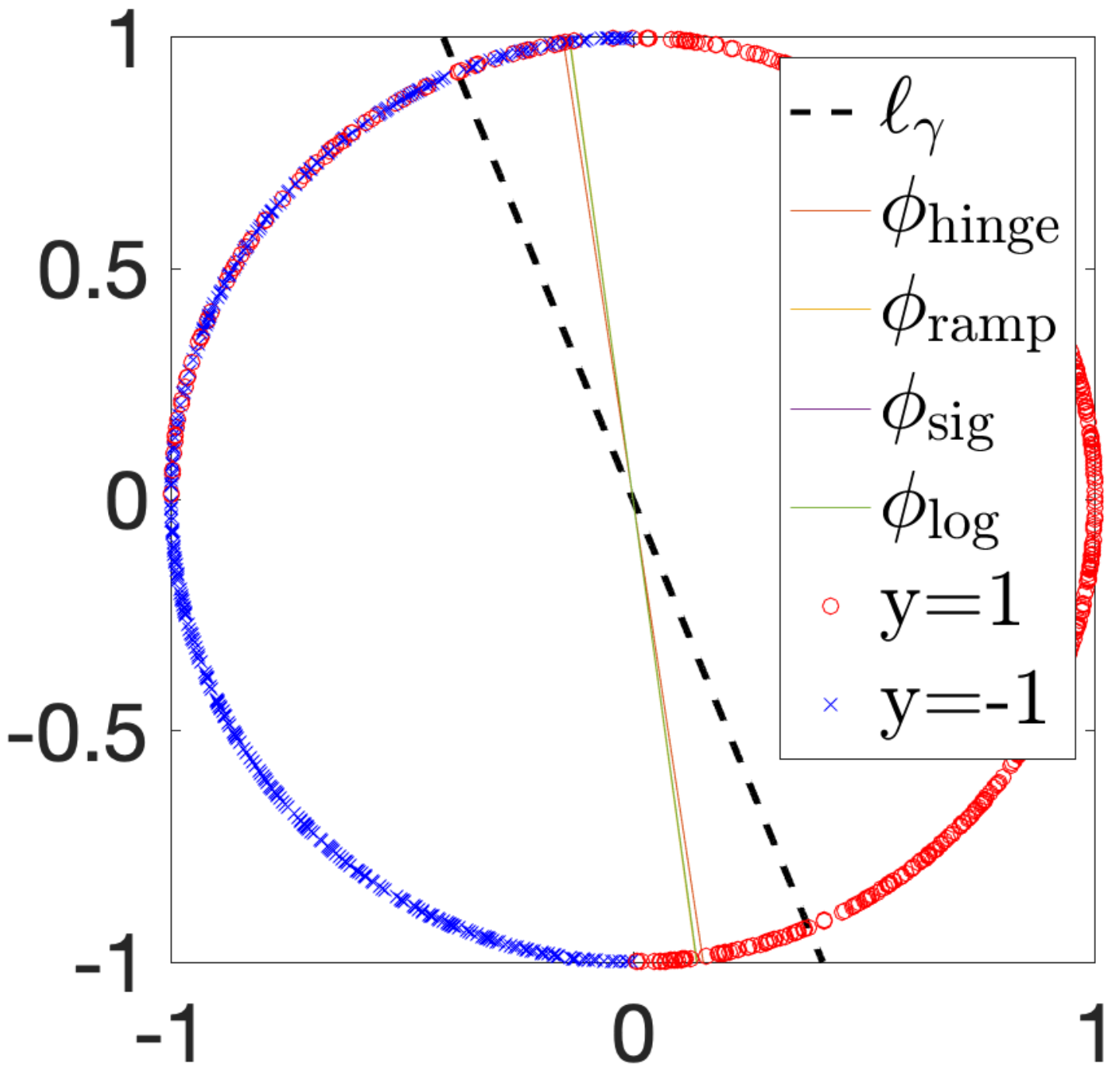}
 }
 \subfigure[sample 2000][c]{
    \label{fig:sample2000}
    \includegraphics[width=0.4\columnwidth]{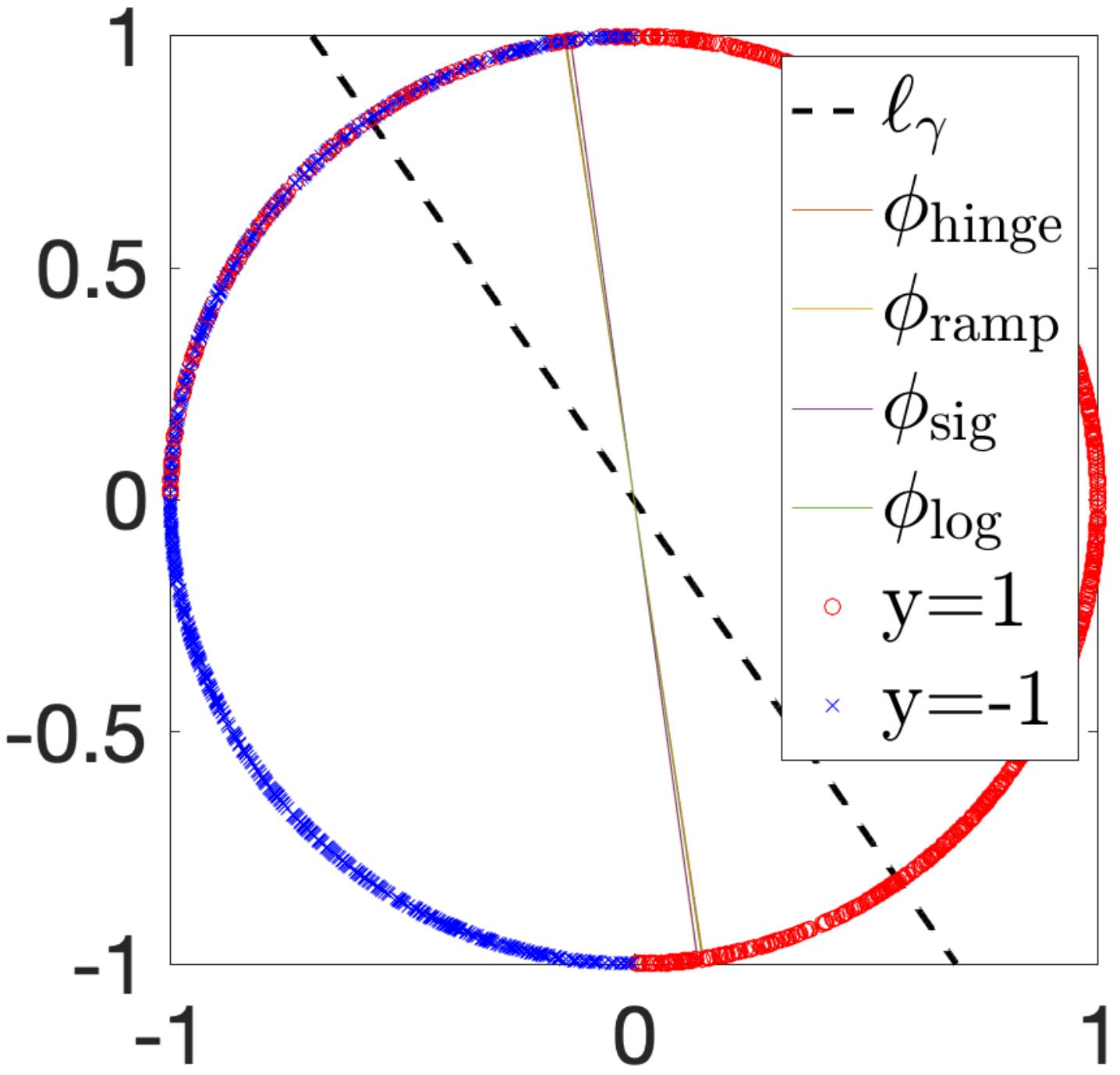}
 }
 \captionof{figure}{Unit Circle}
 \label{fig:unit_circle}
  \end{minipage}
&
\raisebox{-.55cm}{
  \begin{minipage}{0.35\columnwidth}
    \centering
    \includegraphics[width=1.2\columnwidth]{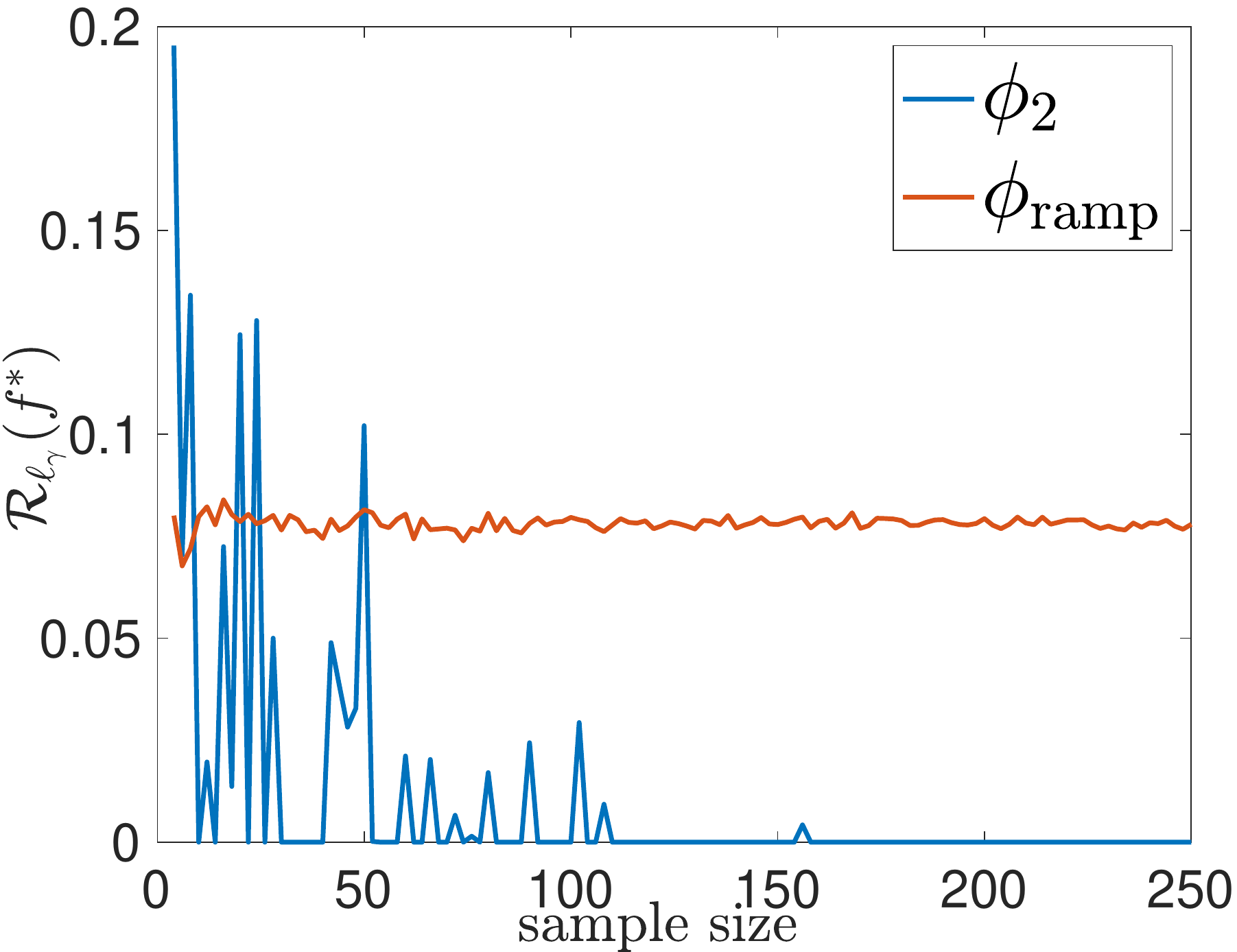}
    \captionof{figure}{Adversarial generalization error of consistent loss and calibrated inconsistent loss against sample size}
 \label{fig:compare}
  \end{minipage}
  }
  \end{tabular}
  \vskip -.15in
\end{table}

We present experiments on simulated data to support our theoretical findings. The goal is two fold. First, we empirically demonstrate that indeed calibrated surrogates in \citep{pmlr-v125-bao20a} may not be $\sH$-consistent unless assumptions on the data distribution are made, even when $\sH$ is the class of linear functions. This is consistent with our negative result in Theorem \ref{Thm:consistent_linear} and provides an empirical counterexample to the claim made in \citep{pmlr-v125-bao20a}. Secondly, we study the necessity of the realizability assumptions that we make in Section \ref{sec:consistency_positive} to establish $\sH$-consistency of quasi-concave even surrogates. We generate data points $\bx \in \mathbb{R}^2$ on the unit circle and consider $\sH$ to be linear models $\sH_{\mathrm{lin}}$. We denote $f(\bx)=\bw\cdot \bx$, $\bw=(\cos(t),\sin(t))^{\top},t\in [0,2\pi), f\in\sH_{\mathrm{lin}}$. %We consider two 
All risks in the experiments are approximated by their empirical counterparts computed over $10^7$ i.i.d. samples from the distribution. 

To demonstrate the need for assumptions on the data distribution for $\sH$-consistency, we construct a scenario we call the \textbf{Unit Circle} case. 
% We consider two different distributions, \textbf{Unit Circle} where $\cR^*_{\ell_{\gamma},\sH_{\mathrm{lin}}}\neq 0$ and \textbf{Segments} where $\cR^*_{\ell_{\gamma},\sH_{\mathrm{lin}}}= 0$. 
We consider four surrogates: $\phi_{\mathrm{hinge}}$, $\phi_{\mathrm{ramp}}$, $\phi_{\mathrm{sig}}$ and $\phi_{\mathrm{log}}$ defined in Appendix~\ref{app:details of experiments}. In general, we refer all of these surrogates as $\phi_{\mathrm{sur}}$. We generate data points $\bx$ from the uniform distribution on the unit circle. Denote $\bx=(\cos(\theta),\sin(\theta))^{\top},\theta\in [0,2\pi)$. Set the label of a point $\bx$ as follows: if $\theta \in \left(\frac{\pi}{2},\pi\right)$, then $y=-1$ with probability $\frac34$ and $y=1$ with probability $\frac14$; if $\theta \in \left(0,\frac{\pi}{2}\right) \text{ or } \left(\frac{3\pi}{2},2\pi\right)$, then $y=1$; if $\theta \in \left(\pi,\frac{3\pi}{2}\right)$, then $y=-1$. Set $\gamma=\frac{\sqrt{2}}{2}$. 

In this case the optimal Bayes ($\ell_{\gamma}$,$\sH_{\mathrm{lin}}$)-risk  $\cR^*_{\ell_{\gamma},\sH_{\mathrm{lin}}}\approx 0.5000\neq 0$ and is achieved by $w_{\ell_{\gamma}}=(\cos(\theta), \sin(\theta))^{\top}$ with $\theta\approx 0.7855$. The results obtained by optimizing the different surrogate losses are in Table~\ref{tab:circle}(a) and the plots for $1000$ samples and $2000$ samples are shown in Figure~\ref{fig:unit_circle}. Table~\ref{tab:circle}(a) shows that neither calibrated nor non-calibrated (convex) surrogates are $\sH_{\mathrm{lin}}$-consistent with respect to $\ell_{\gamma}$ for this distribution. Figure~\ref{fig:unit_circle} shows that the classifiers obtained by optimizing the four surrogates are almost the same but deviate a lot from the optimal Bayes classifier for $\ell_{\gamma}$. This shows that indeed calibrated surrogates may not be consistent, and contradicts Figure $12$ of \citet{pmlr-v125-bao20a}. The discrepancy results from the incorrect calculation of the adversarial Bayes risk in \citep{pmlr-v125-bao20a}.\footnote{Private Communication.}

Next, we justify the realizability assumptions made in Section \ref{sec:consistency_positive} for obtaining $\sH$-consistency of surrogate losses. In order to do this we construct a scenario that we call as the \textbf{Segments} case. Here, we consider six surrogates, the four studied above and two more surrogates $\phi_1$ and $\phi_2$ defined in Appendix~\ref{app:details of experiments}. The loss $\phi_1$ is a convex loss and $\phi_2$ is the $\rho$-margin ramp loss for some $\rho > \gamma$. In general, we refer all of these surrogates to $\phi_{\mathrm{sur}}$. We show in Appendix~\ref{app:theoretical analysis of experiments}, $\phi_{\mathrm{hinge}}$, $\phi_{\mathrm{log}}$ and $\phi_1$ are not calibrated while $\phi_{\mathrm{ramp}}$, $\phi_{\mathrm{sig}}$ and $\phi_2$ are calibrated with respect to $\ell_{\gamma}$. 

% $\phi_1$ is not $\sH_{\mathrm{lin}}$-consistent for distribution \textbf{Segments} while $\phi_2$ is $\sH_{\mathrm{lin}}$-consistent.

% \noindent\textbf{Unit Circle ($\cR^*_{\ell_{\gamma},\sH_{\mathrm{lin}}}\neq 0$).} 

\begin{table}[t]
\vskip -.2in
\centering
\begin{tabular}{cc}
\raisebox{.33cm}{
\resizebox{.42\linewidth}{!}{
    \begin{tabular}{@{\hspace{0.05cm}} @{\hspace{0.05cm}}l@{\hspace{0.05cm}}|@{\hspace{0.05cm}}c@{\hspace{0.05cm}}|@{\hspace{0.05cm}}c@{\hspace{0.05cm}}|@{\hspace{0.05cm}}c@{\hspace{0.05cm}}|@{\hspace{0.05cm}}c@{\hspace{0.05cm}}}
% 	\hline
	$\phi_{\mathrm{sur}}$ & $\cR_{\ell_{\gamma}}(f^*)$ & $\theta_{\phi_{\mathrm{sur}}}$ & $\sH_{\mathrm{lin}}$-cal. & $\sH_{\mathrm{lin}}$-cons.\\
	\hline
	$\phi_{\mathrm{hinge}}$ & 0.5257 & 0.1420 & \xmark & \xmark \\
	$\phi_{\mathrm{ramp}}$ & 0.5263 & 0.1288 & \cmark & \xmark\\
	$\phi_{\mathrm{sig}}$ & 0.5261 & 0.1320 & \cmark & \xmark\\
	$\phi_{\mathrm{log}}$ & 0.5258 & 0.1414 & \xmark & \xmark \\
% 	\hline
	\end{tabular}
}} &
\resizebox{.48\linewidth}{!}{
    \begin{tabular}{@{\hspace{0.05cm}}l@{\hspace{0.05cm}}|@{\hspace{0.05cm}}c@{\hspace{0.05cm}}|@{\hspace{0.05cm}}c@{\hspace{0.05cm}}|@{\hspace{0.05cm}}c@{\hspace{0.05cm}}|@{\hspace{0.05cm}}c@{\hspace{0.05cm}}|@{\hspace{0.05cm}}c@{\hspace{0.05cm}}}
	%\hline
	$\phi_{\mathrm{sur}}$  & $\cR_{\ell_{\gamma}}(f^*)$ & $\cR_{\phi_{\mathrm{sur}}}(f^*)$ & $\theta_{\phi_{\mathrm{sur}}}$ & $\sH_{\mathrm{lin}}$-cal. & $\sH_{\mathrm{lin}}$-cons.\\
	\hline
	$\phi_{\mathrm{hinge}}$ & 0.0781 & 0.6907 & 1.3548 & \xmark & \xmark \\
	$\phi_{\mathrm{ramp}}$ & 0.0781 & 0.3454 & 1.3548 & \cmark & \xmark\\
	$\phi_{\mathrm{sig}}$ & 0.0777 & 0.4247 & 1.3498 & \cmark & \xmark\\
	$\phi_{\mathrm{log}}$ & 0.0763 & 0.8078 & 1.3341 & \xmark & \xmark \\
	$\phi_1$ & 0.0111 & 0 & $\frac{\pi}{6}$ & \xmark & \xmark \\
	$\phi_2$ & 0 & 0 & 0 & \cmark & \cmark\\
	%\hline
	\end{tabular}
}\\
(a) & (b)
\end{tabular}
\vskip -.15in
\caption{(a) Unit Circle; (b) Segments.}
\label{tab:circle}
\vskip -.2in
\end{table}
% \noindent\textbf{The Segments ($\cR^*_{\ell_{\gamma},\sH_{\mathrm{lin}}}=0$).}
We consider the following data distribution: $\mathbb P(Y=1)=\mathbb P(Y=-1)=\frac12$, and $X\mid Y=1$ is the uniform distribution on the line segment $\curl*{(\hat{\gamma},z)\mid z\in[0,\sqrt{1-\hat{\gamma}^2}]}$ and $X\mid Y=-1$ is the uniform distribution on the line segment $\curl*{(-\hat{\gamma},z)\mid z\in[-\sqrt{1-\hat{\gamma}^2},0]}$ where $\hat{\gamma}=\gamma+\frac{1-\gamma}{100} = \frac{1+99\gamma}{100}$, $\gamma\in(0,1)$.
Finally, we set $\gamma=0.1$. Let $\bw^* = (1,0)^{\top}$. It is easy to check that $\bw^*$ achieves the optimal adversarial Bayes risk ($\ell_{\gamma}$,$\sH_{\mathrm{lin}}$)-risk $\cR^*_{\ell_{\gamma},\sH_{\mathrm{lin}}}=0$.

% Let $\bw^*=(1,0)^{\top}$, we have $\mathbb{E}_{(X,Y)}[\ell_{\gamma}(Yw^*\cdot X)]=\mathbb{E}_{(X,Y)}[\mathds{1}_{Yw^*\cdot X
% \leq\gamma}]=\mathbb{E}_{(X,Y)}[\mathds{1}_{YX_1
% \leq\gamma}]=\mathbb{E}_{(X,Y)}[\mathds{1}_{\hat{\gamma}
% \leq\gamma}]=0$. Therefore, Bayes ($\ell_{\gamma}$,$\sH_{\mathrm{lin}}$)-risk $\cR^*_{\ell_{\gamma},\sH_{\mathrm{lin}}}=0$.

The results for six different surrogate losses are in Table~\ref{tab:circle}(b).
For $\phi_{\mathrm{hinge}}$, $\phi_{\mathrm{ramp}}$, $\phi_{\mathrm{sig}}$ and $\phi_{\mathrm{log}}$, the Bayes ($\phi_{\mathrm{sur}}$,$\sH_{\mathrm{lin}}$)-risk $\cR^*_{\phi_{\mathrm{sur}},\sH_{\mathrm{lin}}}\neq0$. Table~\ref{tab:circle}(b) shows that they are not $\sH_{\mathrm{lin}}$-consistent with respect to $\ell_{\gamma}$. For $\phi_1$ and $\phi_2$, the Bayes ($\phi_{\mathrm{sur}}$,$\sH_{\mathrm{lin}}$)-risk $\cR^*_{\phi_{\mathrm{sur}},\sH_{\mathrm{lin}}}=0$. Table~\ref{tab:circle}(b) shows that $\phi_1$ is not $\sH_{\mathrm{lin}}$-consistent~(recall that $\phi_1$ is not calibrated) but $\phi_2$ is $\sH_{\mathrm{lin}}$-consistent for this distribution.

Hence even when $\cR^*_{\ell_{\gamma},\sH_{\mathrm{lin}}}=0$, unless a condition is also imposed on $\cR^*_{\phi_{\mathrm{sur}},\sH_{\mathrm{lin}}}$, one cannot expect consistency, thus justifying our realizability assumption.
%for consistency results are necessary. 
Note that $\cR^*_{\phi_{\mathrm{sur}},\sH_{\mathrm{lin}}}=\cR^*_{\ell_{\gamma},\sH_{\mathrm{lin}}}=0$ is a special case satisfying the conditions in Theorem~\ref{Thm:calibrate_consistent_nonsup} when $\eta=0$. 
% Furthermore, when $\cR^*_{\phi_{\mathrm{sur}},\sH_{\mathrm{lin}}}=\cR^*_{\ell_{\gamma},\sH_{\mathrm{lin}}}=0$, the convex (non-calibrated) surrogates $\phi_1$ can still not be consistent.
For this distribution, $\phi_{\mathrm{ramp}}$ is not $\sH_{\mathrm{lin}}$-consistent while $\phi_2$ is $\sH_{\mathrm{lin}}$-consistent, although both are calibrated. We compare them in Figure~\ref{fig:compare}, showing that minimizing $\sH_{\mathrm{lin}}$-consistent surrogate $\phi_2$ minimizes the generalization error for large sample sizes but the same does not hold for $\sH_{\mathrm{lin}}$ non-consistent surrogate $\phi_{\mathrm{ramp}}$.
\section{Conclusion}
\label{sec:conclusion}
We presented a detailed study of calibration and
consistency for adversarial robustness. These
results can help guide the design of algorithms
for learning robust predictors, an increasingly important
problem in applications. Our theoretical results show in particular
that many of the surrogate losses typically used
in practice do not benefit from any guarantee.
Our empirical results further illustrate that in
the context of a general example.
Our results also show that some of the calibration results 
presented in previous work do not bear any 
significance, since we prove that in fact they
do not guarantee consistency. Instead, we give a 
series of positive calibration and
consistency results for several families of 
surrogate functions, under some realizability
assumptions.

% Acknowledgments---Will not appear in anonymized version
\acks{We are grateful to the COLT reviewers for their comments.}

\bibliography{advcolt}

\newpage
\appendix

\renewcommand{\contentsname}{Contents of Appendix}
\tableofcontents
\addtocontents{toc}{\protect\setcounter{tocdepth}{3}} 
\clearpage

\section{Related Work}
\label{app:related}

The notions of calibration and consistency with respect to the $0/1$
loss have been widely studied in the statistical learning theory
literature to analyze the properties of surrogate losses
\citep{Zhang2003,bartlett2006convexity}.
\citet{bartlett2006convexity} showed that margin-based convex
surrogates, that is mappings of the form
$(f, \bx, y) \mapsto \phi(yf(\bx))$, where $f$ is a real-valued predictor
and $\phi\colon \Rset \to \Rset_{+}$ a function differentiable at $0$ with
$\phi'(0) < 0$, are calibrated with respect to the class of all measurable functions. 
Extensions of calibration and consistency to multi-class settings have
also been studied \citep{tewari2007consistency}.
In the special case of the $0/1$ loss and margin-based convex
surrogates, calibration immediately implies consistency for the class of all measurable functions.
One can then even derive quantitative bounds relating the excess
$\phi$-risk to the excess $0/1$ loss of any function $f$
\citep{Zhang2003,bartlett2006convexity}.

The case of adversarial loss is more complex. This is because, in
particular, the loss of a predictor $f$ at point $\bx$ does not just
depend on its value $f(\bx)$ at that point but also on its values in a
neighborhood of $\bx$.
\citet{steinwart2007compare} proposed a general framework to study and
characterize calibration and consistency, in particular via a
\emph{calibration function}.  He also defined a
\emph{$\sP$-minimizability} condition under which calibration implies
consistency. But, while $\sP$-minimizability holds for 
the $0/1$ loss and margin-based convex surrogates
 over the class of all measurable
functions, the condition does not hold in general for the adversarial
loss.
Our work borrows tools from the work of
\citet{steinwart2007compare}. However, to establish $\sH$-consistency
in the context of the adversarial loss, additional insights are needed
and often stronger assumptions on the data distribution are
required. These assumptions are captured in the notion of
\emph{realizable $\sH$-consistency} that requires that the optimal
risk of both the $0/1$ loss and the surrogate loss being achieved
inside the class $\sH$. Our positive results for $\sH$-consistency
rely on similar but weaker assumptions. \citet{long2013consistency}
gave examples of surrogate losses that are not $\sH$-consistent when
$\sH$ is the class of all measurable functions but satisfy realizable
$\sH$-consistency when $\sH$ is the class of linear functions.
\citet{zhang2020bayes} studied the notion of \emph{improper realizable
  $\sH$-consistency} of linear classes where the surrogate $\phi$ can
be optimized over a larger class such as that of piecewise linear
functions.

These notions of calibration and consistency are relatively unexplored
for the robust $0/1$ loss.  \citet{pmlr-v125-bao20a} recently
initiated the study of these notions for the robust loss. In
particular, the authors studied the $\gamma$-margin loss defined by:
$\bar{\ell}_\gamma(f, \bx, y) = \mathds{1}_{yf(\bx) \leq \gamma}$.
This loss function coincides with the adversarial loss (only) in the
special case where linear classifiers with adversarial perturbations
measured in $\ell_2$ norm are considered.
The authors showed that, when $\sH$ is linear,
convex surrogates are not $\sH$-calibrated and proposed a class of
\emph{quasi-concave even} $\sH$-calibrated surrogates. 

Our positive results for $\sH$-calibration significantly extend those
beyond linear hypothesis sets. More importantly,
\citet{pmlr-v125-bao20a} incorrectly concluded that $\sH$-calibration
of quasi-concave even surrogate losses implies their
$\sH$-consistency. Our negative results falsify this claim and in fact
rule out the $\sH$-consistency of a large class of surrogates, unless
assumptions on the data distribution are imposed. Finally, while the
results of \citet{pmlr-v125-bao20a} do imply that quasi-concave even
surrogates are $\sH$-consistent with respect to the
$\bar{\ell}_\gamma$ loss over the set of all measurable functions,
this does not provide insights into the adversarial loss since the two
losses only coincide for linear hypothesis sets.

There has also been recent works on theoretically understanding different aspects of adversarial robustness. \citet{tsipras2018robustness} give constructions under which every classifier with small $0/1$ loss has a large adversarial $0/1$ loss thereby pointing to a tension between the two criteria. This has been tradeoff has been explored in subsequent work \citep{zhang2019theoretically, carmon2019unlabeled}. \citet{bubeck2018adversarial}, \citet{bubeck2018adversarial2} and \citet{awasthi2019robustness} quantify computational bottlenecks in learning classifiers with small adversarial loss. There has also been a line of work analyzing the sample complexity of optimizing adversarial surrogate losses using notions of VC-dimension and Rademacher complexity appropriately extended to the adversarial case \citep{YinRamchandranBartlett2019, khim2018adversarial, awasthi2020adversarial, montasser2019vc, cullina2018pac}. Another recent line of concerns constructing computationally efficient adversarially robust classifiers for linear classifiers \citep{diakonikolas2020complexity} and exploring the connections between adversarial learning and agnostic PAC learning \citep{montasser2020reducing}. 
Finally, an alternative adversarial setting 
has been theoretically studied in 
\citep{feige2015learning, feige2018robust, attias2018improved}, where
the adversary has at his disposal a finite set of perturbations 
for each input.

\newpage
\section{Details of Experiments}
\label{app:details of experiments}

As shown by \citet{pmlr-v125-bao20a}, the adversarial $0/1$ loss $\ell_{\gamma}=\mathds{1}_{yf(\bx) \leq \gamma}$ when $f\in \sH_{\mathrm{lin}}$.
In this experiment, we approximate $\cR^*_{\ell_{\gamma},\sH_{\mathrm{lin}}}$ over a grid. For surrogate losses, we  approximate 
$f^*=\argmin_{f\in\sH_{\mathrm{lin}}} \cR_{\phi_{\mathrm{sur}}}(f)$ over the same grid.
\subsection{Definition of Surrogates}
\label{app:definition of surrogates}
\begin{itemize}[itemsep=-1mm]
\item Shifted Hinge loss: $\phi_{\mathrm{hinge}}=\max\curl*{0,1-t+0.2}$;
\item Shifted Ramp loss: 
$\phi_{\mathrm{ramp}} =\min\curl*{1,\max\curl*{0,\frac{1-t+0.2}{2}}}$;
\item Shifted Sigmoid loss:
$\phi_{\mathrm{sig}}=\frac{1}{1+e^{t-0.2}}$;
\item Shifted Logistic loss: $\phi_{\mathrm{log}}=\log_2(1+e^{t-0.2})$;
\item One convex loss:
$\phi_1(t) = \max\curl*{0,\frac{\gamma}{2}-t}$; and
\item $\rho$-margin loss:
$\phi_2(t)=\min\curl*{1,\max\curl*{0,1-\frac{t}{\hat{\gamma}}}}$ for $\hat{\gamma} > \gamma$. 
\end{itemize}

\subsection{Theoretical Analysis of Surrogates}
\label{app:theoretical analysis of experiments}
$\phi_{\mathrm{hinge}}$, $\phi_{\mathrm{log}}$, and $\phi_1$ are convex surrogates and thus are not calibrated with respect to $\ell_{\gamma}$ by Corollary 9 of \citep{pmlr-v125-bao20a}.
However, $\phi_{\mathrm{ramp}}$, $\phi_{\mathrm{sig}}$ and $\phi_2$ are quasi-concave even losses and calibrated with respect to $\ell_{\gamma}$ since they satisfy the conditions in Theorem 11 of \citep{pmlr-v125-bao20a}.

Note that $\mathbb{E}_{(X,Y)}[\phi_2(Y\bw\cdot X)]=0$ if and only if $w=(1,0)^{\top}$. Therefore, $\phi_2$ is $\sH_{\mathrm{lin}}$-consistent for the distribution \textbf{Segments}. However, for $w=(1,0)^{\top}$ or $w=(\cos(\theta), \sin(\theta))^{\top}$ where $\theta=\frac{\pi}{6}$, we have $\mathbb{E}_{(X,Y)}[\phi_1(Y\bw\cdot X)]=0$. Note when $\bw=(\cos(\theta), \sin(\theta))^{\top}$ where $\theta=\frac{\pi}{6}$, we have $\mathbb{E}_{(X,Y)}[\ell_{\gamma}(Y\bw\cdot X)]\neq0$. Therefore, $\phi_1$ is not $\sH_{\mathrm{lin}}$-consistent for the distribution \textbf{Segments}.

\newpage
\section{Deferred Proofs}
\label{app:deferred-proofs}
\subsection{Proof of 
Theorem~\ref{Thm:calibration_definition_equivalent}}
\label{app:calibration_definition_equivalent}

\calibrationDefinitionEquivalent*
\begin{proof}
1) First, note that for any $\bx \in \sX$ and
$f \in \sH$, we cannot have both $\ell_\gamma(f, \bx, +1) = 0$ and
$\ell_\gamma(f, \bx, -1) = 0$.  In view of that,
$\cC_{\ell_{\gamma},\sH}(f, \bx, \eta) = \eta \ell_\gamma(f, \bx, +1)
+ (1 - \eta)\ell_\gamma(f, \bx, -1)\geq \min \curl{\eta, 1 -
  \eta}$. By assumption, for any $\bx \in \sX$, there exist $f_+ \in
\sH$ such that $\ell_\gamma(f_+, \bx, +1) = 1$ and $\ell_\gamma(f_+, \bx,
-1) = 0$ and $f_- \in \sH$ such that such that $\ell_\gamma(f_-, \bx,
+1) = 0$ and $\ell_\gamma(f_-, \bx, +1) = 1$, that is
$\cC_{\ell_{\gamma},\sH}(f_+, \bx, \eta) = \eta$ and
$\cC_{\ell_{\gamma},\sH}(f_-, \bx, \eta) = 1 - \eta$. Thus, $ \min
\curl{\eta, 1 - \eta}$ is achieved and we have, for all $\bx \in \sX$,
$\cC_{\ell_{\gamma},\sH}^{*}(\bx, \eta) = \min \curl{\eta, 1 - \eta}$.
This implies $\cC_{\ell_{\gamma},\sH}^*(\eta) =
\cC_{\ell_{\gamma},\sH}^{*}(\bx, \eta) = \min \curl{\eta, 1 - \eta}$
for all $\bx \in \sX$.

2) Given the assumption, for any $\bx \in \sX$, we have
$\cC_{\phi,\sH}^{*}(\bx, \eta) = \inf_{f \in \sH}
\bracket{\eta\phi(f(\bx))+(1-\eta)\phi(-f(\bx))} = \inf_{u \in
  \mathbb{R}} \bracket{\eta\phi(u)+(1-\eta)\phi(-u)}$, which is
is independent of $\bx$. 
This implies $\cC_{\phi,\sH}^*(\eta) = \cC_{\phi,\sH}^{*}(\bx, \eta)$,
for $\bx \in \sX$.

3) By definition of the loss function, for any $\bx\in \sX$ and $f \in \sH$, we
have $\cC_{\tilde \phi_{\rho},\sH}^{*}(f, \bx, \eta) = \eta
\phi_\rho(\underline M(f, \bx,\gamma)) + (1 - \eta) \phi_\rho(-\overline M(f, \bx,\gamma))$,
where $\underline M(f, \bx,\gamma) = \inf_{\bx'\colon \|\bx - \bx'\|\leq\gamma} f(\bx')$
and $\overline M(f, \bx,\gamma) = \sup_{\bx'\colon \|\bx - \bx'\|\leq\gamma} f(\bx')$.
Now, we must have either $\phi_\rho(\underline M(f, \bx,\gamma)) = 1$ or
$\phi_\rho(-\overline M(f, \bx,\gamma)) = 1$. Otherwise, we would have $\underline
M(f, \bx,\gamma) > 0$ and $- \overline M(f, \bx,\gamma) > 0$, but since $\underline M(f, \bx,\gamma) \leq
\overline M(f, \bx,\gamma)$, the first inequality would imply $\overline M(f, \bx,\gamma) >
0$, which would contradict the second inequality.
In view of that, the lower bound $\cC_{\tilde \phi_{\rho},\sH}^{*}(f, \bx, \eta) \geq
\min \curl{\eta, 1 - \eta}$ holds.

By assumption, for any $\bx \in \sX$, there exists $f_-$ such that
$\phi_\rho(\underline M(f_-, \bx,\gamma)) = 0$ and
$\phi_\rho(-\overline M(f_-, \bx,\gamma)) = 1$, that is
$\cC_{\tilde \phi_{\rho},\sH}^{*}(f_-, \bx, \eta) = 1 - \eta$, and
$f_+$ such that
$\phi_\rho(\underline M(f_+, \bx,\gamma)) = 1$ and
$\phi_\rho(-\overline M(f_+, \bx,\gamma)) = 0$, that is
$\cC_{\tilde \phi_{\rho},\sH}^{*}(f_+, \bx, \eta) = \eta$. Thus, the lower
bound is reached and, for any $\bx \in \sX$, we have
$\cC_{\tilde \phi_{\rho},\sH}^{*}(\bx, \eta) = \min \curl{\eta, 1 - \eta}$.
This implies 
$\cC_{\tilde \phi_{\rho},\sH}^*(\eta) =
\cC_{\tilde \phi_{\rho},\sH}^{*}(\bx, \eta) = \min
\curl{\eta, 1 - \eta}$, for any $\bx \in \sX$.
\end{proof}

\subsection{Proof of Theorem~\ref{Thm:calibration_convex_general}}
\label{app:calibration_convex_general}

As shown by~\eqref{eq:supinf01} and \citet{awasthi2020adversarial}, for $f\in\sH_g$, the adversarial $0/1$ loss has the equivalent form
\begin{equation}
\label{eq:general loss}
	\ell_{\gamma}(f,\bx,y)=\mathds{1}_{\inf\limits_{\bx'\colon \|\bx-\bx'\|\leq \gamma}\left(y g(\bw \cdot \bx')+by\right)\leq 0}=\mathds{1}_{
	yg(\bw \cdot \bx -\gamma y\|\bw\|)+by\leq 0}=\mathds{1}_{yg(\bw \cdot \bx -\gamma y)+by \leq 0}\,.
\end{equation}
Define $\cF_1=\curl*{\bx\rightarrow \bw \cdot \bx\mid\|\bw\|=1}$ and  $\cF_2=\curl*{\bx\rightarrow b \mid|b|\leq \C}$. Note that for any $f\in\sH_g$ and $\bx\in\sX$, there exist $\alpha_1 \in \cA_{\cF_1}$  and $\alpha_2 \in \cA_{\cF_2}$ such that $f(\bx)=g(\alpha_1)+\alpha_2$, where $\cA_{\cF_1}\overset{\text{def}}=\left\{f_1(\bx)\mid f_1\in\cF_1,\bx\in\sX\right\}$ and $\cA_{\cF_2}\overset{\text{def}}=\left\{f_2(\bx)\mid f_2\in\cF_2,\bx\in\sX\right\}$. Therefore, we can rewrite \eqref{eq:general loss} as
\[
   \ell_{\gamma}(\alpha_1,\alpha_2,y)=\mathds{1}_{y g(\alpha_1 -\gamma y)+\alpha_2y \leq 0}.
\]
Similarly, we can rewrite the inner risk and pseudo-minimal inner risk of $\ell_{\gamma}$ and $\phi$ as 
\[\cC_{\ell_{\gamma}}(\alpha_1,\alpha_2,\eta)=\eta\ell_{\gamma}(\alpha_1,\alpha_2,1)+(1-\eta)\ell_{\gamma}(\alpha_1,\alpha_2,-1),~ \cC_{\ell_{\gamma},\sH_{g}}^*(\eta)=\inf\limits_{\alpha_1 \in \cA_{\cF_1},\alpha_2 \in \cA_{\cF_2}}\cC_{\ell_{\gamma}}(\alpha_1,\alpha_2,\eta),\]
\[\cC_{\phi}(\alpha_1,\alpha_2,\eta)=\eta\phi(g(\alpha_1)+\alpha_2)+(1-\eta)\phi(-g(\alpha_1)-\alpha_2),~
\cC_{\phi,\sH_{g}}^*(\eta)=\inf\limits_{\alpha_1 \in \cA_{\cF_1},\alpha_2 \in \cA_{\cF_2}}\cC_{\phi}(\alpha_1,\alpha_2,\eta),\]
\[\Delta\cC_{\ell_{\gamma},\sH_g}(\alpha_1,\alpha_2,\eta)=\cC_{\ell_{\gamma}}(\alpha_1,\alpha_2,\eta) - \cC^*_{\ell_{\gamma},\sH_g}(\eta),~ \Delta\cC_{\phi,\sH_g}(\alpha_1,\alpha_2,\eta)=\cC_{\phi}(\alpha_1,\alpha_2,\eta) - \cC^*_{\phi,\sH_g}(\eta).\]
Next, we characterize the pseudo-calibration function of losses $(\phi, \ell_{\gamma})$ given hypothesis set $\sH_g$. 

\begin{lemma}
\label{lemma:bar_delta_general}
Given a non-decreasing and continuous function $g$ such that $g(1+\gamma)< \C$ and  $g(-1-\gamma)> -\C$. For a margin-based loss $\phi$ and hypothesis set $\sH_g$, the pseudo-calibration function of losses $(\phi, \ell_{\gamma})$ is \[\hat{\delta}(\epsilon)=\inf_{\eta\in[0,1]}\Bar{\delta}(\epsilon,\eta),\]
where 
\begin{equation*}
\Bar{\delta}(\epsilon,\eta) =
\begin{cases}
+\infty & \text{if} ~\epsilon>\max\curl*{\eta,1-\eta},\\
\inf\limits_{\substack{\alpha_1 \in \cA_{\cF_1},\alpha_2 \in \cA_{\cF_2}\colon\\ -g(\alpha_1+\gamma)\leq\alpha_2\leq -g(\alpha_1-\gamma)}}\Delta\cC_{\phi,\sH_g}(\alpha_1,\alpha_2,\eta) & \text{if} ~ |2\eta-1|<\epsilon\leq\max\curl*{\eta,1-\eta},\\
\inf\limits_{\substack{\alpha_1 \in \cA_{\cF_1},\alpha_2 \in \cA_{\cF_2}\colon\\ -g(\alpha_1+\gamma)\leq\alpha_2\leq -g(\alpha_1-\gamma)\\
\text{or}~(2\eta-1)(\alpha_2+g(\alpha_1+\gamma))\leq0}} \Delta\cC_{\phi,\sH_g}(\alpha_1,\alpha_2,\eta) & \text{if} ~ \epsilon\leq|2\eta-1|.
\end{cases}
\end{equation*}
\end{lemma}

\begin{proof}
The inner $\ell_{\gamma}$-risk is 
\begin{equation*}
\begin{aligned}
  \cC_{\ell_{\gamma}}(\alpha_1,\alpha_2,\eta)&=\eta \mathds{1}_{g(\alpha_1-\gamma)+\alpha_2\leq 0}+(1-\eta)\mathds{1}_{g(\alpha_1+\gamma)+\alpha_2\geq 0}\\
  &=\begin{cases}
   1 & \text{if} ~ -g(\alpha_1+\gamma)\leq\alpha_2\leq -g(\alpha_1-\gamma),\\
   \eta & \text{if} ~ \alpha_2< -g(\alpha_1+\gamma),\\
   1-\eta & \text{if} ~ \alpha_2> -g(\alpha_1-\gamma).\\
  \end{cases}
\end{aligned}
\end{equation*}
Since $-\C<-g(1+\gamma)$ and $\C>-g(-1-\gamma)$, the pseudo-minimal inner $\ell_{\gamma}$-risk is
\begin{equation*}
    \cC^*_{\ell_{\gamma},\sH_g}(\eta)=\min\curl*{\eta,1-\eta}.
\end{equation*}
Then, it can be computed that
\begin{align*}
  \Delta\cC_{\ell_{\gamma},
    \sH_g}(\alpha_1,\alpha_2,\eta)&=
    \begin{cases}
    \max\curl*{\eta,1-\eta} & \text{if} ~ -g(\alpha_1+\gamma)\leq\alpha_2\leq -g(\alpha_1-\gamma),\\
   |2\eta-1|\mathds{1}_{(2\eta-1)(\alpha_2+g(\alpha_1+\gamma))\leq 0} & \text{if} ~ \alpha_2> -g(\alpha_1-\gamma) \text{~or~} \alpha_2< -g(\alpha_1+\gamma).\\
    \end{cases}
\end{align*}
By definition, for a fixed $\eta\in[0,1]$, \[\bar{\delta}(\epsilon,\eta)=\inf\limits_{ \alpha_1 \in \cA_{\cF_1},\alpha_2 \in \cA_{\cF_2}}\curl*{\Delta\cC_{\phi,\sH_g}(\alpha_1,\alpha_2,\eta) \mid \Delta\cC_{\ell_{\gamma},\sH_g}(\alpha_1,\alpha_2,\eta)\geq\epsilon }.\]
If $\epsilon>\max\curl*{\eta,1-\eta}$, then for all $\alpha_1 \in \cA_{\cF_1},\alpha_2 \in \cA_{\cF_2}$, $\Delta\cC_{\ell_{\gamma},\sH_g}(\alpha_1,\alpha_2,\eta)<\epsilon$, which implies that $\bar{\delta}(\epsilon,\eta)=\infty$.
If $|2\eta-1|<\epsilon\leq\max\curl*{\eta,1-\eta}$, then $\Delta\cC_{\ell_{\gamma},\sH_g}(\alpha_1,\alpha_2,\eta)\geq\epsilon$ is achieved when $-g(\alpha_1+\gamma)\leq\alpha_2\leq -g(\alpha_1-\gamma)$, which leads to $\bar{\delta}(\epsilon,\eta)=\inf_{\substack{\alpha_1 \in \cA_{\cF_1},\alpha_2 \in \cA_{\cF_2}\colon\\
-g(\alpha_1+\gamma)\leq\alpha_2\leq -g(\alpha_1-\gamma)}}\Delta\cC_{\phi,\sH_g}(\alpha_1,\alpha_2,\eta)$. 
If $\epsilon\leq|2\eta-1|$, then $\Delta\cC_{\ell_{\gamma},\sH_g}(\alpha_1,\alpha_2,\eta)\geq\epsilon$ is achieved when $-g(\alpha_1+\gamma)\leq\alpha_2\leq -g(\alpha_1-\gamma)$ or $(2\eta-1)(\alpha_2+g(\alpha_1+\gamma))\leq0$. Therefore,  $\bar{\delta}(\epsilon,\eta)=\inf_{\substack{\alpha_1 \in \cA_{\cF_1},\alpha_2 \in \cA_{\cF_2}\colon-g(\alpha_1+\gamma)\leq\alpha_2\leq -g(\alpha_1-\gamma)\\
\text{or}~(2\eta-1)(\alpha_2+g(\alpha_1+\gamma))\leq0}} \Delta\cC_{\phi,\sH_g}(\alpha_1,\alpha_2,\eta)$.
\end{proof}
Note that in our setting, $\alpha_1\in [-1,1]$ and $\alpha_2\in [-\C,\C]$. Therefore $g(\alpha_1)+\alpha_2\in [g(-1)-\C,g(1)+\C]$, since $g$ is continuous.
Then,
\begin{equation}
\begin{aligned}
    &\left\{g(\alpha_1)+\alpha_2\colon\alpha_1 \in \cA_{\cF_1},\alpha_2 \in \cA_{\cF_2}, -g(\alpha_1+\gamma)\leq\alpha_2\leq -g(\alpha_1-\gamma)\right\}\\
    &=[\inf\limits_{\alpha_1\in[-1,1]}g(\alpha_1)-g(\alpha_1+\gamma), \sup\limits_{\alpha_1\in[-1,1]}g(\alpha_1)-g(\alpha_1-\gamma)]\,,\\
    &\left\{g(\alpha_1)+\alpha_2\colon \alpha_1 \in \cA_{\cF_1},\alpha_2 \in \cA_{\cF_2}, \alpha_2\leq -g(\alpha_1-\gamma)\right\}
    =[g(-1)-\C, \sup\limits_{\alpha_1\in[-1,1]}g(\alpha_1)-g(\alpha_1-\gamma)]\,,\\
    &\left\{g(\alpha_1)+\alpha_2\colon \alpha_1 \in \cA_{\cF_1},\alpha_2 \in \cA_{\cF_2}, \alpha_2\geq -g(\alpha_1+\gamma)\right\}
    =[\inf\limits_{\alpha_1\in[-1,1]}g(\alpha_1)-g(\alpha_1+\gamma),g(1)+\C]\,.\\
\end{aligned}
\label{eq:domain_general}
\end{equation}
Since $g$ is non-decreasing, we have $g(\alpha_1)-g(\alpha_1+\gamma)\leq 0$ and $g(\alpha_1)-g(\alpha_1-\gamma)\geq 0$ for any $\alpha_1\in[-1,1]$. Also, $-g(\alpha_1)\in[-g(1),-g(-1)]\subset[-\C,\C]$ . Therefore, 
\begin{equation}
\label{eq:contain0}
   0 \in \left\{g(\alpha_1)+\alpha_2\colon\alpha_1 \in \cA_{\cF_1},\alpha_2 \in \cA_{\cF_2}, -g(\alpha_1+\gamma)\leq\alpha_2\leq -g(\alpha_1-\gamma)\right\}\,. 
\end{equation}

\calibrationConvexGeneral*

\begin{proof}
Suppose that $\phi$ is pseudo-$\sH_g$-calibrated with respect to $\ell_{\gamma}$. By Proposition~\ref{prop:calibration_function_positive}, $\phi$ is pseudo-$\sH_g$-calibrated with respect to $\ell_{\gamma}$ if and only if its pseudo-calibration function $\hat{\delta}$ satisfies $\hat{\delta}(\epsilon)>0$ for all $\epsilon>0$, which leads to $\Bar{\delta}(\epsilon,\eta)>0$ for all $\epsilon>0$ and $\eta\in [0,1]$. By lemma \ref{lemma:bar_delta_general}, take $\eta=\frac12$, we obtain
\begin{equation*}
    	\inf\limits_{\alpha_1 \in \cA_{\cF_1},\alpha_2 \in \cA_{\cF_2}\colon -g(\alpha_1+\gamma)\leq\alpha_2\leq -g(\alpha_1-\gamma)}\Delta\cC_{\phi,\sH_g}(\alpha_1,\alpha_2,\frac12)>0
\end{equation*}
which is equivalent to
\begin{equation}
    \inf\limits_{\alpha_1 \in \cA_{\cF_1},\alpha_2 \in \cA_{\cF_2}\colon -g(\alpha_1+\gamma)\leq\alpha_2\leq -g(\alpha_1-\gamma)}\cC_{\phi}(\alpha_1,\alpha_2,\frac12)>
     \inf_{\alpha_1 \in \cA_{\cF_1},\alpha_2 \in \cA_{\cF_2}}\cC_{\phi}(\alpha_1,\alpha_2,\frac12).
     \label{eq:larger_general}
\end{equation}
By the definition of inner risk, 
\begin{equation*}
    \cC_{\phi}(\alpha_1,\alpha_2,\frac12)=\frac12\phi(g(\alpha_1)+\alpha_2)+\frac12\phi(-g(\alpha_1)-\alpha_2).
\end{equation*}
Define $\bar{\phi}(\alpha_1,\alpha_2)=\phi(g(\alpha_1)+\alpha_2)+\phi(-g(\alpha_1)-\alpha_2)$. By Jensen's inequality, $\phi(0)\leq\frac12 \bar{\phi}(\alpha_1,\alpha_2)$ for all $\alpha_1 \in \cA_{\cF_1}$, $\alpha_2 \in \cA_{\cF_2}$.

Since $\cC_{\phi}(\alpha_1,\alpha_2,\frac12)=\frac12\bar{\phi}(\alpha_1,\alpha_2)$ and \eqref{eq:contain0},
we obtain
\begin{equation*}
\inf\limits_{\alpha_1 \in \cA_{\cF_1},\alpha_2 \in \cA_{\cF_2}\colon -g(\alpha_1+\gamma)\leq\alpha_2\leq -g(\alpha_1-\gamma)}\cC_{\phi}(\alpha_1,\alpha_2,\frac12)=
     \inf_{\alpha_1 \in \cA_{\cF_1},\alpha_2 \in \cA_{\cF_2}}\cC_{\phi}(\alpha_1,\alpha_2,\frac12)
     = \phi(0),
\end{equation*}
contradicting \eqref{eq:larger_general}.
Therefore, $\phi$ is not pseudo-$\sH_g$-calibrated with respect to $\ell_{\gamma}$. By Corollary \ref{corollary:calibration_negative}, $\phi$ is also not $\sH_g$-calibrated with respect to $\ell_{\gamma}$.
\end{proof}

\subsection{Proof of Theorem~\ref{Thm:quasiconcave_calibrate_general}}
\label{app:quasiconcave_calibrate_general}

Following the notations in Appendix~\ref{app:calibration_convex_general}, we first give equivalent conditions of pseudo-calibration based on inner risk of $\phi$ and $\sH_g$.
\begin{lemma}
\label{lemma:equivalent1_general}
Given a non-decreasing and continuous function $g$ such that $g(1+\gamma)< \C$ and  $g(-1-\gamma)> -\C$. Let $\phi$ be a margin-based loss. Then $\phi$ is pseudo-$\sH_g$-calibrated with respect to $\ell_{\gamma}$ if and only if 
\begin{align*}
    \inf\limits_{\substack{\alpha_1 \in \cA_{\cF_1},\alpha_2 \in \cA_{\cF_2}\colon\\ -g(\alpha_1+\gamma)\leq\alpha_2\leq -g(\alpha_1-\gamma)}}\cC_{\phi}(\alpha_1,\alpha_2,\frac12)> &\inf_{\alpha_1 \in \cA_{\cF_1},\alpha_2 \in \cA_{\cF_2}}\cC_{\phi}(\alpha_1,\alpha_2,\frac12)\,,\text{and}\\
    \inf\limits_{\alpha_1 \in \cA_{\cF_1},\alpha_2 \in \cA_{\cF_2}\colon \alpha_2\leq -g(\alpha_1-\gamma)}\cC_{\phi}(\alpha_1,\alpha_2,\eta) > &\inf\limits_{\alpha_1 \in \cA_{\cF_1},\alpha_2 \in \cA_{\cF_2}}\cC_{\phi}(\alpha_1,\alpha_2,\eta) \text{ for all } \eta\in (\frac12,1]\,,\text{and}\\
    \inf\limits_{\alpha_1 \in \cA_{\cF_1},\alpha_2 \in \cA_{\cF_2}\colon \alpha_2\geq -g(\alpha_1+\gamma)}\cC_{\phi}(\alpha_1,\alpha_2,\eta) > &\inf\limits_{\alpha_1 \in \cA_{\cF_1},\alpha_2 \in \cA_{\cF_2}}\cC_{\phi}(\alpha_1,\alpha_2,\eta) \text{ for all } \eta\in [0,\frac12)\,.\\
\end{align*}
\end{lemma}

\begin{proof}
Let $\hat{\delta}$ be the pseudo-calibration function of $(\phi,\ell{\gamma})$ for hypothesis sets $\sH_g$. By Lemma \ref{lemma:bar_delta_general}, $\hat{\delta}(\epsilon)=\inf_{\eta\in[0,1]}\Bar{\delta}(\epsilon,\eta)$, where 
\begin{equation*}
\Bar{\delta}(\epsilon,\eta) =
\begin{cases}
+\infty & \text{if} ~\epsilon>\max\curl*{\eta,1-\eta},\\
\inf\limits_{\substack{\alpha_1 \in \cA_{\cF_1},\alpha_2 \in \cA_{\cF_2}\colon\\ -g(\alpha_1+\gamma)\leq\alpha_2\leq -g(\alpha_1-\gamma)}}\Delta\cC_{\phi,\sH_g}(\alpha_1,\alpha_2,\eta) & \text{if} ~ |2\eta-1|<\epsilon\leq\max\curl*{\eta,1-\eta},\\
\inf\limits_{\substack{\alpha_1 \in \cA_{\cF_1},\alpha_2 \in \cA_{\cF_2}\colon\\ -g(\alpha_1+\gamma)\leq\alpha_2\leq -g(\alpha_1-\gamma)\\
\text{or}~(2\eta-1)(\alpha_2+g(\alpha_1+\gamma))\leq0}} \Delta\cC_{\phi,\sH_g}(\alpha_1,\alpha_2,\eta) & \text{if} ~ \epsilon\leq|2\eta-1|.
\end{cases}
\end{equation*}
By Proposition~\ref{prop:calibration_function_positive}, $\phi$ is pseudo-$\sH_g$-calibrated with respect to $\ell_{\gamma}$ if and only if its pseudo-calibration function $\hat{\delta}$ satisfies $\hat{\delta}(\epsilon)>0$ for all $\epsilon>0$. This is equivalent to $\Bar{\delta}(\epsilon,\eta)>0$ for all $\epsilon>0$ and $\eta\in [0,1]$.\\
For $\eta=\frac12$, we have
\begin{equation}
\begin{aligned}
&\Bar{\delta}(\epsilon,\frac12)>0 \text{ for all } \epsilon>0\\
&\Leftrightarrow \inf\limits_{\alpha_1 \in \cA_{\cF_1},\alpha_2 \in \cA_{\cF_2}\colon -g(\alpha_1+\gamma)\leq\alpha_2\leq -g(\alpha_1-\gamma)}\cC_{\phi}(\alpha_1,\alpha_2,\frac12)
> \inf_{\alpha_1 \in \cA_{\cF_1},\alpha_2 \in \cA_{\cF_2}}\cC_{\phi}(\alpha_1,\alpha_2,\frac12)\,.
\label{eq:keycondition1_general}
\end{aligned}
\end{equation}
For $1\geq\eta>\frac12$, we have $|2\eta-1|=2\eta-1$, $\max\curl*{\eta,1-\eta}=\eta$, and
\begin{equation*}
\begin{aligned}
  &\inf\limits_{\alpha_1 \in \cA_{\cF_1},\alpha_2 \in \cA_{\cF_2}\colon -g(\alpha_1+\gamma)\leq\alpha_2\leq -g(\alpha_1-\gamma)~\text{or}~(2\eta-1)(\alpha_2+g(\alpha_1+\gamma))\leq0} \Delta\cC_{\phi,\sH_g}(\alpha_1,\alpha_2,\eta)\\
  &=
  \inf_{\alpha_1 \in \cA_{\cF_1},\alpha_2 \in \cA_{\cF_2}\colon \alpha_2\leq -g(\alpha_1-\gamma)} \Delta\cC_{\phi,\sH_g}(\alpha_1,\alpha_2,\eta)\,. 
\end{aligned}
\end{equation*}
Therefore, $\Bar{\delta}(\epsilon,\eta)>0 \text{ for all } \epsilon>0 \text{ and } \eta\in(\frac12,1]$ if and only if 
\begin{equation*}
\begin{cases}
\inf\limits_{\substack{\alpha_1 \in \cA_{\cF_1},\alpha_2 \in \cA_{\cF_2}\colon\\ -g(\alpha_1+\gamma)\leq\alpha_2\leq -g(\alpha_1-\gamma)}}\cC_{\phi}(\alpha_1,\alpha_2,\eta) > \inf\limits_{\substack{\alpha_1 \in \cA_{\cF_1},\\
\alpha_2 \in \cA_{\cF_2}}}\cC_{\phi}(\alpha_1,\alpha_2,\eta) &\text{ for } \eta\in(\frac12,1] \text{ s.t. } 2\eta-1<\epsilon\leq \eta,\\
\inf\limits_{\substack{\alpha_1 \in \cA_{\cF_1},\alpha_2 \in \cA_{\cF_2}\colon\\
\alpha_2\leq -g(\alpha_1-\gamma)}}\cC_{\phi}(\alpha_1,\alpha_2,\eta) > \inf\limits_{\alpha_1 \in \cA_{\cF_1},\alpha_2 \in \cA_{\cF_2}}\cC_{\phi}(\alpha_1,\alpha_2,\eta) &\text{ for } \eta\in(\frac12,1] \text{ s.t. } \epsilon\leq 2\eta-1,
\end{cases}
\end{equation*}
for all $\epsilon>0$, which is equivalent to
\begin{equation}
\begin{cases}
\inf\limits_{\substack{\alpha_1 \in \cA_{\cF_1},\alpha_2 \in \cA_{\cF_2}\colon\\ -g(\alpha_1+\gamma)\leq\alpha_2\leq -g(\alpha_1-\gamma)}}\cC_{\phi}(\alpha_1,\alpha_2,\eta) > \inf\limits_{\substack{\alpha_1 \in \cA_{\cF_1},\\
\alpha_2 \in \cA_{\cF_2}}}\cC_{\phi}(\alpha_1,\alpha_2,\eta) &\text{ for } \eta\in(\frac12,1] \text{ s.t. } \epsilon\leq \eta < \frac{\epsilon+1}{2},\\
\inf\limits_{\substack{\alpha_1 \in \cA_{\cF_1},\alpha_2 \in \cA_{\cF_2}\colon \\
\alpha_2\leq -g(\alpha_1-\gamma)}}\cC_{\phi}(\alpha_1,\alpha_2,\eta) > \inf\limits_{\alpha_1 \in \cA_{\cF_1},\alpha_2 \in \cA_{\cF_2}}\cC_{\phi}(\alpha_1,\alpha_2,\eta) &\text{ for  } \eta\in(\frac12,1] \text{ s.t. } \frac{\epsilon+1}{2}\leq \eta,
\end{cases}
\label{eq:condition1 in proof_general}
\end{equation}
for all $\epsilon>0$.
We observe that
\begin{equation*}
\begin{aligned}
    &\left\{\eta\in(\frac12,1]\Bigg|\epsilon\leq \eta < \frac{\epsilon+1}{2},\epsilon>0\right\}=\left\{\frac12<\eta\leq1\right\}\,, \text{ and}\\
    &\left\{\eta\in(\frac12,1]\Bigg|\frac{\epsilon+1}{2}\leq \eta, \epsilon>0\right\}=\left\{\frac12<\eta\leq1\right\}\,, \text{ and}\\
    &\inf\limits_{\substack{\alpha_1 \in \cA_{\cF_1},\alpha_2 \in \cA_{\cF_2}\colon\\ -g(\alpha_1+\gamma)\leq\alpha_2\leq -g(\alpha_1-\gamma)}}\cC_{\phi}(\alpha_1,\alpha_2,\eta) \geq \inf\limits_{\alpha_1 \in \cA_{\cF_1},\alpha_2 \in \cA_{\cF_2}\colon \alpha_2\leq -g(\alpha_1-\gamma)}\cC_{\phi}(\alpha_1,\alpha_2,\eta) \text{ for all } \eta\,.
\end{aligned}
\end{equation*}
Therefore we reduce the above condition
\eqref{eq:condition1 in proof_general} as
\begin{equation}
    \inf\limits_{\alpha_1 \in \cA_{\cF_1},\alpha_2 \in \cA_{\cF_2}\colon \alpha_2\leq -g(\alpha_1-\gamma)}\cC_{\phi}(\alpha_1,\alpha_2,\eta) > \inf\limits_{\alpha_1 \in \cA_{\cF_1},\alpha_2 \in \cA_{\cF_2}}\cC_{\phi}(\alpha_1,\alpha_2,\eta) \text{ for all } \eta\in (\frac12,1]\,.
    \label{eq:keycondition2_general}
\end{equation}
For $\frac12>\eta\geq0$, we have $|2\eta-1|=1-2\eta$, $\max\curl*{\eta,1-\eta}=1-\eta$, and
\begin{equation*}
\begin{aligned}
  &\inf\limits_{\alpha_1 \in \cA_{\cF_1},\alpha_2 \in \cA_{\cF_2}\colon -g(\alpha_1+\gamma)\leq\alpha_2\leq -g(\alpha_1-\gamma)~\text{or}~(2\eta-1)(\alpha_2+g(\alpha_1+\gamma))\leq0} \Delta\cC_{\phi,\sH_g}(\alpha_1,\alpha_2,\eta)\\
  &=
  \inf_{\alpha_1 \in \cA_{\cF_1},\alpha_2 \in \cA_{\cF_2}\colon \alpha_2\geq -g(\alpha_1+\gamma)} \Delta\cC_{\phi,\sH_g}(\alpha_1,\alpha_2,\eta)\,. 
\end{aligned}
\end{equation*}
Therefore, $\Bar{\delta}(\epsilon,\eta)>0 \text{ for all } \epsilon>0 \text{ and } \eta\in[0,\frac12)$ if and only if 
\begin{equation*}
\begin{cases}
\inf\limits_{\substack{\alpha_1 \in \cA_{\cF_1},\alpha_2 \in \cA_{\cF_2}\colon\\ -g(\alpha_1+\gamma)\leq\alpha_2\leq -g(\alpha_1-\gamma)}}\cC_{\phi}(\alpha_1,\alpha_2,\eta) > \inf\limits_{\substack{\alpha_1 \in \cA_{\cF_1},\\
\alpha_2 \in \cA_{\cF_2}}}\cC_{\phi}(\alpha_1,\alpha_2,\eta) &\text{for } \eta\in[0,\frac12) \text{ s.t. } 1-2\eta<\epsilon\leq 1-\eta\\
\inf\limits_{\substack{\alpha_1 \in \cA_{\cF_1},\alpha_2 \in \cA_{\cF_2}\colon\\
\alpha_2\geq -g(\alpha_1+\gamma)}}\cC_{\phi}(\alpha_1,\alpha_2,\eta) > \inf\limits_{\alpha_1 \in \cA_{\cF_1},\alpha_2 \in \cA_{\cF_2}}\cC_{\phi}(\alpha_1,\alpha_2,\eta) &\text{for  } \eta\in[0,\frac12) \text{ s.t. } \epsilon\leq 1-2\eta,
\end{cases}
\end{equation*}
for all $\epsilon>0$, which is equivalent to
\begin{equation}
\begin{cases}
\inf\limits_{\substack{\alpha_1 \in \cA_{\cF_1},\alpha_2 \in \cA_{\cF_2}\colon\\ -g(\alpha_1+\gamma)\leq\alpha_2\leq -g(\alpha_1-\gamma)}}\cC_{\phi}(\alpha_1,\alpha_2,\eta) > \inf\limits_{\substack{\alpha_1 \in \cA_{\cF_1},\\
\alpha_2 \in \cA_{\cF_2}}}\cC_{\phi}(\alpha_1,\alpha_2,\eta) &\text{for  } \eta\in[0,\frac12) \text{ s.t. } \frac{1-\epsilon}{2}< \eta \leq 1-\epsilon,\\
\inf\limits_{\substack{\alpha_1 \in \cA_{\cF_1},\alpha_2 \in \cA_{\cF_2}\colon\\
\alpha_2\geq -g(\alpha_1+\gamma)}}\cC_{\phi}(\alpha_1,\alpha_2,\eta) > \inf\limits_{\alpha_1 \in \cA_{\cF_1},\alpha_2 \in \cA_{\cF_2}}\cC_{\phi}(\alpha_1,\alpha_2,\eta) &\text{for  } \eta\in[0,\frac12) \text{ s.t. } \eta\leq \frac{1-\epsilon}{2},
\end{cases}
\label{eq:condition2 in proof_general}
\end{equation}
for all $\epsilon>0$.
We observe that
\begin{equation*}
\begin{aligned}
    &\left\{\eta\in[0,\frac12)\Bigg|\frac{1-\epsilon}{2}< \eta \leq 1-\epsilon,\epsilon>0\right\}=\left\{0\leq\eta<\frac12\right\}\,, \text{ and}\\
    &\left\{\eta\in[0,\frac12)\Bigg|\eta\leq \frac{1-\epsilon}{2}, \epsilon>0\right\}=\left\{0\leq\eta<\frac12\right\}\,, \text{ and}\\
    &\inf\limits_{\substack{\alpha_1 \in \cA_{\cF_1},\alpha_2 \in \cA_{\cF_2}\colon\\ -g(\alpha_1+\gamma)\leq\alpha_2\leq -g(\alpha_1-\gamma)}}\cC_{\phi}(\alpha_1,\alpha_2,\eta) \geq \inf\limits_{\alpha_1 \in \cA_{\cF_1},\alpha_2 \in \cA_{\cF_2}\colon \alpha_2\geq -g(\alpha_1+\gamma)}\cC_{\phi}(\alpha_1,\alpha_2,\eta) \text{ for all } \eta\,.
\end{aligned}
\end{equation*}
Therefore we reduce the above condition
\eqref{eq:condition2 in proof_general} as
\begin{equation}
    \inf\limits_{\alpha_1 \in \cA_{\cF_1},\alpha_2 \in \cA_{\cF_2}\colon \alpha_2\geq -g(\alpha_1+\gamma)}\cC_{\phi}(\alpha_1,\alpha_2,\eta) > \inf\limits_{\alpha_1 \in \cA_{\cF_1},\alpha_2 \in \cA_{\cF_2}}\cC_{\phi}(\alpha_1,\alpha_2,\eta) \text{ for all } \eta\in [0,\frac12)\,.
    \label{eq:keycondition3_general}
\end{equation}
To sum up, by \eqref{eq:keycondition1_general}, \eqref{eq:keycondition2_general} and \eqref{eq:keycondition3_general}, we conclude the proof.
\end{proof}
Define $\overline{A}=\sup_{\alpha_1\in[-1,1]}g(\alpha_1)-g(\alpha_1-\gamma)$ and $\underline{A}=\inf_{\alpha_1\in[-1,1]}g(\alpha_1)-g(\alpha_1+\gamma)$. Since $g$ is non-decreasing, we have $\overline{A}\geq 0$ and $\underline{A}\leq 0$. 
Note the inner risk $\cC_{\phi}(\alpha_1,\alpha_2,\eta)$ only depends on $t\colon=g(\alpha_1)+\alpha_2$ and $\eta$. Therefore, we can rewrite the inner risk and pseudo-minimal inner risk of $\phi$ as 
\[\cC_{\phi}(t,\eta)=\eta\phi(t)+(1-\eta)\phi(-t),\quad\cC_{\phi,\sH_{g}}^*(\eta)=\inf\limits_{g(-1)-\C\leq t\leq g(1)+\C}\cC_{\phi}(t,\eta),\]
By \eqref{eq:domain_general}, Lemma \ref{lemma:equivalent1_general} is equivalent to Lemma \ref{lemma:equivalent2_general}.
\begin{lemma}
\label{lemma:equivalent2_general}
Given a non-decreasing and continuous function $g$ such that $g(1+\gamma)< \C$ and  $g(-1-\gamma)> -\C$. Let $\phi$ be a margin-based loss. Then $\phi$ is pseudo-$\sH_g$-calibrated with respect to $\ell_{\gamma}$ if and only if 
\begin{equation*}
\begin{aligned}
    \inf\limits_{\underline{A}\leq t\leq \overline{A}}\cC_{\phi}(t,\frac12)> &\inf_{g(-1)-\C\leq t\leq g(1)+\C}\cC_{\phi}(t,\frac12)\,,\text{and}\\
    \inf\limits_{g(-1)-\C\leq t\leq \overline{A}} \cC_{\phi}(t,\eta)  > &\inf\limits_{g(-1)-\C\leq t\leq g(1)+\C}\cC_{\phi}(t,\eta) \text{ for all } \eta\in (\frac12,1]\,,\text{and}\\
    \inf\limits_{\underline{A}\leq t\leq g(1)+\C}\cC_{\phi}(t,\eta) > &\inf\limits_{g(-1)-\C\leq t\leq g(1)+\C}\cC_{\phi}(t,\eta) \text{ for all } \eta\in [0,\frac12)\,.\\
\end{aligned}
\end{equation*}
\end{lemma}

Lemma \ref{lemma:quasiconcave_even_general} concludes some results that will be useful in the proof of Theorem \ref{Thm:quasiconcave_calibrate_general}.

\begin{lemma}
\label{lemma:quasiconcave_even_general}
Let $\phi$ be a margin-based loss. If $\phi$ is bounded, continuous, non-increasing, quasi-concave even, and assume $\phi(g(-1)-\C)>\phi(\C-g(-1))$, $g(-1)+g(1)\geq0$, then
\begin{enumerate}
    \item the inner $\phi$-risk $\cC_{\phi}(t,\eta)$ is quasi-concave in $t\in\mathbb{R}$ for all $\eta\in[0,1]$.
    \label{part1_lemma:quasiconcave_even_general}
    \item $\phi(t)+\phi(-t)$ is non-increasing in $t$ when $t\geq 0$.
    \label{part2_lemma:quasiconcave_even_general}
    \item for $l,u\in \mathbb{R}(l\leq u),~\inf_{t\in [l,u]}\cC_{\phi}(t,\eta)=\min\curl*{\cC_{\phi}(l,\eta),\cC_{\phi}(u,\eta)}$ for all $\eta\in[0,1]$.
    \label{part3_lemma:quasiconcave_even_general}
    \item for all $\eta\in (\frac12,1]$, $\Cpae$ is non-increasing in $t$ when $t\geq 0$.
    \label{part4_lemma:quasiconcave_even_general}
    \item for all $\eta\in (\frac12,1]$, $\cC_{\phi}(g(-1)-\C,\eta)>\cC_{\phi}(g(1)+\C,\eta)$. 
    \label{part5_lemma:quasiconcave_even_general}
    \item for all $\eta\in [0,\frac12)$, $\Cpae$ is non-decreasing in $t$ when $t\leq 0$.
    \label{part6_lemma:quasiconcave_even_general}
    \item for all $\eta\in [0,\frac12)$, $\cC_{\phi}(g(-1)-\C,\eta)<\cC_{\phi}(g(1)+\C,\eta) \text{~if and only if~} \phi(\C-g(-1))+\phi(g(-1)-\C)=\phi(g(1)+\C)+\phi(-g(1)-\C)$. 
    \label{part7_lemma:quasiconcave_even_general}
\end{enumerate}
\end{lemma}

\begin{proof}
 Part \ref{part1_lemma:quasiconcave_even_general},\ref{part2_lemma:quasiconcave_even_general},\ref{part4_lemma:quasiconcave_even_general} of Lemma \ref{lemma:quasiconcave_even_general} are stated in Lemma 13 of \citep{pmlr-v125-bao20a}.
 Part \ref{part3_lemma:quasiconcave_even_general} is a corollary of Part
\ref{part1_lemma:quasiconcave_even_general} by the characterization of continuous and quasi-convex functions in \citep{boyd2004convex}.

Consider Part \ref{part5_lemma:quasiconcave_even_general}. For $\eta\in (\frac12,1]$,
\begin{align*}
    &\cC_{\phi}(g(-1)-\C,\eta)-\cC_{\phi}(g(1)+\C,\eta)\\ 
    \geq&\cC_{\phi}(g(-1)-\C,\eta)-\cC_{\phi}(\C-g(-1),\eta) & (\text{Part \ref{part4_lemma:quasiconcave_even_general} of Lemma \ref{lemma:quasiconcave_even_general}})\\
    =& (2\eta-1)(\phi(g(-1)-\C)-\phi(\C-g(-1)))\\
    >&0. 
\end{align*}

Consider Part \ref{part6_lemma:quasiconcave_even_general}. For $\eta\in [0,\frac12)$, and $\alpha_1,\alpha_2\leq0$. Suppose that $\alpha_1<\alpha_2$, then 
\begin{align*}
    &\phi(\alpha_1)-\phi(-\alpha_1)-\phi(\alpha_2)+\phi(-\alpha_2)\\
    \geq& \phi(\alpha_2)-\phi(-\alpha_2)-\phi(\alpha_2)+\phi(-\alpha_2)\\
    =&0,
\end{align*}
since $\phi$ is non-increasing.

By Part \ref{part2_lemma:quasiconcave_even_general} of Lemma \ref{lemma:quasiconcave_even_general}, $\phi(t)+\phi(-t)$ is non-decreasing in $t$ when $t\leq 0$. 

Therefore, for $\eta\in [0,\frac12)$,
\begin{align*}
    &\cC_{\phi}(\alpha_1,\eta) -  \cC_{\phi}(\alpha_2,\eta) \\
    =& (\phi(\alpha_1)-\phi(-\alpha_1)-\phi(\alpha_2)+\phi(-\alpha_2))\eta+\phi(-\alpha_1)-\phi(-\alpha_2)\\
    \leq& (\phi(\alpha_1)-\phi(-\alpha_1)-\phi(\alpha_2)+\phi(-\alpha_2))\frac12+\phi(-\alpha_1)-\phi(-\alpha_2)\\
    =& \frac12 (\phi(\alpha_1)+\phi(-\alpha_1)-\phi(\alpha_2)-\phi(-\alpha_2))\\
    \leq& 0.
\end{align*}

Consider Part \ref{part7_lemma:quasiconcave_even_general}.
Since $\phi$ is non-increasing, we have
\begin{align*}
    &\phi(g(-1)-\C)-\phi(\C-g(-1))+\phi(-g(1)-\C)-\phi(g(1)+\C)\\
    \geq& \phi(g(-1)-\C)-\phi(\C-g(-1))+\phi(g(1)+\C)-\phi(g(1)+\C)\\
    =&\phi(g(-1)-\C)-\phi(\C-g(-1))\\
    >&0.
\end{align*}
$\Longleftarrow\colon$
Suppose $ \phi(\C-g(-1))+\phi(g(-1)-\C)=\phi(g(1)+\C)+\phi(-g(1)-\C)$, then for $\eta\in [0,\frac12)$,
\begin{align*}
    &\cC_{\phi}(g(-1)-\C,\eta)-\cC_{\phi}(g(1)+\C,\eta)\\ 
     =& (\phi(g(-1)-\C)-\phi(\C-g(-1))+\phi(-g(1)-\C)-\phi(g(1)+\C))\eta\\
     &\qquad+\phi(\C-g(-1))-\phi(-g(1)-\C)\\
     <& (\phi(g(-1)-\C)-\phi(\C-g(-1))+\phi(-g(1)-\C)-\phi(g(1)+\C))\frac12\\
     &\qquad+\phi(\C-g(-1))-\phi(-g(1)-\C)\\
    =& \frac12 (\phi(\C-g(-1))+\phi(g(-1)-\C)-\phi(g(1)+\C)-\phi(-g(1)-\C))\\
    =&0.
\end{align*}
$\Longrightarrow\colon$
Suppose $\cC_{\phi}(g(-1)-\C,\eta)<\cC_{\phi}(g(1)+\C,\eta)$ for $\eta\in [0,\frac12)$, then
\begin{align*}
    &\cC_{\phi}(g(-1)-\C,\eta)-\cC_{\phi}(g(1)+\C,\eta)\\ 
     =& (\phi(g(-1)-\C)-\phi(\C-g(-1))+\phi(-g(1)-\C)-\phi(g(1)+\C))\eta\\
     &\qquad+\phi(\C-g(-1))-\phi(-g(1)-\C)\\
     <&0 
\end{align*}
for $\eta\in [0,\frac12)$. By taking $\eta\rightarrow \frac12$, we have
\begin{align*}
     & \frac12 (\phi(\C-g(-1))+\phi(g(-1)-\C)-\phi(g(1)+\C)-\phi(-g(1)-\C))\\
     =& (\phi(g(-1)-\C)-\phi(\C-g(-1))+\phi(-g(1)-\C)-\phi(g(1)+\C))\frac12\\
     &\qquad+\phi(\C-g(-1))-\phi(-g(1)-\C)\\
     \leq&0. 
\end{align*}
By Part \ref{part2_lemma:quasiconcave_even_general} of Lemma \ref{lemma:quasiconcave_even_general}, we have
\begin{align*}
    &\phi(\C-g(-1))+\phi(g(-1)-\C)-\phi(g(1)+\C)-\phi(-g(1)-\C)\\
    \geq&  \phi(g(1)+\C)+\phi(-g(1)-\C)-\phi(g(1)+\C)-\phi(-g(1)-\C)\\
    =&0.
\end{align*}
Therefore, we obtain $\phi(\C-g(-1))+\phi(g(-1)-\C)-\phi(g(1)+\C)-\phi(-g(1)-\C)=0$, namely $\phi(\C-g(-1))+\phi(g(-1)-\C)=\phi(g(1)+\C)+\phi(-g(1)-\C)$. 
\end{proof}

\QuasiconcaveCalibrateGeneral*

\begin{proof}
 By Lemma \ref{lemma:equivalent2_general}, $\phi$ is pseudo-$\sH_g$-calibrated with respect to $\ell_{\gamma}$ if and only if 
\begin{equation*}
\begin{aligned}
    \inf\limits_{\underline{A}\leq t\leq \overline{A}}\cC_{\phi}(t,\frac12)> &\inf_{g(-1)-\C\leq t\leq g(1)+\C}\cC_{\phi}(t,\frac12)\,,\text{and}\\
    \inf\limits_{g(-1)-\C\leq t\leq \overline{A}} \cC_{\phi}(t,\eta)  > &\inf\limits_{g(-1)-\C\leq t\leq g(1)+\C}\cC_{\phi}(t,\eta) \text{ for all } \eta\in (\frac12,1]\,,\text{and}\\
    \inf\limits_{\underline{A}\leq t\leq g(1)+\C}\cC_{\phi}(t,\eta) > &\inf\limits_{g(-1)-\C\leq t\leq g(1)+\C}\cC_{\phi}(t,\eta) \text{ for all } \eta\in [0,\frac12)\,.\\
\end{aligned}
\end{equation*}
Suppose that $\phi$ is pseudo-$\sH_g$-calibrated with respect to $\ell_{\gamma}$. Since for $\eta\in[0,\frac12)$,
\begin{align*}
    &\inf\limits_{\underline{A}\leq t\leq g(1)+\C} \Cpae
    = \min\curl*{\cC_{\phi}(\underline{A},\eta),\cC_{\phi}(g(1)+\C,\eta)} & \text{(Part \ref{part3_lemma:quasiconcave_even_general} of Lemma \ref{lemma:quasiconcave_even_general})}\\
    &\inf\limits_{g(-1)-\C\leq t\leq g(1)+\C} \Cpae
    = \min\curl*{\cC_{\phi}(g(-1)-\C,\eta),\cC_{\phi}(g(1)+\C,\eta)}, & \text{(Part \ref{part3_lemma:quasiconcave_even_general} of Lemma \ref{lemma:quasiconcave_even_general})}
\end{align*}
we have $\cC_{\phi}(g(-1)-\C,\eta)<\cC_{\phi}(g(1)+\C,\eta)$, otherwise $\inf_{\underline{A}\leq t\leq g(1)+\C} \Cpae\leq \cC_{\phi}(g(1)+\C,\eta)=\inf_{g(-1)-\C\leq t\leq g(1)+\C} \Cpae$. By \text{Part \ref{part7_lemma:quasiconcave_even_general} of Lemma \ref{lemma:quasiconcave_even_general}}, $\phi(\C-g(-1))+\phi(g(-1)-\C)=\phi(g(1)+\C)+\phi(-g(1)-\C)$.

Also,
\begin{align*}
    &\frac12\min\curl*{\phi(\overline{A})+\phi(-\overline{A}),\phi(\underline{A})+\phi(-\underline{A}) }\\
    =& \inf\limits_{\underline{A}\leq t\leq \overline{A}}\cC_{\phi}(t,\frac12) & \text{(Part \ref{part3_lemma:quasiconcave_even_general} of Lemma \ref{lemma:quasiconcave_even_general})}\\
    >& \inf\limits_{g(-1)-\C\leq t\leq g(1)+\C}\cC_{\phi}(t,\frac12) & \text{(Lemma \ref{lemma:equivalent2_general})}\\
    =& \frac12\min\curl*{\phi(\C-g(-1))+\phi(g(-1)-\C),\phi(g(1)+\C)+\phi(-g(1)-\C) } & \text{(Part \ref{part3_lemma:quasiconcave_even_general} of Lemma \ref{lemma:quasiconcave_even_general})}\\
    =& \frac12(\phi(\C-g(-1))+\phi(g(-1)-\C))
\end{align*}

Now for the other direction, assume that $\phi(\C-g(-1))+\phi(g(-1)-\C)=\phi(g(1)+\C)+\phi(-g(1)-\C)$ and $\min\curl*{\phi(\overline{A})+\phi(-\overline{A}),\phi(\underline{A})+\phi(-\underline{A}) }>\phi(\C-g(-1))+\phi(g(-1)-\C)$. Similarly,
\begin{align*}
    &\inf\limits_{\underline{A}\leq t\leq \overline{A}}\cC_{\phi}(t,\frac12)\\
    =& \frac12\min\curl*{\phi(\overline{A})+\phi(-\overline{A}),\phi(\underline{A})+\phi(-\underline{A}) }\\
    >&\frac12(\phi(\C-g(-1))+\phi(g(-1)-\C))\\
    =& \frac12\min\curl*{\phi(\C-g(-1))+\phi(g(-1)-\C),\phi(g(1)+\C)+\phi(-g(1)-\C) }\\
    =& \inf\limits_{g(-1)-\C\leq t\leq g(1)+\C}\cC_{\phi}(t,\frac12).
\end{align*}

For $\eta\in(\frac12,1]$,
\begin{align*}
    &\inf\limits_{g(-1)-\C\leq t\leq \overline{A}} \Cpae
    = \min\curl*{\cC_{\phi}(g(-1)-\C,\eta),\cC_{\phi}(\overline{A},\eta)} & \text{(Part \ref{part3_lemma:quasiconcave_even_general} of Lemma \ref{lemma:quasiconcave_even_general})}\\
    &\inf\limits_{g(-1)-\C\leq t\leq g(1)+\C} \Cpae
    = \min\curl*{\cC_{\phi}(g(-1)-\C,\eta),\cC_{\phi}(g(1)+\C,\eta)} & \text{(Part \ref{part3_lemma:quasiconcave_even_general} of Lemma \ref{lemma:quasiconcave_even_general})}\\
    &= \cC_{\phi}(g(1)+\C,\eta) & \text{(Part \ref{part5_lemma:quasiconcave_even_general} of Lemma \ref{lemma:quasiconcave_even_general})}
\end{align*}
Since $\phi$ is non-increasing, we have
\begin{align*}
    &\phi(-g(1)-\C)-\phi(g(1)+\C)+\phi(\overline{A})-\phi(-\overline{A})\\
    \geq&  \phi(-g(1)-\C)-\phi(g(1)+\C)+\phi(g(1)+\C)-\phi(-g(1)-\C)\\
    =&0.
\end{align*}
Then for $\eta\in(\frac12,1]$,
\begin{align*}
    &\cC_{\phi}(\overline{A},\eta)-\cC_{\phi}(g(1)+\C,\eta)\\
    =& (\phi(\overline{A})-\phi(-\overline{A})+\phi(-g(1)-\C)-\phi(g(1)+\C))\eta+\phi(-\overline{A})-\phi(-g(1)-\C)\\
    \geq& (\phi(\overline{A})-\phi(-\overline{A})+\phi(-g(1)-\C)-\phi(g(1)+\C))\frac12+\phi(-\overline{A})-\phi(-g(1)-\C)\\
    =& \frac12(\phi(\overline{A})+\phi(-\overline{A})-\phi(-g(1)-\C)-\phi(g(1)+\C))\\
    >&0.
\end{align*}
Again, by Part \ref{part5_lemma:quasiconcave_even_general} of Lemma \ref{lemma:quasiconcave_even_general}, for all $\eta\in (\frac12,1]$, $\cC_{\phi}(g(-1)-\C,\eta)-\cC_{\phi}(g(1)+\C,\eta)>0$.

As a result,  for $\eta\in (\frac12,1]$
\begin{align*}
    &\inf\limits_{g(-1)-\C\leq t\leq \overline{A}} \cC_{\phi}(t,\eta)  -\inf\limits_{g(-1)-\C\leq t\leq g(1)+\C}\cC_{\phi}(t,\eta)\\
    =& \min\curl*{\cC_{\phi}(g(-1)-\C,\eta)-\cC_{\phi}(g(1)+\C,\eta) , \cC_{\phi}(\overline{A},\eta)-\cC_{\phi}(g(1)+\C,\eta)}\\
    >&0.
\end{align*}

Finally, for $\eta\in[0,\frac12)$, by \text{Part \ref{part7_lemma:quasiconcave_even_general} of Lemma \ref{lemma:quasiconcave_even_general}}, we have $\cC_{\phi}(g(-1)-\C,\eta)<\cC_{\phi}(g(1)+\C,\eta)$ and 
\begin{align*}
    &\inf\limits_{\underline A\leq t\leq g(1)+\C} \Cpae
    = \min\curl*{\cC_{\phi}(\underline A,\eta),\cC_{\phi}(g(1)+\C,\eta)} & \text{(Part \ref{part3_lemma:quasiconcave_even_general} of Lemma \ref{lemma:quasiconcave_even_general})}\\
    &\inf\limits_{g(-1)-\C\leq t\leq g(1)+\C} \Cpae
    = \min\curl*{\cC_{\phi}(g(-1)-\C,\eta),\cC_{\phi}(g(1)+\C,\eta)} & \text{(Part \ref{part3_lemma:quasiconcave_even_general} of Lemma \ref{lemma:quasiconcave_even_general})}\\
    &=\cC_{\phi}(g(-1)-\C,\eta) .
\end{align*}
Since $\phi(\underline A)+\phi(-\underline A)>\phi(\C-g(-1))+\phi(g(-1)-\C)$ and $\phi$ is non-increasing, we have
\begin{align*}
    &\phi(\C-g(-1))-\phi(g(-1)-\C)+\phi(\underline A)-\phi(-\underline A)\\
    =&  \phi(\C-g(-1))-\phi(-\underline A)+\phi(\underline A)-\phi(g(-1)-\C)\\
    <&  \phi(\underline A)-\phi(g(-1)-\C)+\phi(\underline A)-\phi(g(-1)-\C)\\
    =&  2(\phi(\underline A)-\phi(g(-1)-\C))\\
    \leq&0.
\end{align*}
Then for $\eta\in[0,\frac12)$,
\begin{align*}
    &\cC_{\phi}(\underline A,\eta)-\cC_{\phi}(g(-1)-\C,\eta)\\
    =& (\phi(\underline A)-\phi(-\underline A)+\phi(\C-g(-1))-\phi(g(-1)-\C))\eta+\phi(-\underline A)-\phi(\C-g(-1))\\
    \geq& (\phi(\underline A)-\phi(-\underline A)+\phi(\C-g(-1))-\phi(g(-1)-\C))\frac12+\phi(-\underline A)-\phi(\C-g(-1))\\
   =& \frac12(\phi(\underline A)+\phi(-\underline A)-\phi(g(-1)-\C)-\phi(\C-g(-1)))\\
    >&0.
\end{align*}
Therefore,
\begin{equation*}
    \inf\limits_{\underline{A}\leq t\leq g(1)+\C}\cC_{\phi}(t,\eta) > \inf\limits_{g(-1)-\C\leq t\leq g(1)+\C}\cC_{\phi}(t,\eta) \text{ for all } \eta\in [0,\frac12).
\end{equation*}
\end{proof}

\subsection{Proof of 
Theorem~\ref{Thm:calibration_sup_convex}}
\label{app:calibration_sup_convex}
We first characterize the pseudo-calibration function of losses $(\ell, \ell_{\gamma})$ given a hypothesis set $\sH$. 
\begin{lemma}
\label{lemma:bar_delta_GN}
Given a hypothesis set $\sH$. Assume for any $\bx\in \sX$, there exists $f\in \sH$ such that $\inf_{\| \bx' - \bx \| \leq \gamma } f(\bx')>0$, and $f\in \sH$ such that $\sup_{\| \bx' - \bx \| \leq \gamma }f(\bx')<0$. For a surrogate loss $\ell$, the pseudo-calibration function of losses $(\ell, \ell_{\gamma})$ is $\hat{\delta}(\epsilon)=\inf_{\eta\in[0,1]}\Bar{\delta}(\epsilon,\eta)$, where 
\begin{equation*}
\Bar{\delta}(\epsilon,\eta) =
\begin{cases}
+\infty & \text{if} ~\epsilon>\max\curl*{\eta,1-\eta},\\
\inf\limits_{f\in\sH,\bx\in\sX\colon ~\underline{M}(f,\bx,\gamma)\leq 0 \leq \overline{M}(f,\bx,\gamma)}\Delta\cC_{\ell,\sH}(f,\bx,\eta) & \text{if} ~ |2\eta-1|<\epsilon\leq\max\curl*{\eta,1-\eta},\\
\inf\limits_{f\in\sH,\bx\in\sX\colon ~\underline{M}(f,\bx,\gamma)\leq 0 \leq \overline{M}(f,\bx,\gamma) \text{ or } (2\eta-1)(\underline{M}(f,\bx,\gamma))\leq 0} \Delta\cC_{\ell,\sH}(f,\bx,\eta) & \text{if} ~ \epsilon\leq|2\eta-1|,
\end{cases}
\end{equation*}
and $\underline{M}(f,\bx,\gamma)=\inf_{\bx'\colon \|\bx - \bx'\|\leq\gamma} f(\bx')$, $\overline{M}(f,\bx,\gamma)=\sup_{\bx'\colon \|\bx - \bx'\|\leq\gamma} f(\bx')$.
\end{lemma}

\begin{proof}
Let 
\begin{align*}
    \underline{M}(f,\bx,\gamma)\colon=\inf_{\bx'\colon \|\bx - \bx'\|\leq\gamma} f(\bx')\,,
\end{align*} 
and 
\begin{align*}
\overline{M}(f,\bx,\gamma)\colon=& -\inf_{\bx'\colon \|\bx - \bx'\|\leq\gamma} -f(\bx')\\
=&\sup_{\bx'\colon \|\bx - \bx'\|\leq\gamma} f(\bx')\,.
\end{align*} 
The inner $\ell_{\gamma}$-risk is 
\begin{equation*}
\begin{aligned}
  \cC_{\ell_{\gamma}}(f,\bx,\eta)
  &=\eta \mathds{1}_{\left\{\underline{M}(f,\bx,\gamma)\leq 0\right\}}+(1-\eta) \mathds{1}_{\left\{\overline{M}(f,\bx,\gamma)\geq 0\right\}}\\
  &=\begin{cases}
   1 & \text{if} ~ \underline{M}(f,\bx,\gamma)\leq 0 \leq \overline{M}(f,\bx,\gamma),\\
   \eta & \text{if} ~ \overline{M}(f,\bx,\gamma)<0,\\
   1-\eta & \text{if} ~ \underline{M}(f,\bx,\gamma)> 0.\\
  \end{cases}
\end{aligned}
\end{equation*}
Since $\sH$ satisfies the condition that for any $\bx\in \sX$, there exists $f\in \sH$ such that $\underline{M}(f,\bx,\gamma)>0$, and $f\in \sH$ such that $\overline{M}(f,\bx,\gamma)<0$, the pseudo-minimal inner $\ell_{\gamma}$-risk is
\begin{equation*}
     \cC^*_{\ell_{\gamma},\sH}(\eta)=\min\curl*{\eta,1-\eta}.
\end{equation*}
Then, it can be computed that
\begin{equation*}
\begin{aligned}
  \Delta\cC_{\ell_{\gamma},
    \sH}(f,\bx,\eta)=
    \begin{cases}
    \max\curl*{\eta,1-\eta} & \text{if} ~ \underline{M}(f,\bx,\gamma)\leq 0 \leq \overline{M}(f,\bx,\gamma),\\
   |2\eta-1|\mathds{1}_{(2\eta-1)(\underline{M}(f,\bx,\gamma))\leq 0} & \text{if} ~ \underline{M}(f,\bx,\gamma)>0 \text{~or~} \overline{M}(f,\bx,\gamma)<0.\\
    \end{cases}
\end{aligned}
\end{equation*}
By definition, for a fixed $\eta\in[0,1]$,
\begin{equation*}
    \bar{\delta}(\epsilon,\eta)=\inf\limits_{f\in\sH,\bx\in\sX}\curl*{\Delta\cC_{\ell,\sH}(f,\bx,\eta) \mid \Delta\cC_{\ell_{\gamma},
    \sH}(f,\bx,\eta)\geq\epsilon }
\end{equation*}
If $\epsilon>\max\curl*{\eta,1-\eta}$, then for all $f\in\sH,\bx\in\sX$, $\Delta\cC_{\ell_{\gamma},\sH}(f,\bx,\eta)<\epsilon$, which implies that $\bar{\delta}(\epsilon,\eta)=\infty$.
If $|2\eta-1|<\epsilon\leq\max\curl*{\eta,1-\eta}$, then $\Delta\cC_{\ell_{\gamma},\sH}(f,\bx,\eta)\geq\epsilon$ is achieved when $\underline{M}(f,\bx,\gamma)\leq 0 \leq \overline{M}(f,\bx,\gamma)$, which leads to $\bar{\delta}(\epsilon,\eta)=\inf_{f\in\sH,\bx\in\sX\colon ~\underline{M}(f,\bx,\gamma)\leq 0 \leq \overline{M}(f,\bx,\gamma)}\Delta\cC_{\ell,\sH}(f,\bx,\eta)$. 
If $\epsilon\leq|2\eta-1|$, then $\Delta\cC_{\ell_{\gamma},\sH}(f,\bx,\eta)\geq\epsilon$ is achieved when $\underline{M}(f,\bx,\gamma)\leq 0 \leq \overline{M}(f,\bx,\gamma)$ or $(2\eta-1)(\underline{M}(f,\bx,\gamma))\leq 0$. Therefore, $\bar{\delta}(\epsilon,\eta)=\inf_{f\in\sH,\bx\in\sX\colon ~\underline{M}(f,\bx,\gamma)\leq 0 \leq \overline{M}(f,\bx,\gamma) \text{ or } (2\eta-1)(\underline{M}(f,\bx,\gamma))\leq 0} \Delta\cC_{\ell,\sH}(f,\bx,\eta)$.
\end{proof}

\CalibrationSupConvex*

\begin{proof}
Suppose that $\tilde{\phi}$ is pseudo-$\sH$-calibrated with respect to $\ell_{\gamma}$. By Proposition~\ref{prop:calibration_function_positive}, $\tilde{\phi}$ is pseudo-$\sH$-calibrated with respect to $\ell_{\gamma}$ if and only if its pseudo-calibration function $\hat{\delta}$ satisfies $\hat{\delta}(\epsilon)>0$ for all $\epsilon>0$, which leads to $\Bar{\delta}(\epsilon,\eta)>0$ for all $\epsilon>0$ and $\eta\in [0,1]$. By Lemma~\ref{lemma:bar_delta_GN}, take $\eta=\frac12$, we obtain
\begin{equation*}
    	\inf\limits_{f\in\sH,\bx\in\sX\colon~\underline{M}(f,\bx,\gamma)\leq 0 \leq \overline{M}(f,\bx,\gamma)}\Delta\cC_{\tilde\phi,\sH}(f,\bx,\frac12)>0,
\end{equation*}
which is equivalent to
\begin{equation}
\inf\limits_{f\in\sH,\bx\in\sX\colon~\underline{M}(f,\bx,\gamma)\leq 0 \leq \overline{M}(f,\bx,\gamma)}\cC_{\tilde{\phi}}(f,\bx,\frac12)>
     \inf\limits_{f\in\sH,\bx\in\sX}\cC_{\tilde{\phi}}(f,\bx,\frac12)\,,
     \label{eq:larger_adv_GN}
\end{equation}
where $\underline{M}(f,\bx,\gamma)=\inf_{\bx'\colon \|\bx - \bx'\|\leq\gamma} f(\bx')$, $\overline{M}(f,\bx,\gamma)=\sup_{\bx'\colon \|\bx - \bx'\|\leq\gamma}f(\bx')$.
As shown by \citet{awasthi2020adversarial}, $\tilde{\phi}$ has the equivalent form
\begin{equation*}
	\tilde{\phi}(f,\bx,y)=\phi\left(\inf\limits_{\| \bx' - \bx \| \leq \gamma }\left(yf(\bx')\right)\right)\,.
\end{equation*}
By the definition of inner risk,
\begin{align}
   \cC_{\tilde{\phi}}(f,\bx,\frac12)=\frac12 (\phi(\underline{M}(f,\bx,\gamma))+\phi(-\overline{M}(f,\bx,\gamma)))\,.
\label{eq:phi_CCR_adv_GN} 
\end{align}
Since $\phi$ is convex, by Jensen's inequality,
\begin{equation*}
   \cC_{\tilde{\phi}}(f,\bx,\frac12)\geq \phi\left(\frac12 \underline{M}(f,\bx,\gamma) -\frac12 \overline{M}(f,\bx,\gamma) \right)=\phi\left(\frac12 (\underline{M}(f,\bx,\gamma)-\overline{M}(f,\bx,\gamma))\right)\geq \phi(0),
\end{equation*}
where the last inequality used the fact that
\begin{equation*}
    \frac12 (\underline{M}(f,\bx,\gamma)-\overline{M}(f,\bx,\gamma))\leq 0
\end{equation*}
and $\phi$ is non-increasing.
For $f=0$, we have $\underline{M}(f,\bx,\gamma)=\overline{M}(f,\bx,\gamma)=0$ and by \eqref{eq:phi_CCR_adv_GN},
\begin{equation*}
\cC_{\tilde{\phi}}(f,\bx,\frac12)=\frac12(\phi(0)+\phi(0))=\phi(0)\,.
\end{equation*}
Furthermore, when $\underline{M}(f,\bx,\gamma)=\overline{M}(f,\bx,\gamma)=0$, $\underline{M}(f,\bx,\gamma)\leq0\leq\overline{M}(f,\bx,\gamma)$ is satisfied.
Therefore, we obtain
\begin{equation*}
    \inf\limits_{f\in\sH,\bx\in\sX\colon ~\underline{M}(f,\bx,\gamma)\leq 0 \leq \overline{M}(f,\bx,\gamma)}\cC_{\tilde{\phi}}(f,\bx,\frac12)=
     \inf\limits_{f\in\sH,\bx\in\sX}\cC_{\tilde{\phi}}(f,\bx,\frac12)=\phi(0)\,,
\end{equation*}
where the minimum can be achieved by $f=0$, contradicting \eqref{eq:larger_adv_GN}. Therefore, $\tilde{\phi}$ is not pseudo-$\sH$-calibrated with respect to $\ell_{\gamma}$. By Corollary \ref{corollary:calibration_negative}, $\tilde{\phi}$ is not $\sH$-calibrated with respect to $\ell_{\gamma}$.
\end{proof}

\subsection{Proof of Theorem~\ref{Thm:pos1_GN}}
\label{app:positive_NN}
As with the Proof of Theorem~\ref{Thm:quasiconcave_calibrate_general}, we first give the equivalent conditions of pseudo-calibration based on inner risk of $\phi$ and $\sH_{\mathrm{NN}}$.
\begin{lemma}
\label{lemma:equivalent1_GN}
Given a hypothesis set $\sH$. Assume for any $\bx\in \sX$, there exists $f\in \sH$ such that $\inf_{\| \bx' - \bx \| \leq \gamma } f(\bx')>0$, and $f\in \sH$ such that $\sup_{\bx'\colon \|\bx - \bx'\|\leq\gamma}f(\bx')<0$. Let $\ell$ be a surrogate loss function. Then $\ell$ is pseudo-$\sH$-calibrated with respect to $\ell_{\gamma}$ if and only if 
\begin{align*}
    \inf\limits_{f\in\sH,\bx\in\sX\colon ~\underline{M}(f,\bx,\gamma)\leq 0 \leq \overline{M}(f,\bx,\gamma)}\cC_{\ell}(f,\bx,\frac12)> &\inf\limits_{f\in\sH,\bx\in\sX}\cC_{\ell}(f,\bx,\frac12)\,,\text{and}\\
    \inf\limits_{f\in\sH,\bx\in\sX\colon ~\underline{M}(f,\bx,\gamma)\leq0}\cC_{\ell}(f,\bx,\eta) > &\inf\limits_{f\in\sH,\bx\in\sX}\cC_{\ell}(f,\bx,\eta) \text{ for all } \eta\in (\frac12,1]\,,\text{and}\\
    \inf\limits_{f\in\sH,\bx\in\sX\colon ~\overline{M}(f,\bx,\gamma)\geq0}\cC_{\ell}(f,\bx,\eta) > &\inf\limits_{f\in\sH,\bx\in\sX}\cC_{\ell}(f,\bx,\eta) \text{ for all } \eta\in [0,\frac12)\,.\\
\end{align*}
where $ \underline{M}(f,\bx,\gamma)=\inf_{\bx'\colon \|\bx - \bx'\|\leq\gamma} f(\bx')$,
$\overline{M}(f,\bx,\gamma)=\sup_{\bx'\colon \|\bx - \bx'\|\leq\gamma}f(\bx')$.
\end{lemma}

\begin{proof}
Let $\hat{\delta}$ be the pseudo-calibration function of $(\ell,\ell_{\gamma})$ for the hypothesis set $\sH$. By Lemma \ref{lemma:bar_delta_GN},  $\hat{\delta}(\epsilon)=\inf_{\eta\in[0,1]}\Bar{\delta}(\epsilon,\eta)$, where 
\begin{equation*}
\Bar{\delta}(\epsilon,\eta) =
\begin{cases}
+\infty & \text{if} ~\epsilon>\max\curl*{\eta,1-\eta},\\
\inf\limits_{f\in\sH,\bx\in\sX\colon~\underline{M}(f,\bx,\gamma)\leq 0 \leq \overline{M}(f,\bx,\gamma)}\Delta\cC_{\ell,\sH}(f,\bx,\eta) & \text{if} ~ |2\eta-1|<\epsilon\leq\max\curl*{\eta,1-\eta},\\
\inf\limits_{f\in\sH,\bx\in\sX\colon~\underline{M}(f,\bx,\gamma)\leq 0 \leq \overline{M}(f,\bx,\gamma) \text{ or } (2\eta-1)(\underline{M}(f,\bx,\gamma))\leq 0} \Delta\cC_{\ell,\sH}(f,\bx,\eta) & \text{if} ~ \epsilon\leq|2\eta-1|.
\end{cases}
\end{equation*}
By Proposition \ref{prop:calibration_function_positive}, $\ell$ is pseudo-$\sH$-calibrated with respect to $\ell_{\gamma}$ if and only if its pseudo-calibration function $\hat{\delta}$ satisfies $\hat{\delta}(\epsilon)>0$ for all $\epsilon>0$. This is equivalent to $\Bar{\delta}(\epsilon,\eta)>0$ for all $\epsilon>0$ and $\eta\in [0,1]$.\\
For $\eta=\frac12$, we have
\begin{equation}
\Bar{\delta}(\epsilon,\frac12)>0 \text{ for all } \epsilon>0 \Leftrightarrow \inf\limits_{f\in\sH,\bx\in\sX\colon ~\underline{M}(f,\bx,\gamma)\leq 0 \leq \overline{M}(f,\bx,\gamma)}\cC_{\ell}(f,\bx,\frac12)> \inf\limits_{f\in\sH,\bx\in\sX}\cC_{\ell}(f,\bx,\frac12)\,.
\label{eq:keycondition1_GN}
\end{equation}
For $1\geq\eta>\frac12$, we have $|2\eta-1|=2\eta-1$, $\max\curl*{\eta,1-\eta}=\eta$, and
\begin{equation*}
\begin{aligned}
   &\inf\limits_{f\in\sH,\bx\in\sX\colon ~\underline{M}(f,\bx,\gamma)\leq 0 \leq \overline{M}(f,\bx,\gamma) \text{ or } (2\eta-1)(\underline{M}(f,\bx,\gamma))\leq 0} \Delta\cC_{\ell,\sH}(f,\bx,\eta)\\
   &=
   \inf\limits_{f\in\sH,\bx\in\sX\colon ~\underline{M}(f,\bx,\gamma)\leq 0} \Delta\cC_{\ell,\sH}(f,\bx,\eta)\,. 
\end{aligned}
\end{equation*}
Therefore, $\Bar{\delta}(\epsilon,\eta)>0 \text{ for all } \epsilon>0 \text{ and } \eta\in(\frac12,1]$ if and only if 
\small
\begin{equation*}
\begin{cases}
\inf\limits_{f\in\sH,\bx\in\sX\colon ~\underline{M}(f,\bx,\gamma)\leq 0 \leq \overline{M}(f,\bx,\gamma)}\cC_{\ell}(f,\bx,\eta)> \inf\limits_{f\in\sH,\bx\in\sX}\cC_{\ell}(f,\bx,\eta) &\text{ for all } \eta\in(\frac12,1] \text{ such that } 2\eta-1<\epsilon\leq \eta,\\
\inf\limits_{f\in\sH,\bx\in\sX\colon ~\underline{M}(f,\bx,\gamma)\leq0}\cC_{\ell}(f,\bx,\eta) > \inf\limits_{f\in\sH,\bx\in\sX}\cC_{\ell}(f,\bx,\eta) &\text{ for all } \eta\in(\frac12,1] \text{ such that } \epsilon\leq 2\eta-1,
\end{cases}
\end{equation*}
\normalsize
for all $\epsilon>0$, which is equivalent to
\small
\begin{equation}
\begin{cases}
\inf\limits_{f\in\sH,\bx\in\sX\colon ~\underline{M}(f,\bx,\gamma)\leq 0 \leq \overline{M}(f,\bx,\gamma)}\cC_{\ell}(f,\bx,\eta)> \inf\limits_{f\in\sH,\bx\in\sX}\cC_{\ell}(f,\bx,\eta) &\text{ for all } \eta\in(\frac12,1] \text{ such that } \epsilon\leq \eta < \frac{\epsilon+1}{2},\\
\inf\limits_{f\in\sH,\bx\in\sX\colon ~\underline{M}(f,\bx,\gamma)\leq0}\cC_{\ell}(f,\bx,\eta) > \inf\limits_{f\in\sH,\bx\in\sX}\cC_{\ell}(f,\bx,\eta) &\text{ for all } \eta\in(\frac12,1] \text{ such that } \frac{\epsilon+1}{2}\leq \eta,
\end{cases}
\label{eq:condition1 in proof_GN}
\end{equation}
\normalsize
for all $\epsilon>0$.
We observe that
\begin{equation*}
\begin{aligned}
    &\left\{\eta\in(\frac12,1]\Bigg|\epsilon\leq \eta < \frac{\epsilon+1}{2},\epsilon>0\right\}=\left\{\frac12<\eta\leq1\right\}\,, \text{ and}\\
    &\left\{\eta\in(\frac12,1]\Bigg|\frac{\epsilon+1}{2}\leq \eta, \epsilon>0\right\}=\left\{\frac12<\eta\leq1\right\}\,, \text{ and}\\
    &\inf\limits_{f\in\sH,\bx\in\sX\colon ~\underline{M}(f,\bx,\gamma)\leq 0 \leq \overline{M}(f,\bx,\gamma)}\cC_{\ell}(f,\bx,\eta) \geq \inf\limits_{f\in\sH,\bx\in\sX\colon~\underline{M}(f,\bx,\gamma)\leq0}\cC_{\ell}(f,\bx,\eta) \text{ for all } \eta\,.
\end{aligned}
\end{equation*}
Therefore we reduce the above condition \eqref{eq:condition1 in proof_GN} as
\begin{equation}
    \inf\limits_{f\in\sH,\bx\in\sX\colon ~\underline{M}(f,\bx,\gamma)\leq0}\cC_{\ell}(f,\bx,\eta) > \inf\limits_{f\in\sH,\bx\in\sX}\cC_{\ell}(f,\bx,\eta) \text{ for all } \eta\in (\frac12,1]\,.
    \label{eq:keycondition2_GN}
\end{equation}
For $\frac12>\eta\geq0$, we have $|2\eta-1|=1-2\eta$, $\max\curl*{\eta,1-\eta}=1-\eta$, and
\begin{equation*}
\begin{aligned}
   &\inf\limits_{f\in\sH,\bx\in\sX\colon ~\underline{M}(f,\bx,\gamma)\leq 0 \leq \overline{M}(f,\bx,\gamma) \text{ or } (2\eta-1)(\underline{M}(f,\bx,\gamma))\leq 0} \Delta\cC_{\ell,\sH}(f,\bx,\eta)\\
   &=
   \inf\limits_{f\in\sH,\bx\in\sX\colon ~\overline{M}(f,\bx,\gamma)\geq 0} \Delta\cC_{\ell,\sH}(f,\bx,\eta)\,. 
\end{aligned}
\end{equation*}
Therefore, $\Bar{\delta}(\epsilon,\eta)>0 \text{ for all } \epsilon>0 \text{ and } \eta\in[0,\frac12)$ if and only if 
\small
\begin{equation*}
\begin{cases}
\inf\limits_{f\in\sH,\bx\in\sX\colon ~\underline{M}(f,\bx,\gamma)\leq 0 \leq \overline{M}(f,\bx,\gamma)}\cC_{\ell}(f,\bx,\eta)> \inf\limits_{f\in\sH,\bx\in\sX}\cC_{\ell}(f,\bx,\eta) &\text{ for all } \eta\in[0,\frac12) \text{ such that } 1-2\eta<\epsilon\leq 1-\eta,\\
\inf\limits_{f\in\sH,\bx\in\sX\colon ~\overline{M}(f,\bx,\gamma)\geq0}\cC_{\ell}(f,\bx,\eta) > \inf\limits_{f\in\sH,\bx\in\sX}\cC_{\ell}(f,\bx,\eta) &\text{ for all } \eta\in[0,\frac12) \text{ such that } \epsilon\leq 1-2\eta,
\end{cases}
\end{equation*}
\normalsize
for all $\epsilon>0$, which is equivalent to
\small
\begin{equation}
\begin{cases}
\inf\limits_{f\in\sH,\bx\in\sX\colon ~\underline{M}(f,\bx,\gamma)\leq 0 \leq \overline{M}(f,\bx,\gamma)}\cC_{\ell}(f,\bx,\eta)> \inf\limits_{f\in\sH,\bx\in\sX}\cC_{\ell}(f,\bx,\eta) &\text{ for all } \eta\in[0,\frac12) \text{ such that } \frac{1-\epsilon}{2}< \eta \leq 1-\epsilon,\\
\inf\limits_{f\in\sH,\bx\in\sX\colon ~\overline{M}(f,\bx,\gamma)\geq0}\cC_{\ell}(f,\bx,\eta) > \inf\limits_{f\in\sH,\bx\in\sX}\cC_{\ell}(f,\bx,\eta) &\text{ for all } \eta\in[0,\frac12) \text{ such that } \eta\leq \frac{1-\epsilon}{2},
\end{cases}
\label{eq:condition2 in proof_GN}
\end{equation}
\normalsize
for all $\epsilon>0$.
We observe that
\begin{equation*}
\begin{aligned}
    &\left\{\eta\in[0,\frac12)\Bigg|\frac{1-\epsilon}{2}< \eta \leq 1-\epsilon,\epsilon>0\right\}=\left\{0\leq\eta<\frac12\right\}\,, \text{ and}\\
    &\left\{\eta\in[0,\frac12)\Bigg|\eta\leq \frac{1-\epsilon}{2}, \epsilon>0\right\}=\left\{0\leq\eta<\frac12\right\}\,, \text{ and}\\
    &\inf\limits_{f\in\sH,\bx\in\sX\colon ~\underline{M}(f,\bx,\gamma)\leq 0 \leq \overline{M}(f,\bx,\gamma)}\cC_{\ell}(f,\bx,\eta) \geq \inf\limits_{f\in\sH,\bx\in\sX\colon ~\overline{M}(f,\bx,\gamma)\geq0}\cC_{\ell}(f,\bx,\eta)  \text{ for all } \eta\,.
\end{aligned}
\end{equation*}
Therefore we reduce the above condition \eqref{eq:condition2 in proof_GN} as
\begin{equation}
   \inf\limits_{f\in\sH,\bx\in\sX\colon ~\overline{M}(f,\bx,\gamma)\geq0}\cC_{\ell}(f,\bx,\eta) > \inf\limits_{f\in\sH,\bx\in\sX}\cC_{\ell}(f,\bx,\eta) \text{ for all } \eta\in [0,\frac12)\,.
    \label{eq:keycondition3_GN}
\end{equation}
To sum up, by \eqref{eq:keycondition1_GN}, \eqref{eq:keycondition2_GN} and \eqref{eq:keycondition3_GN}, we conclude the proof.
\end{proof}

\PosGN*

\begin{proof}
By Lemma \ref{lemma:equivalent1_GN}, $\tilde{\phi}_{\rho}$ is pseudo-$\sH_{\mathrm{NN}}$-calibrated with respect to $\ell_{\gamma}$ if and only if 
\begin{equation*}
\begin{aligned}
    \inf\limits_{f\in\sH,\bx\in\sX\colon ~\underline{M}(f,\bx,\gamma)\leq 0 \leq \overline{M}(f,\bx,\gamma)}\cC_{\tilde{\phi}_{\rho}}(f,\bx,\frac12)> &\inf\limits_{f\in\sH,\bx\in\sX}\cC_{\tilde{\phi}_{\rho}}(f,\bx,\frac12)\,,\text{and}\\
    \inf\limits_{f\in\sH,\bx\in\sX\colon ~\underline{M}(f,\bx,\gamma)\leq0}\cC_{\tilde{\phi}_{\rho}}(f,\bx,\eta) > &\inf\limits_{f\in\sH,\bx\in\sX}\cC_{\tilde{\phi}_{\rho}}(f,\bx,\eta) \text{ for all } \eta\in (\frac12,1]\,,\text{and}\\
    \inf\limits_{f\in\sH,\bx\in\sX\colon ~\overline{M}(f,\bx,\gamma)\geq0}\cC_{\tilde{\phi}_{\rho}}(f,\bx,\eta) > &\inf\limits_{f\in\sH,\bx\in\sX}\cC_{\tilde{\phi}_{\rho}}(f,\bx,\eta) \text{ for all } \eta\in [0,\frac12)\,.\\
\end{aligned}
\end{equation*}
where 
$ \underline{M}(f,\bx,\gamma)=\inf_{\bx'\colon \|\bx - \bx'\|\leq\gamma} f(\bx')$,
$\overline{M}(f,\bx,\gamma)=\sup_{\bx'\colon \|\bx - \bx'\|\leq\gamma} f(\bx')$.
As shown by \citet{awasthi2020adversarial}, $\tilde{\phi}_{\rho}$ has the equivalent form
\begin{equation*}
	\tilde{\phi}_{\rho}(f,\bx,y)=\phi_{\rho}\left(\inf\limits_{\bx'\colon \|\bx - \bx'\|\leq\gamma}\left(yf(\bx')\right)\right)\,.
\end{equation*}
The inner $\tilde{\phi}_{\rho}$-risk is
\begin{align*}
   \cC_{\tilde{\phi}_{\rho}}(f,\bx,\eta)=\eta \phi_{\rho}(\underline{M}(f,\bx,\gamma))+(1-\eta)\phi_{\rho}(-\overline{M}(f,\bx,\gamma))\,.
\end{align*}
Next we analyze three cases:
\begin{itemize}
    \item 
    When $\eta=\frac12$, since $\phi_{\rho}$ is non-increasing, 
    \begin{align*}
       &\inf\limits_{f\in\sH,\bx\in\sX\colon ~\underline{M}(f,\bx,\gamma)\leq 0 \leq \overline{M}(f,\bx,\gamma)}\cC_{\tilde{\phi}_{\rho}}(f,\bx,\frac12)\\
       &= \inf\limits_{f\in\sH,\bx\in\sX\colon ~\underline{M}(f,\bx,\gamma)\leq 0 \leq \overline{M}(f,\bx,\gamma)}\frac12 \phi_{\rho}(\underline{M}(f,\bx,\gamma))+\frac12\phi_{\rho}(-\overline{M}(f,\bx,\gamma))\\
       &\geq \frac12 \phi_{\rho}(0)+\frac12 \phi_{\rho}(0)=\phi_{\rho}(0)=1\,.
    \end{align*}
    Take $f=0\in \sH_{\mathrm{NN}}$, then $\underline{M}(f,\bx,\gamma)=\overline{M}(f,\bx,\gamma)=0$, $\cC_{\tilde{\phi}_{\rho}}(f,\bx,\frac12)=\frac12 \phi_{\rho}(0)+\frac12 \phi_{\rho}(0)=\phi_{\rho}(0)=1$. Therefore
    \begin{align*}
     \inf\limits_{f\in\sH,\bx\in\sX\colon ~\underline{M}(f,\bx,\gamma)\leq 0 \leq \overline{M}(f,\bx,\gamma)}\cC_{\tilde{\phi}_{\rho}}(f,\bx,\frac12)=1\,.   
    \end{align*}
    Let $\bx \in \sX$ such that $\|\bx\|=1$, $\bw_j=W\bx$, $u_j=\frac{\Lambda}{n}$, $j=1,\dots,n$. Then for any $\bs\in \curl*{\bs\colon\|\bs\|\leq 1}$, $\bw_j\cdot(\bx+\gamma \bs)=W(\bx\cdot \bx+\gamma(\bx\cdot \bs))\geq W(\|\bx\|^2-\gamma\|\bx\|\|\bs\|)\geq W(1-\gamma)>0.$ Therefore, we obtain $\underline{M}(f,\bx,\gamma)=\inf_{\|\bs\|\leq 1}\sum_{j=1}^n u_j\left(\bw_j \cdot (\bx+\gamma \bs)\right)_{+}\geq \Lambda W(1-\gamma)>0$ and $-\overline{M}(f,\bx,\gamma)\leq-\underline{M}(f,\bx,\gamma)<0$. Then $\cC_{\tilde{\phi}_{\rho}}(f,\bx,\frac12)=\frac12\phi_{\rho}(\underline{M}(f,\bx,\gamma))+\frac12\phi_{\rho}(-\overline{M}(f,\bx,\gamma))\leq\frac12\times \phi_{\rho}(\Lambda W(1-\gamma)) +\frac12 \times 1<1$. Therefore
    \begin{equation}
        \inf\limits_{f\in\sH,\bx\in\sX}\cC_{\tilde{\phi}_{\rho}}(f,\bx,\frac12)
        \leq \frac12 \phi_{\rho}(\Lambda W(1-\gamma)) +\frac12<
        1=\inf\limits_{f\in\sH,\bx\in\sX\colon ~\underline{M}(f,\bx,\gamma)\leq 0 \leq \overline{M}(f,\bx,\gamma)}\cC_{\tilde{\phi}_{\rho}}(f,\bx,\frac12)\,.
        \label{eq:RNN1}
    \end{equation}
    \item 
    When $\eta\in (\frac12,1]$, since $\phi_{\rho}$ is non-increasing and $-\overline{M}(f,\bx,\gamma)\leq -\underline{M}(f,\bx,\gamma)$,
    \begin{align*}
    &\inf\limits_{f\in\sH,\bx\in\sX\colon ~\underline{M}(f,\bx,\gamma)\leq0}\cC_{\tilde{\phi}_{\rho}}(f,\bx,\eta)\\
    &=\inf\limits_{f\in\sH,\bx\in\sX\colon ~\underline{M}(f,\bx,\gamma)\leq0} \eta\phi_{\rho}(\underline{M}(f,\bx,\gamma))+(1-\eta)\phi_{\rho}(-\overline{M}(f,\bx,\gamma))\\
    &=\inf\limits_{f\in\sH,\bx\in\sX\colon ~\underline{M}(f,\bx,\gamma)\leq0}\eta\phi_{\rho}(\underline{M}(f,\bx,\gamma))+(1-\eta)\phi_{\rho}(-\underline{M}(f,\bx,\gamma))\\
    &\qquad+(1-\eta)(\phi_{\rho}(-\overline{M}(f,\bx,\gamma))-\phi_{\rho}(-\underline{M}(f,\bx,\gamma)))\\
    &=\inf\limits_{f\in\sH,\bx\in\sX\colon ~\underline{M}(f,\bx,\gamma)\leq0}\eta\phi_{\rho}(\underline{M}(f,\bx,\gamma))+(1-\eta)\phi_{\rho}(-\underline{M}(f,\bx,\gamma))\\
    &\geq \eta\,.
\end{align*}
Let $\bx \in \sX$ such that $\|\bx\|=1$, $\bw_j=W\bx$, $u_j=\frac{\Lambda}{n}$, $j=1,\dots,n$. Then for any $\bs\in \curl*{\bs\colon\|\bs\|\leq 1}$, $\bw_j\cdot(\bx+\gamma \bs)=W(\bx\cdot \bx+\gamma(\bx\cdot \bs))\geq W(\|\bx\|^2-\gamma\|\bx\|\|\bs\|)\geq W(1-\gamma)>0.$ Since $\Lambda W(1-\gamma)\geq \rho$, we obtain $\underline{M}(f,\bx,\gamma)=\inf_{\|\bs\|\leq 1}\sum_{j=1}^n u_j\left(\bw_j \cdot (\bx+\gamma \bs)\right)_{+}\geq \Lambda W(1-\gamma)\geq \rho$ and $-\overline{M}(f,\bx,\gamma)\leq-\underline{M}(f,\bx,\gamma)\leq-\rho$. Then 
\[\cC_{\tilde{\phi}_{\rho}}(f,\bx,\eta)=\eta\phi_{\rho}(\underline{M}(f,\bx,\gamma))+(1-\eta)\phi_{\rho}(-\overline{M}(f,\bx,\gamma))=\eta\times 0 +(1-\eta) \times 1=1-\eta.\] Therefore
    \begin{equation}
        \inf\limits_{f\in\sH,\bx\in\sX}\cC_{\tilde{\phi}_{\rho}}(f,\bx,\eta)\leq 1-\eta
        <\eta
        \leq\inf\limits_{f\in\sH,\bx\in\sX\colon ~\underline{M}(f,\bx,\gamma)\leq 0}\cC_{\tilde{\phi}_{\rho}}(f,\bx,\eta)\,.
        \label{eq:RNN2}
    \end{equation}
     \item 
    When $\eta\in [0,\frac12)$, since $\phi_{\rho}$ is non-increasing and $\underline{M}(f,\bx,\gamma)\leq \overline{M}(f,\bx,\gamma)$,
    \begin{align*}
    &\inf\limits_{f\in\sH,\bx\in\sX\colon ~\overline{M}(f,\bx,\gamma)\geq0}\cC_{\tilde{\phi}_{\rho}}(f,\bx,\eta)\\
    &=\inf\limits_{f\in\sH,\bx\in\sX\colon ~\overline{M}(f,\bx,\gamma)\geq0}\eta\phi_{\rho}(\underline{M}(f,\bx,\gamma))+(1-\eta)\phi_{\rho}(-\overline{M}(f,\bx,\gamma))\\
    &=\inf\limits_{f\in\sH,\bx\in\sX\colon ~\overline{M}(f,\bx,\gamma)\geq0} \eta\phi_{\rho}(\overline{M}(f,\bx,\gamma))+(1-\eta)\phi_{\rho}(-\overline{M}(f,\bx,\gamma))+\eta(\phi_{\rho}(\underline{M}(f,\bx,\gamma))-\phi_{\rho}(\overline{M}(f,\bx,\gamma)))\\
    &\geq\inf\limits_{f\in\sH,\bx\in\sX\colon ~\overline{M}(f,\bx,\gamma)\geq0} \eta\phi_{\rho}(\overline{M}(f,\bx,\gamma))+(1-\eta)\phi_{\rho}(-\overline{M}(f,\bx,\gamma))\\
    & \geq 1-\eta\,.
\end{align*}
Let $\bx \in \sX$ such that $\|\bx\|=1$, $\bw_j=W\bx$, $u_j=-\frac{\Lambda}{n}$, $j=1,\dots,n$. Then for any $\bs\in \curl*{\bs\colon\|\bs\|\leq 1}$, $\bw_j\cdot(\bx+\gamma \bs)=W(\bx\cdot \bx+\gamma(\bx\cdot \bs))\geq W(\|\bx\|^2-\gamma\|\bx\|\|\bs\|)\geq W(1-\gamma)>0.$ Since $\Lambda W(1-\gamma)\geq \rho$, we obtain  $\overline{M}(f,\bx,\gamma)=\sup_{\|\bs\|\leq 1}\sum_{j=1}^n u_j\left(\bw_j \cdot (\bx+\gamma \bs)\right)_{+}\leq -\Lambda W(1-\gamma)\leq -\rho$ and $\underline{M}(f,\bx,\gamma)\leq\overline{M}(f,\bx,\gamma)\leq-\rho$. Then $\cC_{\tilde{\phi}_{\rho}}(f,\bx,\eta)=\eta\phi_{\rho}(\underline{M}(f,\bx,\gamma))+(1-\eta)\phi_{\rho}(-\overline{M}(f,\bx,\gamma))=\eta\times 1 +(1-\eta) \times 0=\eta$. Therefore
    \begin{equation}
        \inf\limits_{f\in\sH,\bx\in\sX}\cC_{\tilde{\phi}_{\rho}}(f,\bx,\eta)\leq \eta<
        1-\eta
        \leq\inf\limits_{f\in\sH,\bx\in\sX\colon ~ \overline{M}(f,\bx,\gamma)\geq0}\cC_{\tilde{\phi}_{\rho}}(f,\bx,\eta)\,.
        \label{eq:RNN3}
    \end{equation}
\end{itemize}
To sum up, by \eqref{eq:RNN1}, \eqref{eq:RNN2} and \eqref{eq:RNN3}, we conclude the proof.
\end{proof}

\subsection{Proof of Theorem~\ref{Thm:consistent_linear}}
\label{app:consistent_linear}

\ConsistentLinear*

\begin{proof}
Let $\bx$ follow the uniform distribution on the unit circle. Denote $\bx=(\cos(\theta),\sin(\theta))^{\top},\theta\in [0,2\pi)$ and $f(\bx)=\bw\cdot \bx$, $\bw=(\cos(t),\sin(t))^{\top},t\in [0,2\pi), f\in\sH_{\mathrm{lin}}=\curl*{\bx\rightarrow \bw\cdot \bx\mid \|\bw\|_2=1}$. We set the label of a point $\bx$ as follows: if $\theta \in \left(\sigma,\pi\right)$, where $\sigma \in \left(0,\pi\right)$, then set $y=-1$ with probability $\frac34$ and $y=1$ with probability $\frac14$; if $\theta \in \left(0,\sigma\right) \text{ or } \left(\sigma+\pi,2\pi\right)$, then set $y=1$; if $\theta \in \left(\pi,\sigma+\pi\right)$, then set $y=-1$.

Let $\eta\colon\sX \rightarrow [0,1]$ be a measurable function such that $\eta(X)=\mathbb{P}(Y=1\mid X)$. For $\ell_{\gamma}(\tau) = \mathds{1}_{\tau\leq\gamma}$, we want to solve
\begin{equation*}
	\cR^*_{\ell_{\gamma},\sH_{\mathrm{lin}}}=\min_{f\in\sH_{\mathrm{lin}}} \cR_{\ell_{\gamma}}(f)=\min_{f\in\sH_{\mathrm{lin}}}\mathbb{E}_{X}[\ell_{\gamma}(f(X))\eta + \ell_{\gamma}(-f(X))(1-\eta)]\,.
\end{equation*}
Let $\eta'\colon\Theta \rightarrow [0,1]$ be a measurable function such that $\eta'=\mathbb{P}(Y=1|\Theta)$, $\Theta\sim \cU(0,2\pi)$. In our example, we have
\begin{equation*}
\eta'=
\begin{cases}
	\frac14 & \theta\in \left(\sigma,\pi\right)\,,\\
	1 & \theta\in \left(0,\sigma\right) \text{ or } \theta\in \left(\sigma+\pi,2\pi\right)\,,\\
	0 & \theta\in \left(\pi,\sigma+\pi\right)\,.
\end{cases}
\end{equation*}
Therefore we obtain 
\begingroup
\allowdisplaybreaks
\begin{align*}
 &\cR^*_{\ell_{\gamma},\sH_{\mathrm{lin}}}=\min_{t\in [0,2\pi)}\mathbb{E}_{\Theta}[\ell_{\gamma}(\cos(\Theta-t))\eta' + \ell_{\gamma}(-\cos(\Theta-t))(1-\eta')]\\
 &=\frac{1}{2\pi}\min_{t\in [0,2\pi)}\int_{\sigma}^{\pi} \frac14\ell_{\gamma}(\cos(\theta-t)) + \frac34 \ell_{\gamma}(-\cos(\theta-t))~d\theta +\int_{\sigma-\pi}^{\sigma} \ell_{\gamma}(\cos(\theta-t))~d\theta\\
 &\qquad +\int_{-\pi}^{\sigma-\pi} \ell_{\gamma}(-\cos(\theta-t))~d\theta\\
 &=\frac{1}{2\pi}\min_{t\in [0,2\pi)} \int_{\sigma}^{\pi} \frac14\ell_{\gamma}(\cos(\theta-t)) + \frac34 \ell_{\gamma}(-\cos(\theta-t))d\theta +\int_{\sigma-\pi}^{0} \ell_{\gamma}(\cos(\theta-t))d\theta\\
 &\qquad +\int_{0}^{\sigma} \ell_{\gamma}(\cos(\theta-t))d\theta +\int_{0}^{\sigma} \ell_{\gamma}(\cos(\theta-t))d\theta\\
 &=\frac{1}{2\pi}\min_{t\in [0,2\pi)} \int_{\sigma}^{\pi} \frac14\ell_{\gamma}(\cos(\theta-t))~d\theta
 + \int_{\sigma-\pi}^{0}\frac34 \ell_{\gamma}(\cos(\theta-t))~d\theta +\int_{\sigma-\pi}^{0} \ell_{\gamma}(\cos(\theta-t))~d\theta\\
  &\qquad+\int_{0}^{\sigma} 2\ell_{\gamma}(\cos(\theta-t))~d\theta\\
 &=\frac{1}{2\pi}\min_{t\in [0,2\pi)} \int_{\sigma}^{\pi} \frac14\ell_{\gamma}(\cos(\theta-t))~d\theta +\int_{\sigma-\pi}^{0} \frac74\ell_{\gamma}(\cos(\theta-t))~d\theta
 +\int_{0}^{\sigma} \frac74\ell_{\gamma}(\cos(\theta-t))~d\theta\\
  &\qquad+\int_{0}^{\sigma} \frac14\ell_{\gamma}(\cos(\theta-t))~d\theta\\
 &=\frac{1}{2\pi}\min_{t\in [0,2\pi)} \int_{0}^{\pi} \frac14\ell_{\gamma}(\cos(\theta-t))~d\theta +\int_{\sigma-\pi}^{\sigma} \frac74\ell_{\gamma}(\cos(\theta-t))~d\theta
\stepcounter{equation}\tag{\theequation}\label{target3}
 \\
 &=\frac{1}{2\pi}\min_{t\in [0,2\pi)} \int_{0}^{\pi} \frac14\ell_{\gamma}(\cos(\theta-t))~d\theta +\int_{0}^{\pi} \frac74\ell_{\gamma}(-\cos(\theta-t+\sigma))~d\theta\\
 &=\frac{1}{2\pi}\min_{t\in [0,2\pi)} \int_{-t}^{\pi-t} \frac14\ell_{\gamma}(\cos(\theta))~d\theta + \frac74\ell_{\gamma}(-\cos(\theta+\sigma))~d\theta\\
  &=\frac{1}{2\pi}\min_{t\in [0,2\pi)} \int_{-t}^{\pi-t} \frac14 \mathds{1}_{\cos(\theta)\leq \gamma}
 + \frac74 \mathds{1}_{-\cos(\theta+\sigma)\leq \gamma} ~d\theta\,.
\end{align*}
\endgroup
Take $\gamma=\cos(\frac{\sigma}{2})\in (0,1)$. 
For $\sigma\in (0,\frac{\pi}{2}]$, we analyze six cases:
\begin{itemize}
    \item When $-t\in [-\frac{3\sigma}{2},-\frac{\sigma}{2}]$,
    \begin{align*}
      &\int_{-t}^{\pi-t} \frac14 \mathds{1}_{\cos(\theta)\leq \gamma}
 + \frac74 \mathds{1}_{-\cos(\theta+\sigma)\leq \gamma} ~d\theta\\
 &=  \int_{-t}^{-\frac{\sigma}{2}} \frac{1}{4}+\frac{7}{4}~d\theta
 +\int_{-\frac{\sigma}{2}}^{\frac{\sigma}{2}} \frac{7}{4}~d\theta
 +  \int_{\frac{\sigma}{2}}^{-\frac{3\sigma}{2}+\pi} \frac{1}{4}+\frac{7}{4}~d\theta
 + \int_{-\frac{3\sigma}{2}+\pi}^{\pi-t} \frac{1}{4}~d\theta\\
& =2\pi-\frac{23}{8}\sigma+\frac{7}{4}t\geq 2\pi-2\sigma
  \end{align*}
  where the equality is achieved when $t=\frac{\sigma}{2}$.
 \item When $-t\in [-\frac{\sigma}{2},\frac{\sigma}{2}]$,
\begin{align*}
     & \int_{-t}^{\pi-t} \frac14 \mathds{1}_{\cos(\theta)\leq \gamma}
 + \frac74 \mathds{1}_{-\cos(\theta+\sigma)\leq \gamma} ~d\theta\\
 &=  \int_{-t}^{\frac{\sigma}{2}} \frac{7}{4}~d\theta
 +\int_{\frac{\sigma}{2}}^{-\frac{3\sigma}{2}+\pi} \frac{1}{4}+\frac{7}{4}~d\theta
 + \int_{-\frac{3\sigma}{2}+\pi}^{-\frac{\sigma}{2}+\pi} \frac{1}{4}~d\theta+\int_{-\frac{\sigma}{2}+\pi}^{\pi-t} \frac{1}{4}+\frac{7}{4}~d\theta\\
& =2\pi-\frac{15}{8}\sigma-\frac{1}{4}t\geq 2\pi-2\sigma
  \end{align*}
  where the equality is achieved when $t=\frac{\sigma}{2}$.
  \item When $-t\in [\frac{\sigma}{2},-\frac{3\sigma}{2}+\pi]$,
  \begin{align*}
      &\int_{-t}^{\pi-t} \frac14 \mathds{1}_{\cos(\theta)\leq \gamma}
 + \frac74 \mathds{1}_{-\cos(\theta+\sigma)\leq \gamma} ~d\theta\\
 &=  \int_{-t}^{-\frac{3\sigma}{2}+\pi} \frac{1}{4}+\frac{7}{4}~d\theta
 + \int_{-\frac{3\sigma}{2}+\pi}^{-\frac{\sigma}{2}+\pi} \frac{1}{4}~d\theta +  \int_{-\frac{\sigma}{2}+\pi}^{\pi-t} \frac{1}{4}+\frac{7}{4}~d\theta\\
& =2\pi-\frac{7}{4}\sigma\,.
  \end{align*}
  \item When $-t\in [-\frac{3\sigma}{2}+\pi,-\frac{\sigma}{2}+\pi]$,
  \begin{align*}
      &\int_{-t}^{\pi-t} \frac14 \mathds{1}_{\cos(\theta)\leq \gamma}
 + \frac74 \mathds{1}_{-\cos(\theta+\sigma)\leq \gamma} ~d\theta\\
 &=  \int_{-t}^{-\frac{\sigma}{2}+\pi} \frac{1}{4}~d\theta
 + \int_{-\frac{\sigma}{2}+\pi}^{\pi-t} \frac{1}{4}+\frac{7}{4}~d\theta\\
& =\frac{\pi}4+\frac{7}{8}\sigma-\frac{7}{4}t\geq 2\pi-\frac{7}{4}\sigma
  \end{align*}
  where the equality is achieved when $t=\frac{3\sigma}{2}-\pi$.
  \item When $-t\in [-\frac{\sigma}{2}+\pi,\frac{\sigma}{2}+\pi]$,
  \begin{align*}
      &\int_{-t}^{\pi-t} \frac14 \mathds{1}_{\cos(\theta)\leq \gamma}
 + \frac74 \mathds{1}_{-\cos(\theta+\sigma)\leq \gamma} ~d\theta\\
 &=  \int_{-t}^{-\frac{\sigma}{2}+2\pi} \frac{1}{4}+\frac{7}{4}~d\theta
 + \int_{-\frac{\sigma}{2}+2\pi}^{\pi-t} \frac{7}{4}~d\theta\\
& =\frac{9\pi}{4}-\frac{1}{8}\sigma+\frac{1}{4}t\geq 2\pi-\frac{1}{4}\sigma
  \end{align*}
  where the equality is achieved when $t=-\frac{\sigma}{2}-\pi$.
  \item
  When $-t\in [\frac{\sigma}{2}+\pi,-\frac{3\sigma}{2}+2\pi]$,
  \begin{align*}
      &\int_{-t}^{\pi-t} \frac14 \mathds{1}_{\cos(\theta)\leq \gamma}
 + \frac74 \mathds{1}_{-\cos(\theta+\sigma)\leq \gamma} ~d\theta\\
 &=  \int_{-t}^{-\frac{\sigma}{2}+2\pi} \frac{1}{4}+\frac{7}{4}~d\theta
 + \int_{-\frac{\sigma}{2}+2\pi}^{\frac{\sigma}{2}+2\pi} \frac{7}{4}~d\theta 
 + \int_{\frac{\sigma}{2}+2\pi}^{\pi-t} \frac{1}{4}+\frac{7}{4}~d\theta\\
& =2\pi-\frac{1}{4}\sigma\,.
  \end{align*}
  \end{itemize}
  Similarly for $\sigma\in [\frac{\pi}{2},\pi)$, we analyze six cases:
\begin{itemize}
    \item When $-t\in [-\frac{3\sigma}{2},\frac{\sigma}{2}-\pi]$,
    \begin{align*}
      &\int_{-t}^{\pi-t} \frac14 \mathds{1}_{\cos(\theta)\leq \gamma}
 + \frac74 \mathds{1}_{-\cos(\theta+\sigma)\leq \gamma} ~d\theta\\
 &=  \int_{-t}^{-\frac{\sigma}{2}} \frac{1}{4}+\frac{7}{4}~d\theta
 +  \int_{-\frac{\sigma}{2}}^{-\frac{3\sigma}{2}+\pi} \frac{7}{4}~d\theta\\
& =\frac{7}{4}\pi-\frac{11}{4}\sigma+2t\geq \frac{15}{4}\pi-\frac{15}{4}\sigma
  \end{align*}
  where the equality is achieved when $t=\pi-\frac{\sigma}{2}$.
 \item When $-t\in [\frac{\sigma}{2}-\pi,-\frac{\sigma}{2}]$,
\begin{align*}
      &\int_{-t}^{\pi-t} \frac14 \mathds{1}_{\cos(\theta)\leq \gamma}
 + \frac74 \mathds{1}_{-\cos(\theta+\sigma)\leq \gamma} ~d\theta\\
 &=  \int_{-t}^{-\frac{\sigma}{2}} \frac{1}{4}+\frac{7}{4}~d\theta
 +\int_{-\frac{\sigma}{2}}^{-\frac{3\sigma}{2}+\pi} \frac{7}{4}~d\theta
+\int_{\frac{\sigma}{2}}^{\pi-t} \frac{1}{4}~d\theta\\
& =2\pi-\frac{23}{8}\sigma+\frac{7}{4}t\geq 2\pi-2\sigma
  \end{align*}
  where the equality is achieved when $t=\frac{\sigma}{2}$.
  \item When $-t\in [-\frac{\sigma}{2},-\frac{3\sigma}{2}+\pi]$,
  \begin{align*}
     &\int_{-t}^{\pi-t} \frac14 \mathds{1}_{\cos(\theta)\leq \gamma}
 + \frac74 \mathds{1}_{-\cos(\theta+\sigma)\leq \gamma} ~d\theta\\
 &=  \int_{-t}^{-\frac{3\sigma}{2}+\pi} \frac{7}{4}~d\theta
 + \int_{\frac{\sigma}{2}}^{-\frac{\sigma}{2}+\pi} \frac{1}{4}~d\theta +  \int_{-\frac{\sigma}{2}+\pi}^{\pi-t} \frac{1}{4}+\frac{7}{4}~d\theta\\
& =2\pi-\frac{15}{8}\sigma-\frac{1}{4}t\geq 2\pi-2\sigma
  \end{align*}
  where the equality is achieved when $t=\frac{\sigma}{2}$.
  \item When $-t\in [-\frac{3\sigma}{2}+\pi,\frac{\sigma}{2}]$,
  \begin{align*}
      &\int_{-t}^{\pi-t} \frac14 \mathds{1}_{\cos(\theta)\leq \gamma}
 + \frac74 \mathds{1}_{-\cos(\theta+\sigma)\leq \gamma} ~d\theta\\
 &=  \int_{\frac{\sigma}{2}}^{-\frac{\sigma}{2}+\pi} \frac{1}{4}~d\theta
 + \int_{-\frac{\sigma}{2}+\pi}^{\pi-t} \frac{1}{4}+\frac{7}{4}~d\theta\\
& =\frac{\pi}4+\frac{3}{4}\sigma-2t\geq \frac{9}{4}\pi-\frac{9}{4}\sigma
  \end{align*}
  where the equality is achieved when $t=\frac{3\sigma}{2}-\pi$.
  \item When $-t\in [\frac{\sigma}{2},-\frac{\sigma}{2}+\pi]$,
  \begin{align*}
      &\int_{-t}^{\pi-t} \frac14 \mathds{1}_{\cos(\theta)\leq \gamma}
 + \frac74 \mathds{1}_{-\cos(\theta+\sigma)\leq \gamma} ~d\theta\\
 &=  \int_{-t}^{-\frac{\sigma}{2}+\pi} \frac{7}{4}~d\theta
 + \int_{-\frac{\sigma}{2}+\pi}^{\pi-t} \frac{1}{4}+\frac{7}{4}~d\theta\\
& =\frac{7\pi}{4}+\frac{1}{8}\sigma-\frac{1}{4}t\geq \frac{7\pi}{4}+\frac{1}{4}\sigma
  \end{align*}
  where the equality is achieved when $t=-\frac{\sigma}{2}$.
  \item
  When $-t\in [-\frac{\sigma}{2}+\pi,-\frac{3\sigma}{2}+2\pi]$,
  \begin{align*}
      &\int_{-t}^{\pi-t} \frac14 \mathds{1}_{\cos(\theta)\leq \gamma}
 + \frac74 \mathds{1}_{-\cos(\theta+\sigma)\leq \gamma} ~d\theta\\
 &=  \int_{-t}^{-\frac{\sigma}{2}+2\pi} \frac{1}{4}+\frac{7}{4}~d\theta
 + \int_{-\frac{\sigma}{2}+2\pi}^{\pi-t} \frac{7}{4}~d\theta\\
& =\frac94\pi-\frac{1}{8}\sigma+\frac{1}{4}t\geq \frac{7}{4}\pi+\frac{1}{4}\sigma
  \end{align*}
  where the equality is achieved when $t=\frac{3\sigma}{2}-2\pi$.
  \end{itemize}
  Therefore for $\sigma\in (0,\pi)$,
  \begin{equation*}
      \min_{t\in [0,2\pi)} \int_{-t}^{\pi-t} \frac14 \mathds{1}_{\cos(\theta)\leq \gamma}
 + \frac74 \mathds{1}_{-\cos(\theta+\sigma)\leq \gamma} ~d\theta=2\pi-2\sigma
  \end{equation*}
where the equality is achieved when $t=\frac{\sigma}{2}$. Therefore
 \begin{equation*}
     \cR^*_{\ell_{\gamma},\sH_{\mathrm{lin}}}=\frac{1}{2\pi}\times (2\pi-2\sigma)=1-\frac{\sigma}{\pi}\,,
 \end{equation*}
 where the unique Bayes classifier satisfies $t_1^*=\frac{\sigma}{2}$.
 
 For continuous margin-based loss $\phi$, by \eqref{target3} we have 
\begin{equation}
\begin{aligned}
  \cR^*_{\phi,\sH_{\mathrm{lin}}}&=\frac{1}{2\pi}\min_{t\in [0,2\pi]} \int_{0}^{\pi} \frac14\phi(\cos(\theta-t))~d\theta +\int_{0}^{\pi} \frac74\phi(\sin(\theta-t))~d\theta\\ 
 &= \frac{1}{2\pi}\min_{t\in [0,2\pi]} \int_{-t}^{\pi-t} \frac14\phi(\cos(\theta))+ \frac74\phi(-\cos(\theta+\sigma))~d\theta.\\
\end{aligned}
\label{surrogate}
\end{equation}
If $t^*=\frac{\sigma}{2}$ is the minimizer of $g(t)=\int_{-t}^{\pi-t} \frac14\phi(\cos(\theta))+ \frac74\phi(-\cos(\theta+\sigma))~d\theta,~t\in [0,2\pi]$, since $\frac{\sigma}{2}$ is not at the boundary of $[0,2\pi]$, we need
\begin{equation*}
    g'\left(\frac{\sigma}{2}\right)=0\,.
\end{equation*}
Since $\phi$ is continuous, by Leibniz Integral Rule, we have
\begin{align*}
   g'\left(\frac{\sigma}{2}\right)&=-\frac14\phi\left(\cos\left(\pi-\frac{\sigma}{2}\right)\right)- \frac74\phi\left(-\cos\left(\pi+\frac{\sigma}{2}\right)\right)+\frac14\phi\left(\cos\left(-\frac{\sigma}{2}\right)\right)+ \frac74\phi\left(-\cos\left(\frac{\sigma}{2}\right)\right)\\
   &=-\frac14\phi\left(-\cos\left(\frac{\sigma}{2}\right)\right)-\frac74\phi\left(\cos\left(\frac{\sigma}{2}\right)\right)+\frac14\phi\left(\cos\left(\frac{\sigma}{2}\right)\right)+\frac74\phi\left(-\cos\left(\frac{\sigma}{2}\right)\right)\\
   &=\frac32 \phi\left(-\cos\left(\frac{\sigma}{2}\right)\right)-\frac32 \phi\left(\cos\left(\frac{\sigma}{2}\right)\right)\,.
\end{align*}
Thus if $t^*=\frac{\sigma}{2}$ is the minimizer of $\cR^*_{\phi,\sH_{\mathrm{lin}}}$, we need $\phi$ satisfies 
\begin{equation}
    \phi\left(-\cos\left(\frac{\sigma}{2}\right)\right)= \phi\left(\cos\left(\frac{\sigma}{2}\right)\right)\,.
    \label{eq: conditions of phi}
\end{equation}
Therefore, if $\phi$ is $\sH_{\mathrm{lin}}$-consistent with respect to $\ell_{\gamma}$, we need $\phi$ satisfies \eqref{eq: conditions of phi} for any  $\sigma \in (0,\pi)$. Namely $\phi$ satisfies
\begin{equation*}
    \phi(-\tau)=\phi(\tau),\quad \tau\in [0,1)\,.
\end{equation*}
Note in our example, $\tau\in [-1,1]$, $\phi$ is continuous. We obtain that if $\phi$ is $\sH_{\mathrm{lin}}$-consistent with respect to $\ell_{\gamma}$, $\phi$ must be even function in $[-1,1]$. Next we claim that if $\phi$ is even function in $[-1,1]$, $\phi$ is not $\sH_{\mathrm{lin}}$-consistent with respect to $\ell_{\gamma}$. Indeed, for the distribution $y=1$ if $\theta \in (0,\pi)$ and $y=-1$ if $\theta \in (\pi,2\pi)$, we have
\begin{equation}
\begin{aligned}
  \cR^*_{\phi,\sH_{\mathrm{lin}}}
   &= \frac{1}{2\pi}\min_{t\in [0,2\pi]} \int_{0}^{\pi} \phi(\cos(\theta-t))+ \int_{\pi}^{2\pi} \phi(-\cos(\theta-t))~d\theta\\
   &= \frac{1}{\pi}\min_{t\in [0,2\pi]} \int_{0}^{\pi} \phi(\cos(\theta-t)) ~d\theta\\
 &=\frac{1}{\pi}\min_{t\in [0,2\pi]}\int_{-t}^{\pi-t}\phi(\cos(\theta))~d\theta\,.
\end{aligned}
\end{equation}
Note that when $\phi$ is even function in $[-1,1]$, 
$h(t)=\int_{-t}^{\pi-t}\phi(\cos(\theta))~d\theta$ satisfies
\begin{equation*}
    h'(t)=-\phi(-\cos(t))+\phi(\cos(t))=0, \quad t\in [0,2\pi]\,.
\end{equation*}
Thus $h(t)$ is a constant for $t\in [0,2\pi]$ and $\cR^*_{\phi,\sH_{\mathrm{lin}}}$ can be attained for any classifier $t\in [0,2\pi]$. However, $\cR^*_{\ell_{\gamma},\sH_{\mathrm{lin}}}$ can not be attained for any classifier $t\in [0,2\pi]$ with respect to this distribution. Therefore when $\phi$ is even function in $[-1,1]$, $\phi$ is not $\sH_{\mathrm{lin}}$-consistent with respect to $\ell_{\gamma}$. By the claim, any continuous loss is not $\sH_{\mathrm{lin}}$-consistent with respect to $\ell_{\gamma}$.
\end{proof}

\subsection{Proof of Theorem~\ref{Thm:calibrate_consistent_nonsup} and Theorem~\ref{Thm:calibrate_consistent_sup}}
\label{app:calibrate_consistent_nonsup}

Since the proofs adopt some results of \citep{steinwart2007compare}, 
we introduce the notation used in \citep{steinwart2007compare} 
to make the proofs more clear. In this section, we denote the loss $\ell(f, \bx, y)$ defined on a particular hypothesis set $\sH$ as $\ell_{\sH}(f,\bx,y)$. For a joint distribution $\sP$ over $\sX \times \sY$, the corresponding conditional distribution and marginal distribution are denoted as $\sP(\cdot | \bx)$ and $\sP_{X}$ respectively. In \citep{steinwart2007compare}, given a distribution $\sP$ over $\sX \times \sY$, the $\ell_{\sH}$-risk and the inner $\ell_{\sH}$-risk of a classifier $f \in \sH$ for the loss $\ell_{\sH}$ are denoted by
\[
    \cR_{\ell_{\sH},\sP}(f) = \E_{(\bx, y) \sim \sP}[\ell_{\sH}(f, \bx, y)],\quad \cC_{\ell_{\sH},\sP(\cdot|\bx),\bx}(f) = \E_{y \sim \sP(\cdot|\bx)}[\ell_{\sH}(f, \bx, y)].
\]
Accordingly, the minimal $\ell_{\sH}$-risk and minimal inner $\ell_{\sH}$-risk are denoted by $\cR_{\ell_{\sH},\sP}^*$ and $\cC_{\ell_{\sH},\sP(\cdot|\bx),\bx}^*$. 
For convenience, we will alternately use the notations of risk and inner risk presented above and Section~\ref{sec:preliminaries} for the proofs. Next, we introduce the $\sP$-minimizability proposed in \citep{steinwart2007compare}.
\begin{definition}[$\sP$-minimizability]
Given a distribution $\sP$ over $\sX\times\sY$ and a hypothesis set $\sH$. We say that loss $\ell_{\sH}(f,\bx,y)$ is $\sP$-minimizable if for all $\epsilon>0$ there exists $f_{\epsilon}\in \sH$ such that for all $\bx \in \sX$ we have
\[\cC_{\ell_{\sH},\sP(\cdot|\bx),\bx}(f_{\epsilon})<\cC_{\ell_{\sH},\sP(\cdot|\bx),\bx}^*+\epsilon.\]
\end{definition}
The following lemmas are useful in the proofs of Theorem~\ref{Thm:calibrate_consistent_nonsup} and Theorem~\ref{Thm:calibrate_consistent_sup}.
\begin{lemma}
\label{lemma:P-min}
Given a distribution $\sP$ over $\sX\times\sY$ and a hypothesis set $\sH$. Let $\phi$ be a margin-based loss. Then $\phi_{\sH_{\mathrm{all}}}$ is $\sP$-minimizable. If there exists $f^*\in\sH\subset\sH_{\mathrm{all}}$ such that $\cR^*_{\phi_{\sH_{\mathrm{all}}},\sP}=\cR_{\phi_{\sH},\sP}(f^*)$ , then $\phi_{\sH}$ is also $\sP$-minimizable in the almost surely sense.
\end{lemma}

\begin{proof}
By Theorem 3.2 of \citep{steinwart2007compare}, since $\cC^*_{\phi_{\sH_{\mathrm{all}}},\sP(\cdot|\bx),\bx}<\infty $ for all $\bx\in \sX$, $\phi_{\sH_{\mathrm{all}}}$ is $\sP$-minimizable.
Therefore, by Lemma 2.5 of \citep{steinwart2007compare}, we have
\begin{equation*}
    \cR^*_{\phi_{\sH_{\mathrm{all}}},\sP}=\int_{\sX} \cC_{\phi_{\sH_{\mathrm{all}}},\sP(\cdot|\bx),\bx}^*~d\sP_X(\bx).
\end{equation*}
Then by the assumption, 
\begin{equation*}
    \int_{\sX} \cC_{\phi_{\sH},\sP(\cdot|\bx),\bx}(f^*)~d\sP_X(\bx)=\cR_{\phi_{\sH},\sP}(f^*)=\cR^*_{\phi_{\sH_{\mathrm{all}}},\sP}=\int_{\sX} \cC_{\phi_{\sH_{\mathrm{all}}},\sP(\cdot|\bx),\bx}^*~d\sP_X(\bx).
\end{equation*}
Since 
\begin{equation*}
   \cC_{\phi_{\sH_{\mathrm{all}}},\sP(\cdot|\bx),\bx}^*\leq \cC_{\phi_{\sH},\sP(\cdot|\bx),\bx}(f^*),
\end{equation*}
for almost all $x\in X$,
\begin{equation*}
   \cC_{\phi_{\sH_{\mathrm{all}}},\sP(\cdot|\bx),\bx}^*= \cC_{\phi_{\sH},\sP(\cdot|\bx),\bx}(f^*).
\end{equation*}
As a result, for all $\epsilon>0$, there exists an $f^*\in\sH$ such that for almost all $x\in X$ we have 
\begin{equation*}
    \cC_{\phi_{\sH},\sP(\cdot|\bx),\bx}(f^*) < \cC_{\phi_{\sH_{\mathrm{all}}},\sP(\cdot|\bx),\bx}^*+\epsilon
    \leq \cC_{\phi_{\sH},\sP(\cdot|\bx),\bx}^*+\epsilon.
\end{equation*}
This completes the proof.
\end{proof}

\begin{lemma}
\label{lemma:eta}
Given a distribution $\sP$ over $\sX\times\sY$ and a hypothesis set $\sH$. Let $\phi$ be a margin-based loss.
If for $\eta\geq0$, there exists $f^*\in\sH\subset\sH_{\mathrm{all}}$  such that $\cR_{\phi_{\sH},\sP}(f^*)\leq \cR^*_{\phi_{\sH_{\mathrm{all}}},\sP}+\eta$, then $\phi_{\sH}$ satisfies 
\begin{equation*}
    \int_{\sX} \cC_{\phi_{\sH},\sP(\cdot|\bx),\bx}^*~d\sP_X(\bx)\leq
    \cR^*_{\phi_{\sH},\sP}\leq\int_{\sX} \cC_{\phi_{\sH},\sP(\cdot|\bx),\bx}^*~d\sP_X(\bx)+\eta.
\end{equation*}
\end{lemma}

\begin{proof}
By Lemma \ref{lemma:P-min}, $\phi_{\sH_{\mathrm{all}}}$ is $\sP$-minimizable. Then by Lemma 2.5 of \citep{steinwart2007compare}, we have
\begin{equation*}
   \cR^*_{\phi_{\sH_{\mathrm{all}}},\sP}=\int_{\sX} \cC_{\phi_{\sH_{\mathrm{all}}},\sP(\cdot|\bx),\bx}^*~d\sP_X(\bx).
\end{equation*}
Therefore, 
\begin{equation*}
    \cR^*_{\phi_{\sH},\sP} \leq\cR_{\phi_{\sH},\sP}(f^*)\leq \int_{\sX} \cC_{\phi_{\sH_{\mathrm{all}}},\sP(\cdot|\bx),\bx}^*~d\sP_X(\bx)+\eta
    \leq \int_{\sX} \cC_{\phi_{\sH},\sP(\cdot|\bx),\bx}^*~d\sP_X(\bx)+\eta.
\end{equation*}
Also,
\begin{align*}
    \int_{\sX} \cC_{\phi_{\sH},\sP(\cdot|\bx),\bx}^*~d\sP_X(\bx)
    &\leq \int_{\sX} \inf_{f\in\sH} \cC_{\phi_{\sH},\sP(\cdot|\bx),\bx}(f)~d\sP_X(\bx)\\
    &\leq \inf_{f\in\sH}\int_{\sX}  \cC_{\phi_{\sH},\sP(\cdot|\bx),\bx}(f)~d\sP_X(x) = \cR^*_{\phi_{\sH},\sP}.
\end{align*}
\end{proof}

\begin{lemma}
 \label{lemma:correctly classify}
Given a distribution $\sP$ over $\sX\times\sY$ with random variables $X$ and $Y$ and a hypothesis set $\sH$ such that $\cR^*_{\ell_{\gamma},\sH}=\cR_{\ell_{\gamma}}(f^*)=0$, where $f^*\in \sH$ achieves the Bayes risk. Then
 $f^*$ correctly classify $\bx\in \sX$ in the almost surely sense and for almost all $\bx\in \sX$, any $\bx'\in \curl*{\bx'\colon\|\bx'-\bx\|\leq \gamma}$ has same label as $\bx$.
\end{lemma}

 \begin{proof}
Since $\cR^*_{{\ell_{\gamma}}_{\sH},\sP}=\cR^*_{\ell_{\gamma},\sH}=0$, the distribution $\sP$ is $\sH$-realizable. Therefore $\mathbb{P}(Y=1|X=\bx)=1$ or $0$. Thus
\begin{equation*}
    \cC_{{\ell_{\gamma}}_{\sH},\sP(\cdot|\bx),\bx}(f)=
    \begin{cases}
    \sup\limits_{\bx'\colon \|\bx-\bx'\|\leq \gamma} \mathds{1}_{\curl*{ f(\bx')\leq 0}}, &\text{ if } \mathbb{P}(Y=1|X=\bx)=1,\\
    \sup\limits_{\bx'\colon \|\bx-\bx'\|\leq \gamma} \mathds{1}_{\curl*{- f(\bx')\leq 0}}, &\text{ if } \mathbb{P}(Y=1|X=\bx)=0,
    \end{cases}
\end{equation*}
Since $\cR_{{\ell_{\gamma}}_{\sH},\sP}(f^*)=\cR_{\ell_{\gamma}}(f^*)=0$, we have $\cC_{{\ell_{\gamma}}_{\sH},\sP(\cdot|\bx),\bx}(f^*)=0$ for almost all $x\in \sX$.
When $\mathbb{P}(Y=1|X=\bx)=1$, we obtain
\begin{equation}
    \sup\limits_{\bx'\colon \|\bx-\bx'\|\leq \gamma} \mathds{1}_{\curl*{f^*(\bx')\leq 0}}=0\implies f^*(\bx')>0 \text{ for any } \bx'\in \curl*{\bx'\colon\|\bx'-\bx\|\leq \gamma}.
    \label{eq:classify1}
\end{equation}
When $\mathbb{P}(Y=1|X=\bx)=0$, we obtain
\begin{equation}
  \sup\limits_{\bx'\colon \|\bx-\bx'\|\leq \gamma} \mathds{1}_{\curl*{-f^*(\bx')\leq 0}}=0\implies f^*(\bx')<0 \text{ for any } \bx'\in \curl*{\bx'\colon\|\bx'-\bx\|\leq \gamma}.
  \label{eq:classify2}
\end{equation}
Thus $f^*(\bx)>0$ when $\mathbb{P}(Y=1|X=\bx)=1$ and $f^*(\bx)<0$ when $\mathbb{P}(Y=1|X=\bx)=0$ for almost all $\bx\in \sX$. Therefore $f^*$ correctly classify $\bx\in \sX$ in the almost surely sense. Furthermore, by \eqref{eq:classify1} and \eqref{eq:classify2}, for almost all $\bx\in \sX$, any $\bx'\in \curl*{\bx'\colon\|\bx'-\bx\|\leq \gamma}$ has same label as $\bx$.
 \end{proof}

\begin{lemma}
\label{lemma:sup-P mini}
Given a distribution $\sP$ over $\sX\times\sY$ and a hypothesis set $\sH$ such that $\cR^*_{\ell_{\gamma},\sH}=0$. Let $\phi$ be a margin-based loss and $\tilde{\phi}(f,\bx,y)=\sup_{\bx'\colon \|\bx-\bx'\|\leq \gamma}\phi(y f(\bx'))$. If $\phi_{\sH}$ is $\sP$-minimizable in the almost surely sense, then $\tilde{\phi}_{\sH}$ is also $\sP$-minimizable in the almost surely sense.
\end{lemma}

\begin{proof}
As shown by \citet{awasthi2020adversarial}, $\tilde{\phi}$ has the equivalent form
\begin{equation*}
	\tilde{\phi}(f,\bx,y)=\phi\left(\inf\limits_{\bx'\colon \|\bx-\bx'\|\leq \gamma}\left(yf( \bx')\right)\right)\,.
\end{equation*}
Since $\cR^*_{{\ell_{\gamma}}_{\sH},\sP}=\cR^*_{\ell_{\gamma},\sH}=0$, the distribution $\sP$ is $\sH$-realizable. Therefore $\mathbb{P}(Y=1|X=\bx)=1$ or $0$. Thus
\begin{equation*}
    \cC_{\phi_{\sH},\sP(\cdot|\bx),\bx}(f)=
    \begin{cases}
    \phi(f(\bx)), &\text{ if } \mathbb{P}(Y=1|X=\bx)=1,\\
    \phi(-f(\bx)), &\text{ if } \mathbb{P}(Y=1|X=\bx)=0,
    \end{cases}
\end{equation*}
Note $\tilde{\phi}(f,\bx,1)=\phi\left(\inf_{\bx'\colon \|\bx-\bx'\|\leq \gamma}f( \bx')\right)=\phi(f(m_{f,\bx}))$, where WLOG we assume that $f$ is continuous and $m_{f,\bx}\in \curl*{\bx'\colon \|\bx-\bx'\|\leq \gamma}$ is the point such that $\min_{\bx'\colon \|\bx-\bx'\|\leq \gamma}f( \bx')=f(m_{f,\bx})$.
Similarly $\tilde{\phi}(f,\bx,-1)=\phi\left(-\sup_{\bx'\colon \|\bx-\bx'\|\leq \gamma}f( \bx')\right)=\phi(-f(M_{f,\bx}))$, where WLOG we assume that $f$ is continuous and $M_{f,\bx}\in \curl*{\bx'\colon \|\bx-\bx'\|\leq \gamma}$ is the point such that $\max_{\bx'\colon \|\bx-\bx'\|\leq \gamma}f( 
\bx')=f(M_{f,\bx})$. Then for $\tilde{\phi}_{\sH}$, we have
 \begin{equation*}
    \cC_{\tilde{\phi}_{\sH},\sP(\cdot|\bx),\bx}(f)=
    \begin{cases}
    \phi(f(m_{f,\bx})), &\text{ if } \mathbb{P}(Y=1|X=\bx)=1,\\
    \phi(-f(M_{f,\bx})), &\text{ if } \mathbb{P}(Y=1|X=\bx))=0,
    \end{cases}
\end{equation*} 
Since $\phi_{\sH}$ is $\sP$-minimizable in the almost surely sense, by the definition for all $\epsilon>0$, there exists an $f^*\in\sH$ such that for almost all $\bx\in \sX$ we have 
\begin{equation*}
    \cC_{\phi_{\sH},\sP(\cdot|\bx),\bx}(f^*)
    < \cC_{\phi_{\sH},\sP(\cdot|\bx),\bx}^*+\epsilon.
\end{equation*}
When $\mathbb{P}(Y=1|X=\bx)=1$, we obtain
\small
\begin{equation*}
 \cC_{\tilde{\phi}_{\sH},\sP(\cdot|\bx),\bx}(f^*)=\phi(f^*(m_{f^*,\bx}))=\cC_{\phi_{\sH},\sP(\cdot|m_{f^*,\bx}),m_{f^*,\bx}}(f^*)<\cC_{\phi_{\sH},\sP(\cdot|m_{f^*,\bx}),m_{f^*,\bx}}^*+\epsilon\leq \cC_{\tilde{\phi}_{\sH},\sP(\cdot|\bx),\bx}^*+\epsilon
\end{equation*}
\normalsize
where we used the fact that $m_{f^*,\bx}$ satisfies $\mathbb{P}(Y=1|X=m_{f^*,\bx})=1$ by Lemma \ref{lemma:correctly classify} and $\phi$ is non-increasing. 
Similarly, when $\mathbb{P}(Y=1|X=\bx)=0$, we obtain
\small
\begin{equation*}
 \cC_{\tilde{\phi}_{\sH},\sP(\cdot|\bx),\bx}(f^*)=\phi(-f^*(M_{f^*,\bx}))=\cC_{\phi_{\sH},\sP(\cdot|M_{f^*,\bx}),M_{f^*,\bx}}(f^*)<\cC_{\phi_{\sH},\sP(\cdot|M_{f^*,\bx}),M_{f^*,\bx}}^*+\epsilon\leq \cC_{\tilde{\phi}_{\sH},\sP(\cdot|\bx),\bx}^*+\epsilon
\end{equation*}
\normalsize
where we used the fact that $M_{f^*,\bx}$ satisfies $\mathbb{P}(Y=1|X=M_{f^*,\bx})=0$ by Lemma \ref{lemma:correctly classify} and $\phi$ is non-increasing.
Above all, for all $\epsilon>0$, there exists an $f^*\in\sH$ such that for almost all $\bx\in \sX$ we have 
\begin{equation*}
    \cC_{\tilde{\phi}_{\sH},\sP(\cdot|\bx),\bx}(f^*)
    < \cC_{\tilde{\phi}_{\sH},\sP(\cdot|\bx),\bx}^*+\epsilon.
\end{equation*}
\end{proof}

We modify Theorem 2.8 of \citep{steinwart2007compare}, whose proof is very similar.
\begin{theorem}
\label{theorem:general consistency }
Given a distribution $\sP$ over $\sX\times\sY$ and a hypothesis set $\sH$. Let $\ell_1\colon\sH\times\sX\times \sY \rightarrow[0,\infty]$, $\ell_2\colon\sH\times\sX\times \sY\rightarrow[0,\infty]$ be two losses defining on $\sH$ such that $\cR^*_{\ell_1,\sP}=\int_{\sX} \cC_{\ell_1,\sP(\cdot|\bx),\bx}^*~d\sP_X(\bx)<\infty$ and
$\int_{\sX} \cC_{\ell_2,\sP(\cdot|\bx),\bx}^*~d\sP_X(\bx)\leq \cR^*_{\ell_2,\sP}\leq\int_{\sX} \cC_{\ell_2,\sP(\cdot|\bx),\bx}^*~d\sP_X(\bx)+\eta<\infty$ for $\eta\geq0$. Furthermore assume that there exist a function $b\in\cL_1(\sP_X)$ and measurable functions $\delta(\epsilon,\cdot):X\rightarrow(0,\infty)$, $\epsilon>0$, such that 
    \begin{equation*}
        \cC_{\ell_1,\sP(\cdot|\bx),\bx}(f)\leq \cC_{\ell_1,\sP(\cdot|\bx),\bx}^*+b(\bx)
    \end{equation*}
    and
    \begin{equation*}
        \cC_{\ell_2,\sP(\cdot|\bx),\bx}(f)< \cC_{\ell_2,\sP(\cdot|\bx),\bx}^*+\delta(\epsilon,\bx) \implies 
        \cC_{\ell_1,\sP(\cdot|\bx),\bx}(f)< \cC_{\ell_1,\sP(\cdot|\bx),\bx}^*+\epsilon
    \end{equation*}
    for all $\bx\in \sX$, $\epsilon>0$ and $f\in\sH$. Then for all $\epsilon>0$ there exists $\delta>0$ such that for all $f\in\sH$ we have
    \begin{equation*}
        \cR_{\ell_2,\sP}(f)+\eta <\cR_{\ell_2,\sP}^*+\delta \implies \cR_{\ell_1,\sP}(f)<\cR_{\ell_1,\sP}^*+\epsilon.
    \end{equation*}
\end{theorem}

\begin{proof}
    Define $\cC_{1,\bx}(f)=\cC_{\ell_1,\sP(\cdot|\bx),\bx}(f)- \cC_{\ell_1,\sP(\cdot|\bx),\bx}^*$ and $\cC_{2,\bx}(f)=\cC_{\ell_2,\sP(\cdot|\bx),\bx}(f)- \cC_{\ell_2,\sP(\cdot|\bx),\bx}^*$ for $\bx\in \sX$, $f\in \sH$. For a fixed $\epsilon>0$, define $h(\bx)=\delta(\epsilon,\bx)$, $\bx\in \sX$. Then for all $\bx\in \sX$ and $f\in\sH$ such that $\cC_{1,\bx}(f)\geq\epsilon$, we have $\cC_{2,\bx}(f)\geq h(\bx)$. Therefore,
    \begin{align*}
        \cR_{\ell_2,\sP}(f)-\cR_{\ell_2,\sP}^*+\eta
        &\geq \cR_{\ell_2,\sP}(f)-\int_{\sX} \cC_{\ell_2,\sP(\cdot|\bx),\bx}^*~d\sP_X(\bx)\\
        &=\int_{\sX} \cC_{2,\bx}(f)~d\sP_X(\bx)
        \geq \int_{\cC_{1,\bx}(f)\geq\epsilon}h(\bx)~d\sP_X(\bx),
    \end{align*}
    for all $f\in\sH$. Furthermore, since $h(\bx)>0$ for all $\bx\in \sX$, the measure $\nu\colon=b\sP_X$ is absolutely continuous with respect to $\mu\colon=h\sP_X$, and thus there exists $\delta>0$ such that $\nu(A)<\epsilon$ for all measurable $A\subset X$ with $\mu(A)<\delta$. Therefore, for $f\in\sH$ with $\cR_{\ell_2,\sP}(f)-\cR_{\ell_2,\sP}^*+\eta<\delta $ and $A\colon=\curl*{\bx\in \sX, \cC_{1,\bx}(f)\geq\epsilon}$, we obtain
    \begin{align*}
        \cR_{\ell_1,\sP}(f)-\cR_{\ell_1,\sP}^*
        &=
        \int_{\cC_{1,\bx}(f)\geq\epsilon}\cC_{1,\bx}(f)~d\sP_X(\bx) + \int_{\cC_{1,\bx}(f)<\epsilon}\cC_{1,\bx}(f)~d\sP_X(\bx)\\
        &\leq \int_Ab(\bx)~d\sP_X(\bx)+\epsilon
        < 2\epsilon.
    \end{align*}
    \end{proof}

\CalibrateConsistentNonsup*

\begin{proof}
Since $\cR^*_{{\ell_{\gamma}}_{\sH},\sP}=\cR^*_{\ell_{\gamma},\sH}=0$, we obtain
\begin{equation*}
    0\leq \int_{\sX} \cC_{{\ell_{\gamma}}_{\sH},\sP(\cdot|\bx),\bx}^*~d\sP_X(\bx)\leq \cR^*_{{\ell_{\gamma}}_{\sH},\sP}=0.
\end{equation*}
By Lemma \ref{lemma:eta}, $\phi_{\sH}$ satisfies 
\begin{equation*}
    \int_X \cC_{\phi_{\sH},\sP(\cdot|\bx),\bx}^*~d\sP_X(\bx)\leq
    \cR^*_{\phi_{\sH},\sP}\leq\int_{\sX} \cC_{\phi_{\sH},\sP(\cdot|\bx),\bx}^*~d\sP_X(\bx)+\eta<\infty .
\end{equation*}
Since for all $\bx\in \sX$ and $f\in \sH$, $\cC_{{\ell_{\gamma}}_{\sH},\sP(\cdot|\bx),\bx}(f)\leq 1$, we obtain
\begin{equation*}
    \cC_{{\ell_{\gamma}}_{\sH},\sP(\cdot|\bx),\bx}(f)\leq \cC_{{\ell_{\gamma}}_{\sH},\sP(\cdot|\bx),\bx}^*+1.
\end{equation*}
Also, since $\phi$ is $\sH$-calibrated with respect to $\ell_{\gamma}$, for all $x\in \sX$, $\epsilon>0$ and $f\in \sH$, there exists $\delta>0$ such that 
\begin{equation*}
        \cC_{\phi_{\sH},\sP(\cdot|\bx),x}(f)< \cC_{\phi_{\sH},\sP(\cdot|\bx),\bx}^*+\delta \implies
        \cC_{{\ell_{\gamma}}_{\sH},\sP(\cdot|\bx),\bx}(f)< \cC_{{\ell_{\gamma}}_{\sH},\sP(\cdot|\bx),\bx}^*+\epsilon.
    \end{equation*}
Therefore by Theorem \ref{theorem:general consistency }, for all $\epsilon>0$ there exists $\delta>0$ such that for all $f\in\sH$ we have
    \begin{equation}
        \cR_{\phi_{\sH},\sP}(f)+\eta <\cR_{\phi_{\sH},\sP}^*+\delta \implies \cR_{{\ell_{\gamma}}_{\sH},\sP}(f)<\cR_{{\ell_{\gamma}}_{\sH},\sP}^*+\epsilon.
        \label{eq:consistency1}
    \end{equation}
Using the notations in Section \ref{sec:preliminaries}, we can rewrite \eqref{eq:consistency1} as 
\begin{align*}
    \cR_{\phi}(f)+\eta <\cR_{\phi,\sH}^*+\delta \implies \cR_{\ell_{\gamma}}(f)<\cR_{\ell_{\gamma},\sH}^*+\epsilon.
\end{align*}
\end{proof}

\CalibrateConsistentSup*

\begin{proof}
By Lemma \ref{lemma:P-min} and Lemma \ref{lemma:sup-P mini}, $\tilde{\phi}_{\sH}$ is $\sP$-minimizable in the almost surely sense. Then for any $n\in \mathbb{N}$, there exists an $f_n^*\in\sH$ such that for almost all $\bx\in \sX$ we have
\begin{equation*}
    \cC_{\tilde{\phi}_{\sH},\sP(\cdot|\bx),\bx}(f_n^*)
    < \cC_{\tilde{\phi}_{\sH},\sP(\cdot|\bx),\bx}^*+\frac1n.
\end{equation*}
Therefore
\begin{align*}
    &\cR_{\tilde{\phi}_{\sH},\sP}^*\leq \int_{\sX}  \cC_{\tilde{\phi}_{\sH},\sP(\cdot|\bx),\bx}(f_n^*)~d\sP_X(\bx) \leq \int_{\sX} \cC_{\tilde{\phi}_{\sH},\sP(\cdot|\bx),\bx}^*~d\sP_X(\bx)+\frac1n\\
    &\qquad\leq
    \inf_{f\in\sH}\int_{\sX}  \cC_{\tilde{\phi}_{\sH},\sP(\cdot|\bx),\bx}(f)~d\sP_X(\bx)+\frac1n \leq\cR^*_{\tilde{\phi}_{\sH},\sP}+\frac1n.
\end{align*}
By taking $n\rightarrow \infty$, we obtain
\begin{equation*}
   \cR_{\tilde{\phi}_{\sH},\sP}^*=\int_{\sX} \cC_{\tilde{\phi}_{\sH},\sP(\cdot|\bx),\bx}^*~d\sP_X(\bx)\,. 
\end{equation*}
Since $\cR^*_{{\ell_{\gamma}}_{\sH},\sP}=\cR^*_{\ell_{\gamma},\sH}=0$, we obtain
\begin{equation*}
    0\leq \int_{\sX} \cC_{{\ell_{\gamma}}_{\sH},\sP(\cdot|\bx),\bx}^*~d\sP_X(\bx)\leq \cR^*_{{\ell_{\gamma}}_{\sH},\sP}=0.
\end{equation*}
Since for all $x\in \sX$ and $f\in \sH$, $\cC_{{\ell_{\gamma}}_{\sH},\sP(\cdot|\bx),\bx}(f)\leq 1$, we obtain
\begin{equation*}
    \cC_{{\ell_{\gamma}}_{\sH},\sP(\cdot|\bx),\bx}(f)\leq \cC_{{\ell_{\gamma}}_{\sH},\sP(\cdot|\bx),\bx}^*+1.
\end{equation*}
Also, since $\tilde{\phi}$ is $\sH$-calibrated with respect to $\ell_{\gamma}$, for all $x\in \sX$, $\epsilon>0$ and $f\in \sH$, there exists $\delta>0$ such that 
\begin{equation*}
        \cC_{\tilde{\phi}_{\sH},\sP(\cdot|\bx),\bx}(f)< \cC_{\tilde{\phi}_{\sH},\sP(\cdot|\bx),\bx}^*+\delta \implies 
        \cC_{{\ell_{\gamma}}_{\sH},\sP(\cdot|\bx),\bx}(f)< \cC_{{\ell_{\gamma}}_{\sH},\sP(\cdot|\bx),\bx}^*+\epsilon\,.
    \end{equation*}
    Therefore by Theorem \ref{theorem:general consistency } ($\eta=0$ here), for all $\epsilon>0$ there exists $\delta>0$ such that for all $f\in\sH$ we have
    \begin{equation}
        \cR_{\tilde{\phi}_{\sH},\sP}(f) <\cR_{\tilde{\phi}_{\sH},\sP}^*+\delta \implies \cR_{{\ell_{\gamma}}_{\sH},\sP}(f)<\cR_{{\ell_{\gamma}}_{\sH},\sP}^*+\epsilon.
         \label{eq:consistency2}
    \end{equation}
Using the notations in Section \ref{sec:preliminaries}, we can rewrite \eqref{eq:consistency2} as 
\begin{align*}
    \cR_{\tilde{\phi}}(f)+\eta <\cR_{\tilde{\phi},\sH}^*+\delta \implies \cR_{\ell_{\gamma}}(f)<\cR_{\ell_{\gamma},\sH}^*+\epsilon.
\end{align*}
\end{proof}

\end{document}